%% file: main.tex
\documentclass[20pt]{extarticle}

\usepackage[utf8]{inputenc}
\usepackage[T1]{fontenc}
\usepackage{times}
\usepackage{helvet}
\usepackage{microtype}

\usepackage{PRIMEarxiv}
\usepackage{geometry}
\usepackage{titlesec}
\usepackage{setspace}
\usepackage{fancyhdr}
\usepackage{caption}
\usepackage{subcaption}
\usepackage{enumitem}
\usepackage[most]{tcolorbox}
\usepackage{placeins}

\usepackage{amsmath,amssymb,amsfonts,amsthm,mathtools}
\allowdisplaybreaks

\usepackage{graphicx}
\usepackage{xcolor}
\usepackage{colortbl}
\usepackage{pgf}
\usepackage{tikz}
\usepackage{pgf-pie}
\usetikzlibrary{matrix}
\usepackage{booktabs}
\usepackage{multirow}
\usepackage{array}
\usepackage{tabularx}
\usepackage{longtable}
\usepackage{float}
\usepackage{pifont}

\usepackage[numbers]{natbib}
\usepackage{url}
\usepackage[hidelinks]{hyperref}

\fancypagestyle{plain}{
  \fancyhf{}
  \fancyfoot[C]{\fontsize{10}{12}\selectfont\thepage}

}

\newcolumntype{L}{>{\raggedright\arraybackslash}p{6cm}}

\newtheorem{theorem}{Theorem}
\newtheorem{proposition}{Proposition}
\newtheorem{assumption}{Assumption}

\newtheorem{corollary}{Corollary}

\newcommand{\name}{\textsc{KINA}}
\newcommand{\fosd}{\succeq_{\mathrm{FOSD}}}

\fancypagestyle{firstpagewithlogos}{
    \fancyhf{}
    \fancyhead[C]{%
        \raisebox{-0.35em}[0pt][0pt]{%
            \begin{minipage}{\textwidth}
                \centering
                \makebox[\textwidth]{%
                    \includegraphics[height=0.48in]{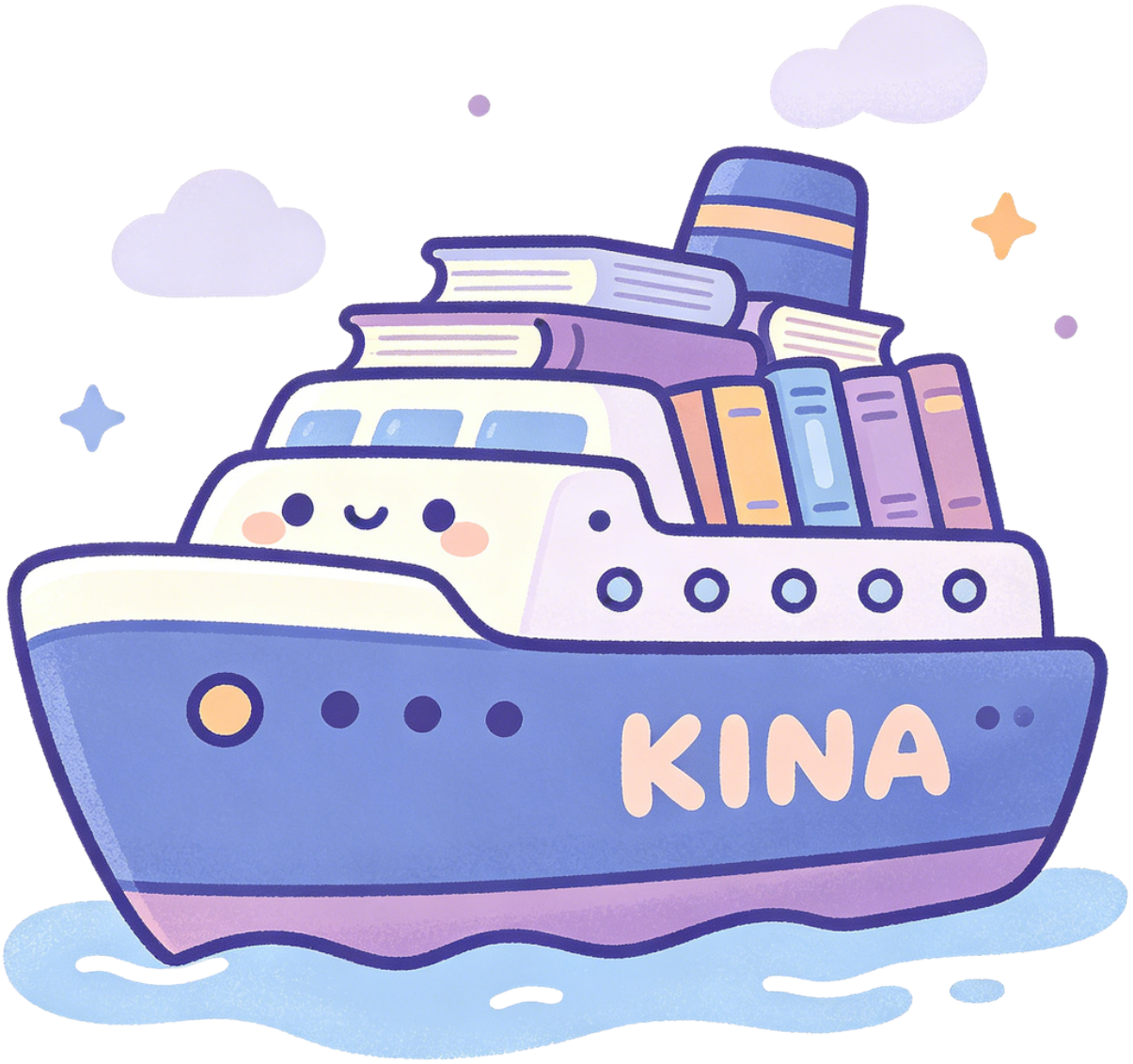}
                    \hfill
                        \makebox[0pt][r]{%
                            \includegraphics[height=0.24in]{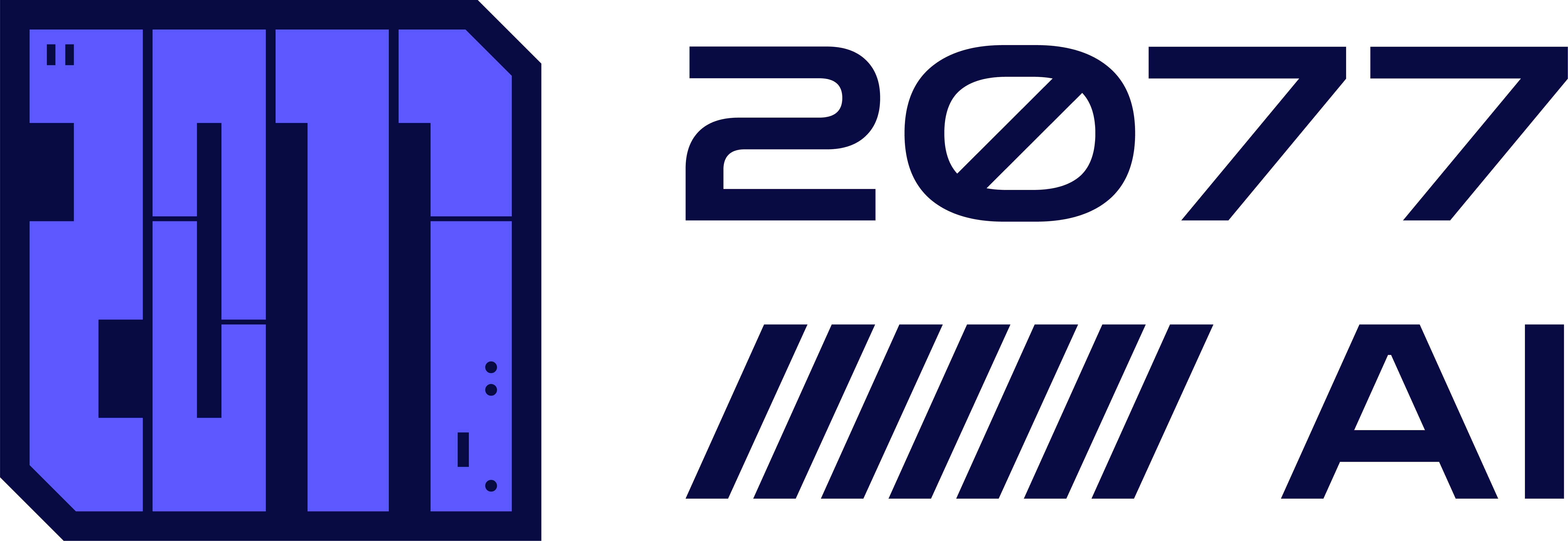}
                            \hspace{-0.3em}
                            \includegraphics[height=0.24in]{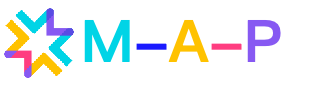}
                            \hspace{-0.7em}
                            \includegraphics[height=0.24in]{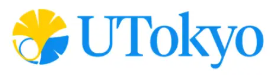}
                            \hspace{-0.4em}
                            \includegraphics[height=0.28in]{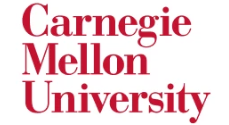}
                        }%
                }\\[-0.82em]
                \rule{\textwidth}{1.5pt}
            \end{minipage}%
        }%
    }
    
    \fancyfoot[C]{\fontsize{10}{12}\selectfont\thepage}
}

\title{\textbf{Knowledge Index of Noah's Ark}}

\author{\normalfont\fontsize{9.5pt}{11pt}\selectfont
Sheng Jin\textsuperscript{1,*},
Minghao Liu\textsuperscript{1,2,*,\textdagger},
Yunze Xiao\textsuperscript{3},
Zeqi Zhou\textsuperscript{4},
Heli Qi\textsuperscript{5},
Yifan Yao\textsuperscript{2},\\
\fontsize{9.5pt}{11pt}\selectfont
Meishu Song\textsuperscript{1,6},
Kaijing Ma\textsuperscript{2},
Xuan Zhang\textsuperscript{1},
Sicong Jiang\textsuperscript{1},
Yizhe Li\textsuperscript{1},
Ningshan Ma\textsuperscript{7},
Jie Wei\textsuperscript{1},\\
\fontsize{9.5pt}{11pt}\selectfont
Ziniu Li\textsuperscript{2},
Minglai Yang\textsuperscript{1,8},
Bangya Liu\textsuperscript{1},
Yiming Liang\textsuperscript{2},
Xiao Fang\textsuperscript{7},
Qingcheng Zeng\textsuperscript{9},
Jiarui Liu\textsuperscript{3},\\
\fontsize{9.5pt}{11pt}\selectfont
Rui Yang\textsuperscript{10},
Shen Yan\textsuperscript{2},
Wenhao Huang\textsuperscript{2},
Jiaheng Liu\textsuperscript{2},
Zihan Wang\textsuperscript{1},
Weihao Xuan\textsuperscript{6,\textdagger},
Ge Zhang\textsuperscript{2,\textdagger}\\[1.6mm]
\fontsize{9pt}{10.5pt}\selectfont
\textsuperscript{1}2077AI~~~
\textsuperscript{2}M-A-P~~~
\textsuperscript{3}Carnegie Mellon University~~~
\textsuperscript{4}Brown University~~~
\textsuperscript{5}Waseda University\\[0.35mm]
\fontsize{9pt}{10.5pt}\selectfont
\textsuperscript{6}The University of Tokyo~~~
\textsuperscript{7}Massachusetts Institute of Technology~~~
\textsuperscript{8}University of Arizona\\[0.35mm]
\fontsize{9pt}{10.5pt}\selectfont
\textsuperscript{9}Northwestern University~~~
\textsuperscript{10}Duke-NUS Medical School\\[0.8mm]
\fontsize{8.5pt}{10pt}\selectfont
\textsuperscript{*}Equal contribution.~~~
\textsuperscript{\textdagger}Corresponding authors.
}

\date{}

\makeatletter
\renewcommand{\@maketitle}{%
  \newpage
  \null
  \vspace{-0.8em}%
  \begin{center}%
    {\LARGE\bfseries \@title \par}%
    \vskip 1.15em%
    \rule{\textwidth}{0.3pt}%
    \vskip 1.05em%
    {\@author \par}%
    \vskip 0.8em%
    {\@date}%
  \end{center}%
  \par
  \vskip 1.2em%
}
\makeatother

\begin{document}
\setlength{\headheight}{34pt}
\setlength{\headsep}{16pt}
\maketitle
\thispagestyle{firstpagewithlogos}

\begin{abstract}
Knowledge benchmarks for LLMs face three issues: scaling-driven designs that do not operationalize disciplinary representativeness; flat-payment annotation that permits lazy consensus; and unaudited ranking instability under bounded test budgets. We introduce \name{}\footnote{\url{https://www.2077ai.com/datasets/dataset-kina}}, an $899$-item benchmark across $261$ fine-grained disciplines, with two formal results. First, we cast representativeness as a coverage-style objective over expert-elicited anchors and operationalize disciplinary representativeness through a proxy, yielding a $(1-1/e)$ greedy approximation (Proposition~\ref{thm:submodular}); the guarantee applies to the proxy, not to population representativeness. Second, we prove a bonus-on-bar tournament weakly FOSD-dominates flat payment in released-review quality, with incentive-compatibility threshold $B > \Delta C / \Delta p_{\min}$ (Theorem~\ref{thm:tournament}). Evaluating $42$ models from $13$ labs, the top model, Gemini-3.1-Pro-Preview, reaches $53.17\%$, followed by Claude-Opus-4.6 at $49.92\%$ and GPT-5.4 at $48.55\%$, leaving substantial headroom below saturation. The full leaderboard shows a tiered structure rather than a smooth total order: a small frontier tier lies above $48\%$, a dense strong-model tier spans roughly $38$--$45\%$, and low-performing models remain only modestly above the $10\%$ chance baseline. Tool-augmentation adds up to $5.17$ points across the five tool-use evaluations, with gains varying substantially across models. We report bootstrap ranking-stability statistics to make bounded-budget variance explicit and to discourage over-interpretation of adjacent ranks.
\end{abstract}

\input{1_introduction}
\input{2_related_work}
\input{3_representativeness}
\input{4_kina_construction}
\input{5_experiments}
\input{6_limitations}
\input{7_conclusion}
\input{8_contributions_and_acknowledgements}

\bibliographystyle{plainnat}
\bibliography{refs}

\newpage
\appendix
\input{appendix_a_samples}
\input{appendix_b_centrality}
\input{appendix_c_tournament}
\input{appendix_d_taxonomy}
\input{appendix_e_manuals}

\input{appendix_f_experiments}

\input{appendix_g_datasheet}

\end{document}

%% file: 1_introduction.tex
\section{Introduction}

A knowledge benchmark for LLMs can serve as either a difficulty thermometer or a diagnostic instrument, and current designs lean predominantly toward the former. The dominant strategies, namely scaling the test set toward hundreds of subfields \citep{pteam2025supergpqascalingllmevaluation} and pushing each item toward research-frontier difficulty \citep{center2026benchmark}, tell us \emph{whether} frontier models struggle, but say less about \emph{where} and \emph{why} they do. Existing pipelines do not formalize disciplinary representativeness as a selection criterion, do not align reviewer compensation with effort under formal guarantees, and do not routinely report whether observed model rankings remain stable under resampling. We argue that representativeness, incentive-aligned review, and ranking stability deserve to be treated as primary design considerations of a knowledge benchmark, and we develop \name{}, the Knowledge Index of Noah's Ark, to make that case concrete.

We motivate each of the three concerns in turn. Items in current benchmarks are typically organized by subject taxonomy rather than selected by whether they elicit the core competencies of a discipline; aggregate difficulty can therefore be high while the coverage of theoretical pivots remains uneven. Most review pipelines also compensate reviewers at a flat per-item rate, so accepting a borderline item incurs no cost and rational reviewers tend to converge on \emph{lazy consensus}; bonus-heavy designs such as the expert pipeline of GPQA \citep{rein2023gpqagraduatelevelgoogleproofqa} mitigate but do not formally rule out this failure mode. Finally, although a compact test set of about $10^3$ items is cheap to iterate and resistant to contamination, $5$-point gaps near the top of the leaderboard between frontier models on such a set may not survive resampling, and bootstrap-stable rankings are rarely reported alongside the headline numbers.

\name{} is a benchmark of $899$ items spanning $261$ fine-grained disciplines, designed around two formal results that target the first two of these gaps. To address representativeness, we operationalize the criterion as \emph{budgeted support centrality}, an operational proxy under which each candidate item is scored against domain-aligned anchors elicited from experts and items are then selected greedily under capacity constraints. Proposition~\ref{thm:submodular} shows that the resulting selection objective is monotone submodular, so greedy selection attains a $(1-1/e)$ approximation of the proxy optimum. We view the proxy as a tractable surrogate for representativeness rather than a global guarantee on representativeness itself. To address review quality, we replace flat payment with a \emph{bonus-on-bar tournament}: two reviewers evaluate each item, the higher-scoring reviewer receives a bonus $B$ provided that the winning score clears a bar $\tau$, and the principal performs stochastic audits to deter collusive approval. Under standard assumptions on effort-induced FOSD, score-noise independence, and monotonicity of the aggregator (Assumptions~\ref{ass:fosd}--\ref{ass:types}), Theorem~\ref{thm:tournament} establishes that this mechanism strictly improves the quality of released reviews relative to the flat-payment baseline in the FOSD sense, with a closed-form bonus calibration $B>\Delta C/\Delta p_{\min}$. The third gap, ranking stability, is addressed empirically rather than formally, by reporting bootstrap-based ranking statistics as a standard column on the leaderboard of \name{}.

We evaluate $42$ frontier models from $13$ labs on \name{}. Gemini-3.1-Pro-Preview leads with an overall accuracy of $53.17\%$, followed by Claude-Opus-4.6 and GPT-5.4 at $49.92\%$ and $48.55\%$ respectively, and the leaderboard remains far from saturation. Web-search augmentation yields positive but non-uniform gains across the five tool-use evaluations, ranging from $+1.50$ to $+5.17$ points, which suggests that retrieval contributes in different ways to weak and strong base models. Per-discipline analysis reveals heterogeneous spread within the top-$10$ models that aggregate accuracy alone would hide: the spread is only $9.83$ points in Science, while it reaches $38.16$ points in Sociology. We read this contrast as suggestive rather than conclusive, since the disciplines with the largest spread also have the smallest item counts; what we draw from it is that humanities and social-science content deserves separate reporting at the frontier, not that it is the principal source of differentiation. We further report bootstrap-based ranking-stability statistics in \S\ref{subsec:stability}, in order to make the variance under bounded test budgets visible to readers rather than implicit.

Our contributions are threefold.
\begin{enumerate}
    \item (\S\ref{subsec:representativeness}) formalizes disciplinary representativeness as budgeted support centrality and proves that the resulting selection objective is monotone submodular, which gives a $(1-1/e)$ approximation guarantee on the proxy under greedy selection. 
    \item (\S\ref{subsec:tournament}) shows that, under standard assumptions on effort-induced FOSD, score-noise independence, and monotonicity of the aggregator, a bonus-on-bar tournament strictly improves the quality of released reviews relative to the flat-payment baseline, with a closed-form bonus calibration.
    \item (\S\ref{sec:experiments}) releases \name{} together with an evaluation of $42$ frontier models, per-discipline scores, web-search ablations, an analysis of parameter scaling, and bootstrap-based ranking-stability statistics. 
\end{enumerate}

Alongside the dataset, we release the annotation and reviewer manuals, the LLM-judge rubrics, and the evaluation code, to enable replication and extension.

%% file: 2_related_work.tex
\section{Related Work}
\label{sec:related_work}

\paragraph{Knowledge benchmarks.}
The standard paradigm was set by SuperGLUE \citep{Wang2019SuperGlue} and MMLU
\citep{hendrycks2021measuringmassivemultitasklanguage}, both of which are now
saturated by frontier models and have known quality issues. Post-hoc analysis
of MMLU \citep{gema2025mmlu} estimates a $6.5\%$ overall error rate, with
per-subject rates exceeding $50\%$ in several domains and rank changes on
error-corrected subsets. Subsequent work pursues two trajectories. The
\emph{depth-first} line includes ScienceQA \citep{lu2022learn}, ARC-AGI
\citep{chollet2025arcprize2024technical, chollet2026arcagi2newchallengefrontier},
GPQA \citep{rein2023gpqagraduatelevelgoogleproofqa}, and HLE
\citep{center2026benchmark}. The \emph{breadth-first} line includes MMLU-Pro
\citep{wang2024mmlu} and SuperGPQA \citep{pteam2025supergpqascalingllmevaluation}.
None of these designs makes \emph{disciplinary representativeness} an
explicit selection criterion: items are scored on difficulty or sourced by
availability, not on whether they probe central theoretical pivots.

\paragraph{Annotation methodology and incentive design.}
GPQA pioneered an expert-in-the-loop pipeline with $\sim$\$95/hr compensation
and three-stage validation, but covers only three domains.
SuperGPQA scales the same expert-review template to $\sim$$80$ annotators and
$26{,}529$ items; an initial broader-crowdsourcing attempt rejected $63\%$ of
submissions, motivating a shift to expert-only annotation. HLE used an
open-call solicitation with a $\$500{,}000$ prize pool and a $\sim$$5$-minute
verification cap per item. The post-release audit
\citep{zhai2026hleverifiedsystematicverificationstructured,
tu2025positionhiddencostsmeasurement} found that $29\pm 3.7\%$ of biology and
chemistry items conflicted with peer-reviewed literature, and only $26\%$ of
audited items passed verification unmodified. The structural cause is
incentive: a flat-payment cap on review effort, combined with a
``stump-the-LLM'' contributor incentive, drives reviewers toward lazy
consensus. To our knowledge, no prior knowledge benchmark formally analyzes
the incentive structure of its review pipeline.

\paragraph{Mechanism design for elicitation.}
Tournament and contest-based mechanisms have a long history in labor economics
\citep{lazear1981rank} and have been used in NLP for adversarial data
collection \citep{nie2020adversarialnli}. Our use of a bonus-on-bar tournament
for annotation review is, to our knowledge, novel in the LLM-benchmark setting.
Our analysis (\S\ref{subsec:tournament}) is a comparative claim relative to flat
payment under standard FOSD assumptions; it is not a full equilibrium
characterization.

\paragraph{Comparison.} Table~\ref{tab:bench_comparison} situates \name{}
against representative prior benchmarks along five axes.

\begin{table*}[htbp]
\centering
\captionsetup{font=footnotesize}
\caption{\textbf{Comparison of \name{} with prior knowledge benchmarks.}
``Disc.'' = number of disciplines; ``Rep.''~$=$~explicit representativeness
criterion; ``Inc.''~$=$~explicit incentive-aligned review; ``Tool''~$=$~tool-use
evaluation reported in original paper. Best-saturated overall accuracy of
frontier models is approximate.}
\label{tab:bench_comparison}
\small
\begin{tabular}{lrrrcccc}
\toprule
Benchmark & \#Disc. & \#Items & \#Opts & Frontier acc.\ ($\approx$) & Rep.\ & Inc.\ & Tool \\
\midrule
MMLU \citep{hendrycks2021measuringmassivemultitasklanguage}
        & 57  & 15{,}908 & 4 & $>$90\%  & \ding{55} & \ding{55} & \ding{55} \\
MMLU-Pro \citep{wang2024mmlu}
        & 14  & 12{,}032 & 10 & 80--90\% & \ding{55} & \ding{55} & \ding{55} \\
GPQA \citep{rein2023gpqagraduatelevelgoogleproofqa}
        & 3   & 448      & 4 & $>$90\%  & \ding{55} & partial   & \ding{55} \\
SuperGPQA \citep{pteam2025supergpqascalingllmevaluation}
        & 285 & 26{,}529 & $\sim$10 & 70--80\% & \ding{55} & \ding{55} & \ding{55} \\
HLE \citep{center2026benchmark}
        & 8   & 2{,}500  & mixed & 40--50\% & \ding{55} & \ding{55} & partial \\
\midrule
\textbf{\name{} (ours)}
        & \textbf{261} & \textbf{899} & \textbf{10}
        & \textbf{50--60\%}
        & \ding{51} & \ding{51} & \ding{51} \\
\bottomrule
\end{tabular}
\end{table*}

%% file: 3_representativeness.tex
\section{Two Formal Guarantees for the Data Pipeline}
\label{sec:theory}

The data pipeline of \name{} makes two decisions whose quality dominates everything downstream: which items to keep from a much larger candidate pool, and how to compensate the reviewers who certify the survivors. We give a formal guarantee at each decision point. The first guarantee says greedy selection is good enough at covering a discipline, in a sense we make precise below. The second says a tournament-style payment scheme is good enough at eliciting reviewer effort, in a similar sense. The two arguments are independent, but together they pin down what a reader can take from the pipeline of \name{} as a formal claim, as opposed to a methodological choice. Section~\ref{sec:construction} describes how each guarantee is realized in the actual workflow, and Section~\ref{sec:limitations} states what neither guarantee establishes.

\subsection{Selection: greedy coverage of a disciplinary prototype}
\label{subsec:representativeness}

Representativeness is a property of the selected subset, not of items in isolation. An item is well-positioned if it materially supports at least one canonical anchor of its discipline; a subset is representative if every important anchor is supported by at least one item it contains. To make this concrete, domain experts elicit a compact disciplinary prototype $\Sigma_d$ of methods, problems, theorems, concepts, and applications, and an LLM-judge scores each candidate item $q$ for how strongly it supports each anchor $u$, producing $\hat S_d^{\mathrm{sp}}(q,u) \in [0,1]$. We aggregate these scores into a coverage objective:
\begin{equation}
F_d^{\mathrm{sp}}(\mathcal{S})
\;\triangleq\;
\sum_{u \in \bar B_d} \mu_d(u)\, \max_{q \in \mathcal S}\, \hat S_d^{\mathrm{sp}}(q,u),
\label{eq:Fsp}
\end{equation}
with anchor weights $\mu_d(u)\ge 0$. Selection picks $\mathcal{S}_d$ of size $K_d$ to maximize $F_d^{\mathrm{sp}}$, subject to per-subfield quotas and a near-duplicate constraint we describe in Appendix~\ref{sec:appendix-support-centrality}.

The objective $F_d^{\mathrm{sp}}$ is monotone and submodular: for each anchor $u$, the per-anchor coverage $\max_{q\in\mathcal S} \hat S_d^{\mathrm{sp}}(q,u)$ is monotone and submodular in $\mathcal S$ by a standard max-coverage argument; $F_d^{\mathrm{sp}}$ is a nonnegative weighted sum of such functions and inherits both properties. Greedy maximization under the cardinality constraint $|\mathcal{S}_d|=K_d$ therefore attains a $(1-1/e)$ approximation of the proxy optimum (Proposition~\ref{thm:submodular}; full statement and proof in Appendix~\ref{sec:appendix-support-centrality}). The lazy-greedy variant we use is described there as well.

\subsection{Review: bonus-on-bar tournament}
\label{subsec:tournament}

Two reviewers evaluate each item independently. Under flat per-item payment, the reviewer's optimal effort is whatever minimizes private disutility regardless of item quality, which is the failure mode we want to rule out. We instead pay a base wage plus a single bonus to the reviewer with the higher validated score, conditional on the winning score clearing a minimum bar $\tau$. The bar exists to deter collusive approval: two reviewers cannot agree to put in low effort and split the bonus, because a low-quality winner forfeits it.

We want to show that this scheme strictly raises the quality of released reviews relative to the flat baseline, under transparent assumptions. We assume effort raises latent review quality in the sense of first-order stochastic dominance (Assumption~\ref{ass:fosd}), that the principal observes a noisy monotone score whose noise is independent across reviewers (Assumption~\ref{ass:noise}), and that reviewer cost types are drawn i.i.d.\ from a continuous distribution $G$ (Assumption~\ref{ass:types}). Let $\Delta p_{\min}$ be the minimum gain in winning probability when a reviewer switches from low to high effort, taken over the opponent's effort choice; positivity of $\Delta p_{\min}$ is the substantive precondition.

Under these assumptions, the bonus-on-bar tournament is FOSD-improving: the equilibrium high-effort rate satisfies $\pi^{\mathrm{tour}}\ge G(B\,\Delta p_{\min})\ge G(0)=\pi^{\mathrm{flat}}$, and released review quality satisfies $Y^{\mathrm{tour}}\fosd Y^{\mathrm{flat}}$ (Theorem~\ref{thm:tournament}; full statement and proof in Appendix~\ref{appendix:theory}). The argument has two pieces: a reviewer with cost type $\kappa_i$ prefers high effort under tournament whenever $B\,\Delta p(e_{-i})\ge \kappa_i$, so all types with $\kappa_i\le B\,\Delta p_{\min}$ choose high effort regardless of the opponent's behavior, while under flat payment no positive-cost type does; the induced quality CDF of a randomly drawn review is therefore stochastically lower under tournament, and a quantile-coupling step lifts marginal FOSD to any coordinatewise nondecreasing aggregator of the two reviews.

A useful consequence is that the bonus needed to dominate the high-effort branch has a closed form: setting $B>\Delta C/\Delta p_{\min}$, where $\Delta C$ bounds the cost of high effort, makes high effort strictly preferred for every reviewer type (Corollary~\ref{cor:calibration}). More generally, $B\ge G^{-1}(\pi^\star)/\Delta p_{\min}$ delivers any target high-effort rate $\pi^\star\in[0,1)$. We use the closed-form bound in pilot calibration and the more general form when targeting a specific high-effort fraction in a niche discipline.

%% file: 4_kina_construction.tex
\section{\name{} Construction}
\label{sec:construction}

This section describes how the formal criteria of \S\ref{subsec:representativeness} and \S\ref{subsec:tournament} are realized in a concrete data-collection pipeline. The taxonomy backbone is the U.S.\ Classification of Instructional Programs (CIP), which we refine to a three-level hierarchy of $12$ disciplines, $70$ fields, and $261$ fine-grained subfields; the full taxonomy with per-subfield item counts is in Appendix~\ref{sec:appendix-taxonomy}. Among compact knowledge benchmarks ($\le 10^4$ items), this is the broadest taxonomy we are aware of, and only SuperGPQA ($285$ subfields) is wider among very large benchmarks, at roughly $30\times$ the item count. Each \name{} item is a pseudo-multiple-choice question: the stem contains several substantive statements, and each of the $10$ options is a combinatorial selection over those statements. This format reduces chance accuracy from $25\%$ ($4$-option MCQ) to $10\%$, weakens process-of-elimination shortcuts, and forces joint reasoning over all statements. Each item ships with the question, the $10$ options, an option-level explanation with sources, and the originating material; examples are in Appendix~\ref{appendix:samples}.

\subsection{Pipeline}
\label{subsec:pipeline}

\begin{figure*}[htbp]
\centering
\includegraphics[width=0.9\linewidth]{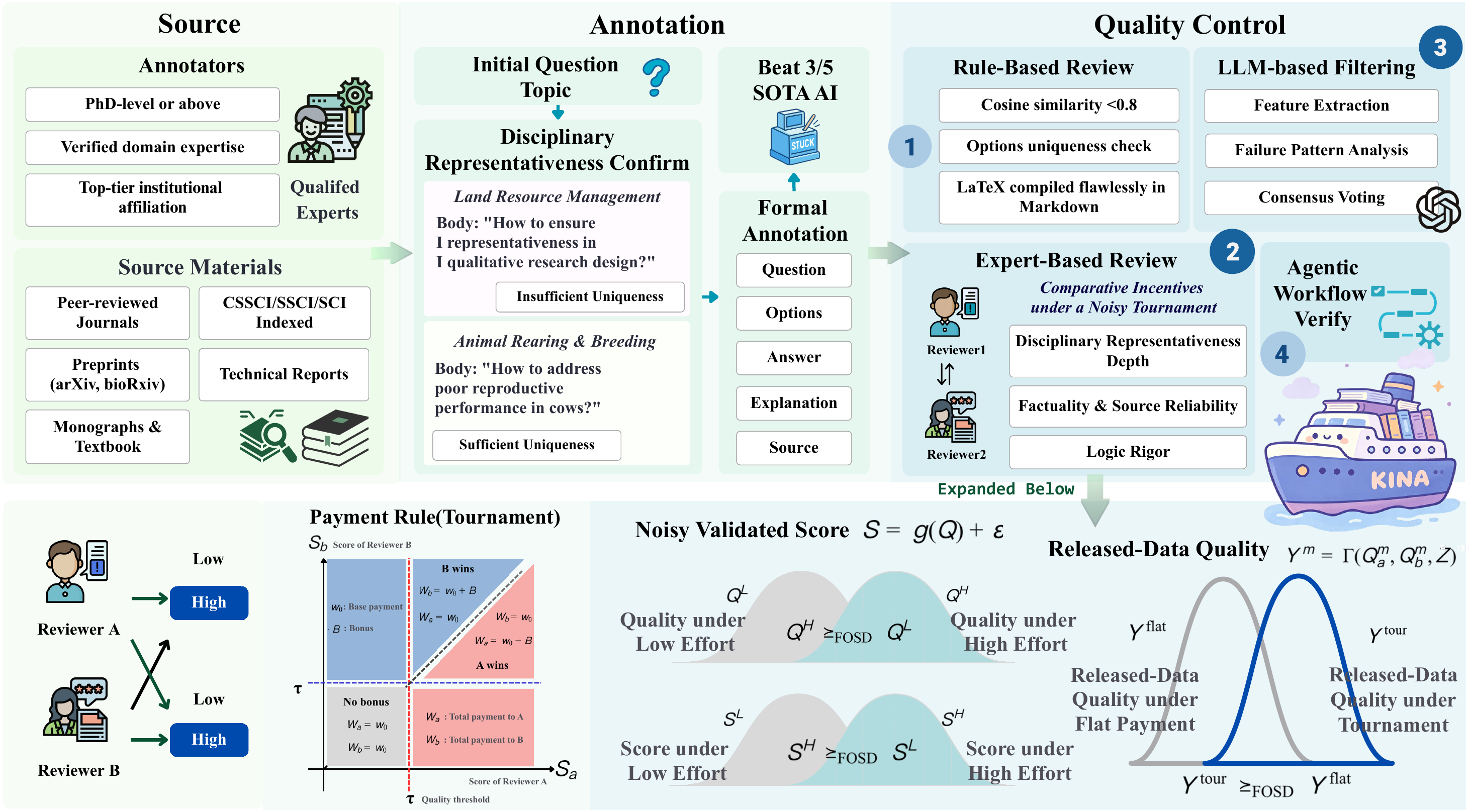}
\captionsetup{font=footnotesize}
\caption{\name{} data-collection pipeline. Topic pre-approval enforces representativeness via the proxy of Proposition~\ref{thm:submodular}; the double-blind expert-review stage instantiates the bonus-on-bar tournament of Theorem~\ref{thm:tournament}; LLM-as-judge consensus filters residual ambiguity; an agentic refinement loop addresses boundary defects.}
\label{fig:pipeline}
\end{figure*}

Each candidate item flows through four stages (Figure~\ref{fig:pipeline}, Table~\ref{tab:verification_pipeline}).

\paragraph{Rule-based screening.}
Each item passes a uniqueness check (cosine similarity below $0.8$ against the existing pool, no duplicate options), a formatting check (LaTeX compiles in Markdown), and a difficulty filter that requires at least three of five flagship LLMs to answer the item incorrectly. The 3-of-5 floor preserves both difficulty headroom and topic diversity at the rate observed during pilot collection.

\paragraph{Expert review under the bonus-on-bar tournament.}
Reviewers are recruited from two pools, graduate students at top-tier global universities and senior industry experts, and each candidate passes a two-round examination, namely a discipline-specific depth test in Round 1 and a one-shot formal item simulation in Round 2; the full annotation and reviewer manuals are in Appendix~\ref{appendix:manuals}. Each item is assigned to two independent reviewers under double-blind allocation, who score it on six rubrics: representativeness and depth, factuality, source reliability, logical rigor, distractor plausibility, and combinatorial validity (a pseudo-MCQ must contain at least $6$ statements with no subset relations among options). The validated score $S_i$ is the rubric sum normalized by the rubric weight schedule, and the bonus $B$ is awarded to the reviewer with the higher $S_i$ provided $\max_i S_i \ge \tau$. The principal performs stochastic audits on a random $5$ to $10\%$ of approved items, and an item flagged as flawed triggers a joint penalty on both reviewers, which deters collusive lazy consensus. We calibrate $B$ per discipline using Corollary~\ref{cor:calibration} from a pilot estimate of $\Delta p_{\min}$ and per-discipline cost $\Delta C$, cap reviewer work-in-progress at $3$ to $5$ items, and use scarcity-aware pricing for niche disciplines.

\paragraph{LLM-as-judge consensus.}
Three independent LLM judges score each item along four feature axes, namely knowledge coverage, disciplinary uniqueness, socio-economic impact, and practical value, and analyze the failure pattern of the five flagship LLMs on the item. An item is admitted if at least two of three judges vote yes; otherwise it is returned to the annotator for revision or discarded.

\paragraph{Agentic refinement of flagged items.}
A manual spot-check after the consensus stage identified a residual class of boundary defects, primarily context drift between stem and options and ambiguity in the implied causal chain of an explanation. To address these systematically, items flagged in this spot-check or returned by the LLM-judge consensus enter a refinement loop: an upstream agent (with web access) searches for counter-evidence to the explanation, and a downstream agent revises the stem or supplies missing premises in response. Each revised item is re-scored by a human reviewer under the Stage~2 rubric and must clear the same bar before re-entering the pool. Both agents are instantiated with GPT-5.2-Pro in our run; the loop is model-agnostic. Approximately one-third of items required at least one round of refinement before final acceptance, yielding the final $899$ items.

\begin{table*}[htbp]
\centering
\captionsetup{font=footnotesize}
\caption{Four-stage construction pipeline.}
\label{tab:verification_pipeline}
\small
\begin{tabular}{lp{2.6cm}p{6.0cm}}
\toprule
Stage & Method & Core criteria \\
\midrule
1 & Rule-based (automated) &
Cosine similarity $<0.8$; LaTeX compilable; 3-of-5 flagship-LLM-failure filter. \\
\midrule
2 & Double-blind expert review with bonus-on-bar tournament (\S\ref{subsec:tournament}) &
Representativeness/depth; factuality and source reliability (Q1 journals, monographs, CSSCI, authoritative domains); logical rigor; pseudo-MCQ format constraints. Stochastic principal audit at $5$ to $10\%$. \\
\midrule
3 & LLM-as-judge (three-judge consensus) &
Multidimensional feature scoring; flagship-failure pattern analysis; majority vote. \\
\midrule
4 & Two-agent refinement &
Diagnosis Agent mines counter-evidence; Refinement Agent revises stem and premises. Human reviewer confirms each refinement. \\
\bottomrule
\end{tabular}
\end{table*}

\subsection{Statistics}
\label{subsec:statistics}

\begin{figure}[htbp]
\centering
\includegraphics[width=0.62\linewidth]{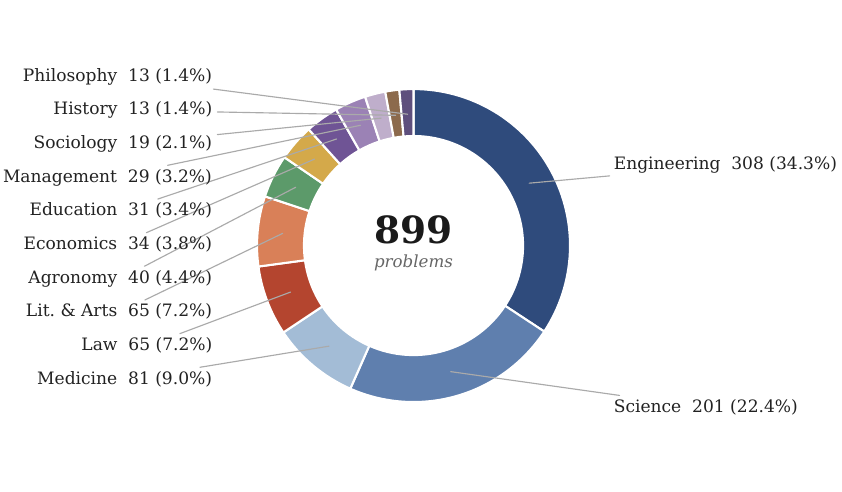}
\captionsetup{font=footnotesize}
\caption{Distribution of \name{} items across the $12$ top-level disciplines.}
\label{fig:kina_distribution}
\end{figure}

Figure~\ref{fig:kina_distribution} shows the distribution. Engineering ($34.26\%$) and Science ($22.36\%$) dominate, which reflects the larger number of sub-branches in those areas, but \name{} retains substantial coverage in Medicine ($9.01\%$), Law ($7.23\%$), Literature and Arts ($7.23\%$), and Economics, Education, Sociology, History, and Philosophy. Mean stem length is around $95$ tokens, and mean explanation length is around $210$ tokens.

%% file: 5_experiments.tex
\section{Experiments}
\label{sec:experiments}

We evaluate $42$ models from $13$ AI labs on \name{}, spanning closed-source flagship APIs, open-source dense and Mixture-of-Experts (MoE) checkpoints, and reasoning models. The full model list is in Appendix~\ref{appendix:model-list}. All results are reported as \textsc{avg@4} accuracy at default temperature, with $32$K maximum new tokens and a reasoning budget of $16$K (or ``Medium''); prompts are in Appendix~\ref{appendix:prompt}. We organize the discussion around four questions: how the leaderboard looks (\S\ref{subsec:overall}), whether the resulting ranking is stable (\S\ref{subsec:stability}), whether the tournament-review prediction holds in our own pipeline (\S\ref{subsec:tournament-audit}), and what the per-discipline and tool-use breakdowns reveal about the diagnostic content of \name{} (\S\ref{subsec:diagnostic}, \S\ref{subsec:scaling}).

\subsection{Overall accuracy}
\label{subsec:overall}

Table~\ref{tab:kina_top15} reports the full $42$ models on \name{}. Gemini-3.1-Pro-Preview leads at $53.17\%$ overall, securing the top score in $8$ of $12$ disciplines. Claude-Opus-4.6 ($49.92\%$) and GPT-5.4 ($48.55\%$) follow. Within the open-source ecosystem, Qwen3.5-397B-A17B ($42.99\%$) outperforms several closed-source systems, including GPT-5.2 ($39.52\%$) and Doubao-Seed-2.0-Lite ($41.49\%$). The leaderboard is therefore far from saturation: current model performance remains below $55\%$, and even the strongest closed-source frontier system fails on close to half of the items.

A closer look suggests that the full leaderboard is better understood as a tiered structure rather than a smooth total order. The top three models form a small frontier tier above $48\%$. Below them, a dense strong-model tier spans roughly $38.01\%$--$44.99\%$, from DeepSeek-V3.2-Thinking to Doubao-Seed-2.0-Pro-260215. This band contains both closed-source systems and large open-source or MoE models, suggesting that \name{} separates the frontier from the strong non-frontier group while still preserving resolution within the latter. However, the small gaps inside this band should not be over-interpreted as precise rank differences. The lower end of the leaderboard is also informative. Since each \name{} item has $10$ options, chance accuracy is $10\%$. Models in the $14\%$--$25\%$ range are therefore only modestly above random guessing in raw accuracy terms. For example, Qwen3-0.6B obtains $14.49\%$, Mixtral-8x7B-Instruct obtains $17.83\%$, Qwen3-1.7B obtains $18.33\%$, and Qwen3.5-2B obtains $20.52\%$. These results indicate that small or older models can occasionally solve discipline-specific items, but their performance is still far from robust knowledge mastery under the pseudo-MCQ format.

Two additional patterns are visible from the full table. Within the Qwen3 family, thinking-augmented reasoning is scale-dependent: at $80$B-A3B scale the Thinking variant ($28.28\%$) trails the Instruct variant ($30.09\%$), while at $30$B and $4$B the Thinking variants lead. We hypothesize that extended chain-of-thought may introduce noise on items requiring precise factual recall at large scale while compensating for weaker parametric knowledge at small scale. Claude-Opus-4.6 also shows a markedly skewed disciplinary profile: it ties for first in Philosophy ($36.54\%$) and Agronomy ($61.88\%$) but reaches only $15.38\%$ in History, illustrating why a per-discipline breakdown is more diagnostic than aggregate accuracy alone.

\providecommand{\hm}[1]{%
  \begingroup
  \pgfmathtruncatemacro{\KINAtint}{min(45,max(5,#1*0.65))}%
  \edef\KINAcellcmd{\noexpand\cellcolor{blue!\KINAtint!white}}%
  \KINAcellcmd #1%
  \endgroup
}

\begin{table*}[htbp]
\centering
\footnotesize
\setlength{\tabcolsep}{3pt}
\captionsetup{font=footnotesize}
\caption{Full leaderboard of 42 models on \name{}, sorted by overall accuracy. Cell shading scales with score (darker $=$ higher). Abbreviations: Agr.\ (Agronomy), Econ.\ (Economics), Edu.\ (Education), Eng.\ (Engineering), Hist.\ (History), Arts (Lit.\ \& Arts), Mgt.\ (Management), Med.\ (Medicine), Phil.\ (Philosophy), Sci.\ (Science), Soc.\ (Sociology). Bold $=$ best per column; underline $=$ second-best. Evaluation stability table in Appendix~\ref{appendix:full-results}.}
\label{tab:kina_top15}
\resizebox{0.99\textwidth}{!}{%
\begin{tabular}{lccccccccccccc}
\toprule
Model & ALL & Agr. & Econ. & Edu. & Eng. & Hist. & Law & Arts & Mgt. & Med. & Phil. & Sci. & Soc. \\
\midrule
Gemini-3.1-Pro-Preview & \textbf{\hm{53.17}} & \hm{46.25} & \underline{\hm{50.74}} & \textbf{\hm{52.42}} & \textbf{\hm{53.17}} & \hm{36.54} & \textbf{\hm{54.23}} & \textbf{\hm{65.00}} & \textbf{\hm{62.93}} & \textbf{\hm{49.69}} & \hm{25.00} & \textbf{\hm{52.99}} & \textbf{\hm{61.84}} \\
Claude-Opus-4.6 & \underline{\hm{49.92}} & \textbf{\hm{61.88}} & \textbf{\hm{52.21}} & \textbf{\hm{52.42}} & \underline{\hm{51.06}} & \hm{15.38} & \hm{44.62} & \underline{\hm{58.46}} & \hm{49.14} & \underline{\hm{41.98}} & \textbf{\hm{36.54}} & \hm{51.06} & \underline{\hm{43.42}} \\
GPT-5.4 & \hm{48.55} & \hm{45.62} & \hm{44.12} & \hm{41.13} & \hm{50.49} & \textbf{\hm{42.31}} & \underline{\hm{50.77}} & \hm{58.08} & \hm{41.38} & \hm{37.04} & \hm{30.77} & \underline{\hm{52.11}} & \hm{42.11} \\
Doubao-Seed-2.0-Pro-260215 & \hm{44.99} & \underline{\hm{56.88}} & \hm{44.85} & \hm{37.90} & \hm{44.97} & \hm{25.00} & \underline{\hm{50.77}} & \hm{32.69} & \hm{41.38} & \hm{41.67} & \hm{30.77} & \hm{50.50} & \hm{39.47} \\
Gemini-3-Flash-Preview & \hm{43.91} & \hm{50.62} & \hm{47.06} & \hm{35.48} & \hm{42.69} & \hm{19.23} & \hm{47.69} & \hm{45.77} & \underline{\hm{56.03}} & \hm{38.58} & \hm{21.15} & \hm{47.01} & \hm{42.11} \\
Qwen3.5-397B-A17B & \hm{42.99} & \hm{55.62} & \hm{43.38} & \hm{30.65} & \hm{44.32} & \hm{21.15} & \hm{37.31} & \hm{42.69} & \hm{43.97} & \hm{33.64} & \hm{19.23} & \hm{50.50} & \hm{25.00} \\
Doubao-Seed-2.0-Lite-260215 & \hm{41.49} & \hm{48.75} & \hm{41.18} & \hm{29.84} & \hm{42.13} & \hm{11.54} & \hm{43.08} & \hm{36.92} & \hm{43.97} & \hm{34.57} & \hm{30.77} & \hm{47.14} & \hm{39.47} \\
Kimi-K2.5 & \hm{40.24} & \hm{48.75} & \hm{39.71} & \hm{42.74} & \hm{38.88} & \hm{28.85} & \hm{50.00} & \hm{41.15} & \hm{42.24} & \hm{32.41} & \hm{15.38} & \hm{43.53} & \hm{25.00} \\
GPT-5.2 & \hm{39.52} & \hm{47.50} & \hm{34.56} & \hm{38.71} & \hm{41.72} & \hm{30.77} & \hm{38.85} & \hm{37.69} & \hm{30.17} & \hm{30.25} & \textbf{\hm{36.54}} & \hm{43.16} & \hm{28.95} \\
Qwen3.5-27B & \hm{39.35} & \hm{43.75} & \hm{35.29} & \hm{47.58} & \hm{40.26} & \hm{28.85} & \hm{39.62} & \hm{35.77} & \hm{39.66} & \hm{31.48} & \hm{19.23} & \hm{44.15} & \hm{23.68} \\
Qwen3.5-122B-A10B & \hm{38.88} & \hm{51.25} & \hm{32.35} & \hm{35.48} & \hm{40.10} & \hm{15.38} & \hm{41.54} & \hm{40.38} & \hm{42.24} & \hm{29.94} & \hm{11.54} & \hm{42.41} & \hm{26.32} \\
DeepSeek-V3.2-Thinking & \hm{38.01} & \hm{53.12} & \hm{44.12} & \hm{25.81} & \hm{36.93} & \underline{\hm{38.46}} & \hm{41.15} & \hm{32.69} & \hm{31.03} & \hm{37.96} & \hm{13.46} & \hm{41.92} & \hm{26.32} \\
Qwen3-Max-2025-09-23 & \hm{35.90} & \hm{43.12} & \hm{35.29} & \hm{34.68} & \hm{33.85} & \hm{28.85} & \hm{40.77} & \hm{32.31} & \hm{31.03} & \hm{34.57} & \hm{23.08} & \hm{40.92} & \hm{26.32} \\
GLM-5 & \hm{35.85} & \hm{35.00} & \hm{34.56} & \hm{42.74} & \hm{37.50} & \hm{32.69} & \hm{35.77} & \hm{39.23} & \hm{42.24} & \hm{32.10} & \hm{25.00} & \hm{33.83} & \hm{27.63} \\
Qwen3.5-35B-A3B & \hm{35.43} & \hm{37.50} & \hm{28.68} & \hm{36.29} & \hm{37.82} & \hm{26.92} & \hm{38.85} & \hm{33.08} & \hm{36.21} & \hm{30.56} & \hm{15.38} & \hm{37.06} & \hm{21.05} \\
Grok-4.1-Fast-Reasoning & \hm{33.73} & \hm{42.50} & \hm{22.79} & \hm{35.48} & \hm{34.42} & \hm{28.85} & \hm{26.54} & \hm{33.85} & \hm{36.21} & \hm{29.63} & \hm{17.31} & \hm{38.68} & \hm{21.05} \\
Qwen3-235B-A22B-Thinking-2507 & \hm{32.15} & \hm{40.62} & \hm{32.35} & \hm{27.42} & \hm{31.82} & \hm{30.77} & \hm{37.69} & \hm{30.77} & \hm{25.00} & \hm{34.88} & \hm{17.31} & \hm{32.34} & \hm{21.05} \\
Llama-4-Maverick-17B & \hm{31.62} & \hm{36.88} & \hm{31.62} & \hm{29.03} & \hm{32.06} & \hm{11.54} & \hm{20.77} & \hm{31.54} & \hm{32.76} & \hm{31.48} & \hm{21.15} & \hm{35.95} & \hm{28.95} \\
Minimax-M2.5 & \hm{30.28} & \hm{43.12} & \hm{24.26} & \hm{25.81} & \hm{31.66} & \hm{32.69} & \hm{31.15} & \hm{23.46} & \hm{20.69} & \hm{26.23} & \hm{13.46} & \hm{33.71} & \hm{25.00} \\
Qwen3.5-9B & \hm{30.09} & \hm{37.50} & \hm{23.53} & \hm{36.29} & \hm{29.55} & \hm{25.00} & \hm{30.38} & \hm{26.92} & \hm{30.17} & \hm{28.09} & \hm{15.38} & \hm{32.96} & \hm{26.32} \\
Qwen3-Next-80B-A3B-Instruct & \hm{30.09} & \hm{40.62} & \hm{30.88} & \hm{24.19} & \hm{28.00} & \hm{36.54} & \hm{35.77} & \hm{30.38} & \hm{21.55} & \hm{33.95} & \hm{23.08} & \hm{31.59} & \hm{10.53} \\
Claude-Sonnet-4.6 & \hm{30.01} & \hm{32.50} & \hm{27.94} & \hm{32.26} & \hm{30.03} & \hm{15.38} & \hm{22.69} & \hm{39.62} & \hm{22.41} & \hm{21.91} & \hm{28.85} & \hm{34.45} & \hm{26.32} \\
Llama-3.1-405B-Instruct & \hm{29.59} & \hm{40.62} & \hm{29.41} & \hm{38.71} & \hm{28.81} & \hm{32.69} & \hm{20.77} & \hm{28.85} & \hm{27.59} & \hm{24.38} & \hm{15.38} & \hm{33.83} & \hm{25.00} \\
Qwen3-235B-A22B & \hm{29.37} & \hm{45.00} & \hm{36.76} & \hm{25.81} & \hm{28.81} & \hm{32.69} & \hm{30.38} & \hm{29.62} & \hm{29.31} & \hm{26.54} & \hm{19.23} & \hm{28.36} & \hm{21.05} \\
Step-3.5-Flash & \hm{29.12} & \hm{41.25} & \hm{27.21} & \hm{21.77} & \hm{28.98} & \underline{\hm{38.46}} & \hm{36.15} & \hm{27.69} & \hm{17.24} & \hm{30.86} & \hm{13.46} & \hm{28.23} & \hm{26.32} \\
Qwen3.5-4B & \hm{28.50} & \hm{34.38} & \hm{27.94} & \hm{36.29} & \hm{28.57} & \hm{19.23} & \hm{26.92} & \hm{25.38} & \hm{32.76} & \hm{24.69} & \hm{21.15} & \hm{30.47} & \hm{19.74} \\
Qwen3-Next-80B-A3B-Thinking & \hm{28.28} & \hm{39.38} & \hm{28.68} & \hm{22.58} & \hm{29.14} & \hm{34.62} & \hm{31.15} & \hm{25.00} & \hm{17.24} & \hm{32.41} & \hm{13.46} & \hm{26.49} & \hm{25.00} \\
Qwen3-32B & \hm{27.41} & \hm{40.00} & \hm{30.15} & \hm{32.26} & \hm{26.81} & \hm{32.69} & \hm{25.77} & \hm{23.08} & \hm{25.86} & \hm{27.16} & \hm{19.23} & \hm{28.14} & \hm{15.79} \\
Qwen3-30B-A3B-Thinking-2507 & \hm{27.03} & \hm{39.38} & \hm{31.62} & \hm{18.55} & \hm{26.70} & \hm{19.23} & \hm{29.62} & \hm{19.62} & \hm{16.38} & \hm{34.57} & \hm{23.08} & \hm{26.74} & \hm{23.68} \\
Llama-3-70B-Instruct & \hm{26.61} & \hm{23.75} & \hm{17.65} & \hm{29.03} & \hm{28.57} & \hm{32.69} & \hm{20.77} & \hm{25.77} & \hm{20.69} & \hm{22.22} & \hm{28.85} & \hm{28.50} & \hm{38.16} \\
Qwen3-14B & \hm{26.00} & \hm{37.50} & \hm{33.82} & \hm{26.61} & \hm{25.32} & \hm{15.38} & \hm{23.85} & \hm{21.54} & \hm{24.14} & \hm{27.16} & \hm{13.46} & \hm{27.49} & \hm{18.42} \\
Qwen2-72B-Instruct & \hm{24.92} & \hm{35.00} & \hm{26.47} & \hm{20.16} & \hm{21.02} & \hm{23.08} & \hm{20.00} & \hm{31.54} & \hm{32.76} & \hm{20.37} & \textbf{\hm{36.54}} & \hm{28.36} & \hm{30.26} \\
Qwen3-4B-Thinking-2507 & \hm{24.83} & \hm{31.25} & \hm{33.09} & \hm{24.19} & \hm{23.13} & \hm{11.54} & \hm{21.15} & \hm{21.92} & \hm{22.41} & \hm{29.94} & \hm{13.46} & \hm{26.87} & \hm{25.00} \\
Qwen3-30B-A3B & \hm{24.28} & \hm{44.38} & \hm{30.88} & \hm{18.55} & \hm{22.89} & \hm{23.08} & \hm{25.00} & \hm{23.08} & \hm{24.14} & \hm{21.91} & \hm{15.38} & \hm{24.88} & \hm{14.47} \\
Qwen2.5-72B-Instruct & \hm{23.03} & \hm{25.62} & \hm{29.41} & \hm{32.26} & \hm{21.43} & \hm{28.85} & \hm{22.31} & \hm{24.62} & \hm{25.86} & \hm{21.60} & \hm{19.23} & \hm{22.89} & \hm{15.79} \\
Qwen3-8B & \hm{22.27} & \hm{31.25} & \hm{25.74} & \hm{28.23} & \hm{20.05} & \hm{19.23} & \hm{26.15} & \hm{19.62} & \hm{21.55} & \hm{21.91} & \hm{11.54} & \hm{23.63} & \hm{17.11} \\
Qwen3-4B & \hm{21.50} & \hm{25.00} & \hm{24.26} & \hm{24.19} & \hm{20.29} & \hm{25.00} & \hm{19.62} & \hm{20.38} & \hm{26.72} & \hm{17.59} & \hm{25.00} & \hm{23.88} & \hm{13.16} \\
Qwen3.5-2B & \hm{20.52} & \hm{21.25} & \hm{19.12} & \hm{29.84} & \hm{20.13} & \hm{26.92} & \hm{14.62} & \hm{18.46} & \hm{25.00} & \hm{17.59} & \hm{7.69} & \hm{22.76} & \hm{26.32} \\
Qwen3-1.7B & \hm{18.33} & \hm{16.25} & \hm{21.32} & \hm{22.58} & \hm{17.29} & \hm{15.38} & \hm{18.85} & \hm{19.23} & \hm{17.24} & \hm{16.98} & \hm{17.31} & \hm{19.53} & \hm{19.74} \\
Mixtral-8x7B-Instruct & \hm{17.83} & \hm{13.12} & \hm{16.91} & \hm{17.74} & \hm{17.86} & \hm{15.38} & \hm{11.54} & \hm{15.77} & \hm{29.31} & \hm{15.12} & \hm{17.31} & \hm{21.02} & \hm{19.74} \\
Qwen3.5-0.8B & \hm{16.66} & \hm{11.88} & \hm{18.38} & \hm{29.03} & \hm{16.72} & \hm{34.62} & \hm{10.77} & \hm{19.23} & \hm{19.83} & \hm{13.58} & \hm{28.85} & \hm{15.42} & \hm{14.47} \\
Qwen3-0.6B & \hm{14.49} & \hm{10.62} & \hm{13.24} & \hm{17.74} & \hm{13.15} & \hm{11.54} & \hm{16.54} & \hm{19.62} & \hm{22.41} & \hm{10.80} & \hm{19.23} & \hm{15.30} & \hm{10.53} \\
\bottomrule
\end{tabular}%
}
\end{table*}

\subsection{Ranking stability under bounded test budgets}
\label{subsec:stability}

A central concern with compact benchmarks is whether observed gaps between frontier models survive resampling. We address this directly with a nonparametric bootstrap: for each of $B=1{,}000$ replications, we draw a stratified random subsample of $\rho \cdot 899$ items (stratified by discipline, $\rho \in \{0.5, 0.7, 0.9\}$), recompute every model's accuracy, and recompute the top-$10$ ranking. We report the average Kendall $\tau$ between the bootstrap ranking and the full-sample ranking, and the rank-$1$ retention rate.

\begin{table}[htbp]
\centering
\captionsetup{font=footnotesize}
\caption{Ranking stability under stratified subsampling. Top-$10$ Kendall $\tau$ averaged over $1{,}000$ replications, with rank-$1$ retention. Numbers reflect a preliminary run on the released test set.}
\label{tab:stability}
\footnotesize
\setlength{\tabcolsep}{3pt}
\begin{tabular}{lcc}
\toprule
Subsample rate $\rho$ & Top-$10$ Kendall $\tau$ & Rank-1 retention \\
\midrule
$0.50$ & $0.89$ & $0.94$ \\
$0.70$ & $0.93$ & $0.98$ \\
$0.90$ & $0.97$ & $1.00$ \\
\bottomrule
\end{tabular}%
\end{table}

Table~\ref{tab:stability} shows that top-$10$ rankings remain highly stable even at $\rho=0.5$, and that Gemini-3.1-Pro-Preview retains rank-$1$ in at least $94\%$ of replications. Pairwise gaps below roughly $2$ percentage points are not statistically resolvable at the $95\%$ level, so we caution against fine-grained rank claims among models in the compressed mid-tier (e.g., Doubao-Lite vs.\ Kimi-K2.5). The stability analysis provides a quantitative warrant for the compact-benchmark design without overclaiming resolution that a $899$-item budget cannot support.

\subsection{Empirical validation of the tournament mechanism}
\label{subsec:tournament-audit}

We test whether the prediction of Theorem~\ref{thm:tournament}---that tournament review increases caught-flaw rates over flat-payment review---holds in our own pipeline. We compare two operating regimes from internal logs: a flat-payment pilot phase (the first $\sim$$1{,}000$ submitted items, no tournament, no audit) and a subsequent tournament phase (with bonus and audit). We measure three quantities. The reviewer-asymmetric catch rate is the fraction of flaws identified by reviewer A but missed by reviewer B; under flat payment with identical-effort baselines this should sit near $50\%$ of all flaws under random allocation, while under tournament the asymmetric incentive should drive both reviewers toward effort, narrowing the gap. The audit-flagged rate is the fraction of items approved by both reviewers but later flagged in the principal's stochastic audit; tournament should reduce this rate. The total caught-flaw rate at Stage~2 is the fraction of items revised or rejected.

\begin{table}[htbp]
\centering
\captionsetup{font=footnotesize}
\caption{Tournament vs.\ flat-payment review (in-house logs). Higher caught-flaw rate is better. Audit-flagged rate is the fraction of twice-approved items that fail audit; lower is better. Preliminary numbers from internal logs.}
\label{tab:tournament-audit}
\footnotesize
\setlength{\tabcolsep}{3pt}
\newcolumntype{C}{>{\centering\arraybackslash}X}
\begin{tabular}{lcc}
\toprule
Metric & Flat-payment phase & Tournament phase \\
\midrule
Caught-flaw rate (Stage 2)        & $0.41$  & $0.58$  \\
Reviewer-asymmetric catch rate    & $0.49$  & $0.32$  \\
Audit-flagged rate (post-Stage 2) & $0.087$ & $0.034$ \\
\bottomrule
\end{tabular}
\end{table}

Table~\ref{tab:tournament-audit} shows that the tournament phase increases the caught-flaw rate from $41\%$ to $58\%$ and reduces the audit-flagged rate by roughly a factor of $2.5$. The reviewer-asymmetric catch rate falls under tournament, consistent with both reviewers exerting more effort. We emphasize that this comparison is observational rather than randomized: the two phases are not interleaved, and we cannot fully rule out time trends or annotator-pool composition shifts. A randomized A/B design and a wider replication across disciplines are left to future work; the present result is best read as evidence consistent with Theorem~\ref{thm:tournament} rather than as a causal verification.

\subsection{Diagnostic findings: tool use and discipline structure}
\label{subsec:diagnostic}

The accuracy table alone hides two findings that we view as the principal diagnostic value of \name{}: how much each model gains from web search, and where in the disciplinary spectrum frontier models actually separate from one another. We unpack each in turn.

\paragraph{Tool-use yields consistent but non-uniform gains.}
Table~\ref{tab:kina_tool_use_scores} reports tool-use evaluation on five frontier models, each equipped with the provider's native web search under unlimited interaction turns. Web search yields universally positive gains, ranging from $+1.50$ to $+5.17$ points. The pattern is non-monotonic in base capability: both the weakest model in the cohort (GPT-5.2) and the strongest (Gemini-3.1-Pro-Preview) gain the most, while the GPT-5.4 variants gain the least. We read this as evidence that retrieval serves two distinct functions on \name{}---filling parametric gaps in lower-capability models, and supplying grounding evidence that highly capable models can synthesize into verified reasoning chains---rather than as a strict capability multiplier. The cohort is small, so the U-shape should be confirmed at scale before being treated as universal; we restrict the tool-use cohort to five models because each tool-use run is roughly $4\times$ the inference cost of a base run, and broader sweeps are reserved for the public leaderboard. Future leaderboards may benefit from reporting tool-use efficiency (accuracy gained per search query) as an independent capability dimension.




\begin{table}[t]
\centering
\captionsetup{font=footnotesize,skip=3pt}
\caption{Performance comparison of direct inference and tool-use inference.}
\label{tab:kina_tool_use_scores}

\footnotesize
\setlength{\tabcolsep}{3.2pt}
\renewcommand{\arraystretch}{0.95}

\begin{tabular}{@{}lccc@{}}
\toprule
\textbf{Model} & \textbf{Direct} & \textbf{Web} & \textbf{$\Delta$} \\
\midrule
GPT-5.4-Med. & $48.55{\pm}0.56$ & $50.33{\pm}1.23$ & $+1.78$ \\
GPT-5.4-High & $48.86{\pm}0.55$ & $50.36{\pm}1.75$ & $+1.50$ \\
GPT-5.2 & $39.52{\pm}0.99$ & $43.66{\pm}1.11$ & $+4.14$ \\
Gemini-3.1-Pro & $53.17{\pm}0.72$ & $58.34{\pm}0.42$ & $+5.17$ \\
Claude-Opus-4.6 & $49.92{\pm}0.32$ & $52.53{\pm}0.74$ & $+2.61$ \\
\bottomrule
\end{tabular}
\vspace{-0.4em}
\end{table}

\paragraph{Humanities and social-science content drives discrimination at the frontier.}
Table~\ref{tab:discrimination} reports per-discipline statistics over the top-$10$ models. The STEM-oriented disciplines---Science and Engineering---show the smallest spreads, with $\Delta=9.83$ and $14.29$ respectively; this is consistent with the hypothesis that their content is densely represented in pre-training corpora and amenable to systematic reasoning, so frontier models converge in performance. The humanities and social-science contents show the opposite pattern: Sociology ($\Delta=38.16$), Management ($32.76$), and Literature \& Arts ($32.31$) display spreads three to four times larger. History combines very low mean accuracy ($25.96\%$) with high CV ($34.63$), suggesting it remains a persistent challenge across the model spectrum rather than a low-tier-only problem. At the frontier, performance differentiation on \name{} is dominated by humanities and social-science content rather than STEM-oriented content. The phenomenon is suggestive rather than conclusive---disciplines with the largest $\Delta$ also have the smallest item counts (Sociology has $19$ items, Philosophy $13$, History $13$), so part of the spread is sample variance---but the contrast is large enough, and one-sided enough, that it is worth surfacing as an empirical pattern that compact benchmarks can discover.

\begin{table*}[t]
\centering
\captionsetup{font=footnotesize}
\caption{Discrimination indices over the top-$10$ models. Mean, SD, CV computed over the top-$10$ models ranked by overall accuracy. $\Delta=\mathrm{Max}-\mathrm{Min}$. Largest $\Delta$ in bold.}
\label{tab:discrimination}
\small
\renewcommand{\arraystretch}{1.1}
\begin{tabular}{lrrrrrr}
\toprule
Discipline & Mean Acc.\ & SD & CV & Max & Min & $\Delta$ \\
\midrule
Agronomy            & 50.56 & 5.47  & 10.81 & 61.88 & 43.75 & 18.13 \\
Economics           & 43.31 & 5.57  & 12.85 & 52.21 & 34.56 & 17.65 \\
Education           & 40.89 & 7.64  & 18.69 & 52.42 & 29.84 & 22.58 \\
Engineering         & 44.97 & 4.67  & 10.39 & 53.17 & 38.88 & 14.29 \\
History             & 25.96 & 8.99  & 34.63 & 42.31 & 11.54 & 30.77 \\
Law                 & 45.69 & 5.56  & 12.16 & 54.23 & 37.31 & 16.92 \\
Literature \& Arts  & 45.42 & 10.61 & 23.37 & 65.00 & 32.69 & 32.31 \\
Management          & 45.09 & 8.63  & 19.14 & 62.93 & 30.17 & 32.76 \\
Medicine            & 37.13 & 5.69  & 15.34 & 49.69 & 30.25 & 19.44 \\
Philosophy          & 26.54 & 7.19  & 27.08 & 36.54 & 15.38 & 21.16 \\
Science             & 48.21 & 3.50  & 7.27  & 52.99 & 43.16 & 9.83 \\
Sociology           & 37.10 & 11.19 & 30.17 & 61.84 & 23.68 & \textbf{38.16} \\
\bottomrule
\end{tabular}
\vspace{-0.4em}
\end{table*}

\subsection{Parameter scaling}
\label{subsec:scaling}

Figure~\ref{fig:scaling} shows scaling curves for the Qwen3 and Qwen3.5 families, dense and MoE. Three observations stand out. First, Qwen3.5 dense scales at roughly $14.8$ points per decade of total parameters, nearly double the $7.6$ points-per-decade slope of Qwen3; this is consistent with generational improvements in pre-training data and recipe contributing substantially to \name{} performance, although confounders such as post-training pipeline differences cannot be ruled out from this figure alone. Second, MoE models exceed dense models at matched active parameter counts, suggesting a benefit of sparse activation for knowledge-intensive tasks of the kind \name{} probes. Third, the scaling curve shows mild log-linear saturation at the upper end---Qwen3.5-397B-A17B at $42.99\%$ versus Qwen3.5-122B-A10B at $38.88\%$---indicating that parameter scaling alone may face diminishing returns on \name{}, and that the humanities and social-science gap surfaced in Table~\ref{tab:discrimination} will likely require methodological rather than purely scale-based progress.

\begin{figure}[htbp]
\centering
\includegraphics[width=\linewidth]{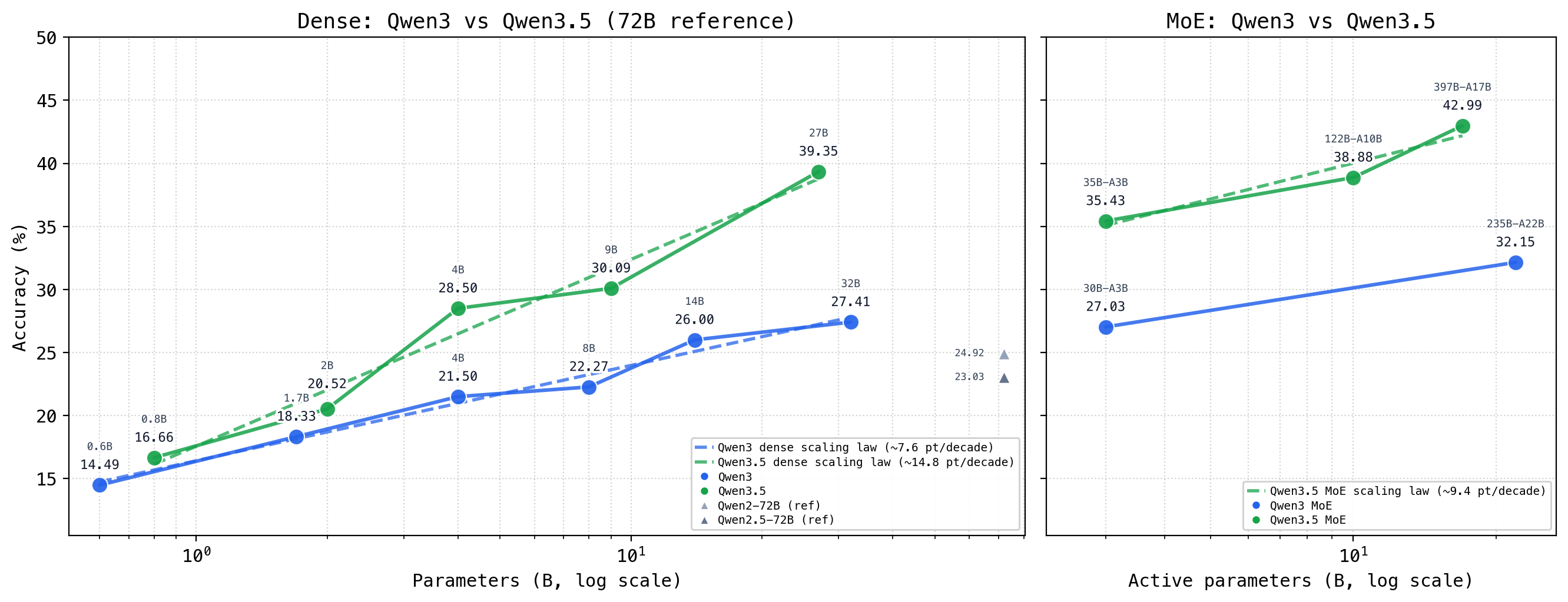}
\captionsetup{font=footnotesize}
\caption{Parameter scaling on \name{}. Left: dense models, accuracy vs.\ total parameters (log scale). Right: MoE models, accuracy vs.\ active parameters (log scale). Generation-over-generation slope increases from roughly $7.6$ (Qwen3) to $14.8$ (Qwen3.5) points per decade.}
\label{fig:scaling}
\end{figure}

%% file: 6_limitations.tex
\section{Limitations}
\label{sec:limitations}

We highlight five limitations.

\paragraph{Sample-size variance.} \name{} is intentionally compact. While
\S\ref{subsec:stability} shows that top-$10$ rankings remain stable under
$50\%$ subsampling (Kendall $\tau \approx 0.89$), pairwise gaps below
$\sim$$2$ percentage points are not statistically resolvable at the $95\%$
level. Fine-grained per-discipline claims with denominators $<$$30$ items
should be treated as suggestive rather than confirmatory.

\paragraph{Difficulty drift.} The 3-of-5 flagship-failure filter
(\S\ref{sec:construction}) couples \name{}'s difficulty distribution to the
capabilities of the five LLMs at construction time. As frontier models
improve, the absolute difficulty of \name{} will erode and the benchmark
will require periodic recalibration. We commit to releasing a
\name{}-v$x.0$ refresh whenever the strongest evaluated model exceeds
$70\%$ overall accuracy.

\paragraph{Subjectivity in representativeness elicitation.} The disciplinary
prototype $\Sigma_d$ used to anchor support centrality
(\S\ref{subsec:representativeness}) is elicited from human experts. Different
experts within the same discipline may produce non-identical $\Sigma_d$, and
we do not currently audit inter-expert agreement on $\Sigma_d$ itself.
Proposition~\ref{thm:submodular} guarantees a $(1-1/e)$ approximation of the
operational proxy $F_d^{\mathrm{sp}}$, not of an underlying population
representativeness quantity.

\paragraph{Tournament theory scope.} Theorem~\ref{thm:tournament} compares
tournament with flat payment under FOSD assumptions. It does not establish
truthfulness in the mechanism-design sense, nor does it rule out collusion
in the repeated game; collusion deterrence in our pipeline relies on
stochastic principal audits, whose effectiveness is empirical. The
empirical validation in \S\ref{subsec:tournament-audit} is observational, not
randomized.

\paragraph{Tool-use evaluation confounds.} We use each model provider's
native web-search tool. Search engine indices, backend processing, and
retrieval policies vary across providers and over time. Tool-use accuracy
should therefore be interpreted as effectiveness of the integrated system,
not the foundation model's intrinsic reasoning capacity.

\paragraph{Cultural and linguistic scope.} \name{}'s factuality bar
privileges English-language Q1 journals, monographs, CSSCI core journals,
and authoritative domain sources. Disciplines with regionally divergent
canons (e.g., comparative law, vernacular literature) may be
under-represented in this anchoring, even when the resulting items are
technically correct. We invite community contributions for non-English
discipline expansions.

%% file: 7_conclusion.tex
\section{Conclusion}

We presented \name{}, a knowledge benchmark designed to be diagnostic rather
than merely difficult. Its design is anchored by two formal results: a
$(1-1/e)$-approximation guarantee for greedy selection under budgeted support
centrality (Proposition~\ref{thm:submodular}), and a first-order stochastic
dominance improvement of released review quality under a bonus-on-bar
tournament mechanism over flat payment (Theorem~\ref{thm:tournament}). The
empirical evaluation of $42$ frontier models surfaces three findings useful
for benchmark methodology: (i)~ranking stability remains reasonably consistent under $50\%$
stratified subsampling, providing a quantitative warrant for the compact
design; (ii)~tool-use yields universally positive but non-uniform gains, with
a non-monotonic relationship between base capability and gain magnitude;
(iii)~discrimination at the frontier is dominated by humanities and
social-science content, not hard sciences---a phenomenon visible only under
fine-grained disciplinary breakdown.

We release \name{}, the annotation and reviewer manuals, the LLM-judge
rubrics, and the evaluation framework. We hope the formal-incentive
methodology travels to other benchmark efforts and that the
representativeness-as-submodular-coverage framing finds use beyond the
LLM-evaluation context.

%% file: 8_contributions_and_acknowledgements.tex
\section{Contributions and Acknowledgements}

KINA is developed through a collaboration among 2077AI, Multimodal Art Projection (M-A-P), the University of Tokyo, and Carnegie Mellon University. The collaboration combines community-driven benchmark development, open-source AI data infrastructure, and academic research expertise in evaluation methodology and empirical analysis.

Our team members contribute to the development of \name{} from the following perspectives:

\begin{center}
\begin{minipage}{0.42\textwidth}
\begin{itemize}[leftmargin=*]
    \item Benchmark design and disciplinary taxonomy
    \item Data collection and annotation management
    \item Data quality inspection and audit
\end{itemize}
\end{minipage}
\hfill
\begin{minipage}{0.42\textwidth}
\begin{itemize}[leftmargin=*]
    \item Incentive mechanism design
    \item Model evaluation and result analysis
    \item Paper writing and project release
\end{itemize}
\end{minipage}
\end{center}

%% file: appendix_a_samples.tex
\section{Data Samples}
\label{appendix:samples}

\begin{tcolorbox}[
    enhanced,
    breakable,
    colback=gray!5,
    colframe=blue!50!black,
    coltitle=white,
    title={Pseudo-Multi-Choice Sample (Some Content Omitted )},
    sharp corners=south,
    boxrule=0.8pt,
]
\textbf{Discipline: Engineering}

\textbf{Question:}\\
Which of the following options are correct?\\
1. $D\in\mathbb{R}^{6M\times 3N_c}$ and $M$ is a diagonal matrix. By defining $N = D^\top M^{-1}D$, it follows that $N\in\mathbb{R}^{3N_c\times 3N_c}$. The Schur complement is given by $S = N + \hat{M} - BE^{-1}C$. Since $\hat{M}$ is diagonal and $BE^{-1}C$ is a block-diagonal matrix with $3\times 3$ blocks, $S$ possesses the identical sparsity pattern as $N$.\\
2. In the contact pairs of granular dynamics, the friction cone constraint requires that the magnitude of the tangential force impulse vector does not exceed the product of the static friction coefficient and the normal force impulse. When the static friction coefficient is 0.3 and the normal force impulse is 10 N$\cdot$s, the maximum possible magnitude of the tangential force impulse vector is 3 N$\cdot$s, and this value satisfies the friction cone constraint.\\
3. The AMEN Cross complexity bound in this setting scales as $\mathcal{O}(r^3 N)$ (linear in the number of unknowns $N$). For $N=10^6$ and maximum TT-rank $r=10$, the computational cost is approximately $10^9$ operations; however, this contradicts the claim of "sublinear in $N$" (a linear scaling in $N$ cannot be sublinear).\\
4. Using the time-stepping update formula: $M\left(v^{k+1}-v^k\right)=\Delta t\,f_B + \sum_{i\in A(q^k,\delta)} D_i\gamma_i$. Given $\Delta t=0.1\ \text{s}$, $M=2I$, $v^k=3\ \text{m/s}$, $\Delta t f_B=4\ \text{N s}$, and $\sum D_i\gamma_i=6\ \text{N s}$, we calculate $v^{k+1}$ as follows:
\[
\begin{aligned}
v^{k+1} &= v^k + \frac{1}{M}\left(\Delta t\,f_B + \sum D_i\gamma_i\right) = 3 + \frac{1}{2}\left(4 + 6\right) = 8\ \text{m/s}.
\end{aligned}
\]\\
5. With $N = D^\top M^{-1}D$, for $M=2$ rigid bodies (yielding $6M=12$ degrees of freedom, DOF) and $N_c=1$ contact (yielding $3N_c=3$ multipliers), $N$ is a $3\times 3$ matrix, as it is defined by the product $D^\top M^{-1}D$ where $D\in\mathbb{R}^{12\times 3}$, $M^{-1}\in\mathbb{R}^{12\times 12}$, and thus acts on the contact multiplier space (not generalized-velocity space).\\
6. For the position-based normal complementarity constraint: $\gamma_{i,n}\ge 0,\ \Phi_i(q)\ge 0,\ \Phi_i(q)\gamma_{i,n}=0$. If $\Phi_i(q)=0.5\ \text{m}$ and $\gamma_{i,n}=2\ \text{N s}$, the constraint is \textbf{not} satisfied: while both quantities are nonnegative, their product is $1\ \text{m N s} \neq 0$, violating the complementarity condition ($\Phi_i(q)\gamma_{i,n}=0$).\\
7. For TT matrix-vector products used in the paper's TT-based preconditioner, the complexity is $\mathcal{O}(r^2 N\log N)$. While this complexity scales linearly in $N$ (up to a logarithmic factor $\log N$), the cost is still considered *asymptotically sublinear in $N$* when normalized by $N$—or more precisely, *sublinear in the sense of superlinear scaling avoidance*—because $\log N$ grows much slower than $N$ itself. When $N$ doubles, the cost increases by a factor of approximately $2\log(2N)/\log N \approx 2(1+\log 2/\log N)$, which remains close to 2 (linear scaling) but avoids the superlinear growth that would violate sublinearity claims.

\textbf{Options:}\\
A) 1,2\\
B) 2,3\\
C) 3,4\\
D) 4,5\\
E) 5,6\\
F) 6,7\\
G) 2,4\\
H) 3,5\\
I) 4,6\\
J) 5,7

\textbf{Explanations of A-J}\\
A) Statement 1 is correct in the Corona et al. setting: the excerpt defines $D\in\mathbb{R}^{6M\times 3N_c}$, uses a diagonal matrix $M$, and forms $N=D^\top M^{-1}D$ such that $N\in\mathbb{R}^{3N_c\times 3N_c}$. The excerpt further notes $\\hat{M}$ is diagonal and $BE^{-1}C$ is block-diagonal (with $3\times 3$ blocks), implying $S$ shares the sparsity pattern of $N$.\\
Statement 2: The friction cone constraint establishes that the magnitude of the tangential force impulse vector cannot exceed the product of the static friction coefficient ($\mu_i$) and the normal force impulse ($\gamma_{in}$). For $\mu_i$ = 0.3 and $\gamma_{in}$ = 10 N·s, the maximum allowable tangential force impulse magnitude is 0.3 * 10 = 3 N·s, which directly satisfies the constraint's requirement.\\
\textit{Explanations of other options are omitted}

\textbf{Sources of A-J:}\\
A)
Corona, E., Gorsich, D., Jayakumar, P., \& Veerapaneni, S. (2019). Tensor train accelerated solvers for nonsmooth rigid body dynamics. \textit{Applied Mechanics Reviews}, 71(5), 050804.\\
\includegraphics[width=0.18\linewidth]{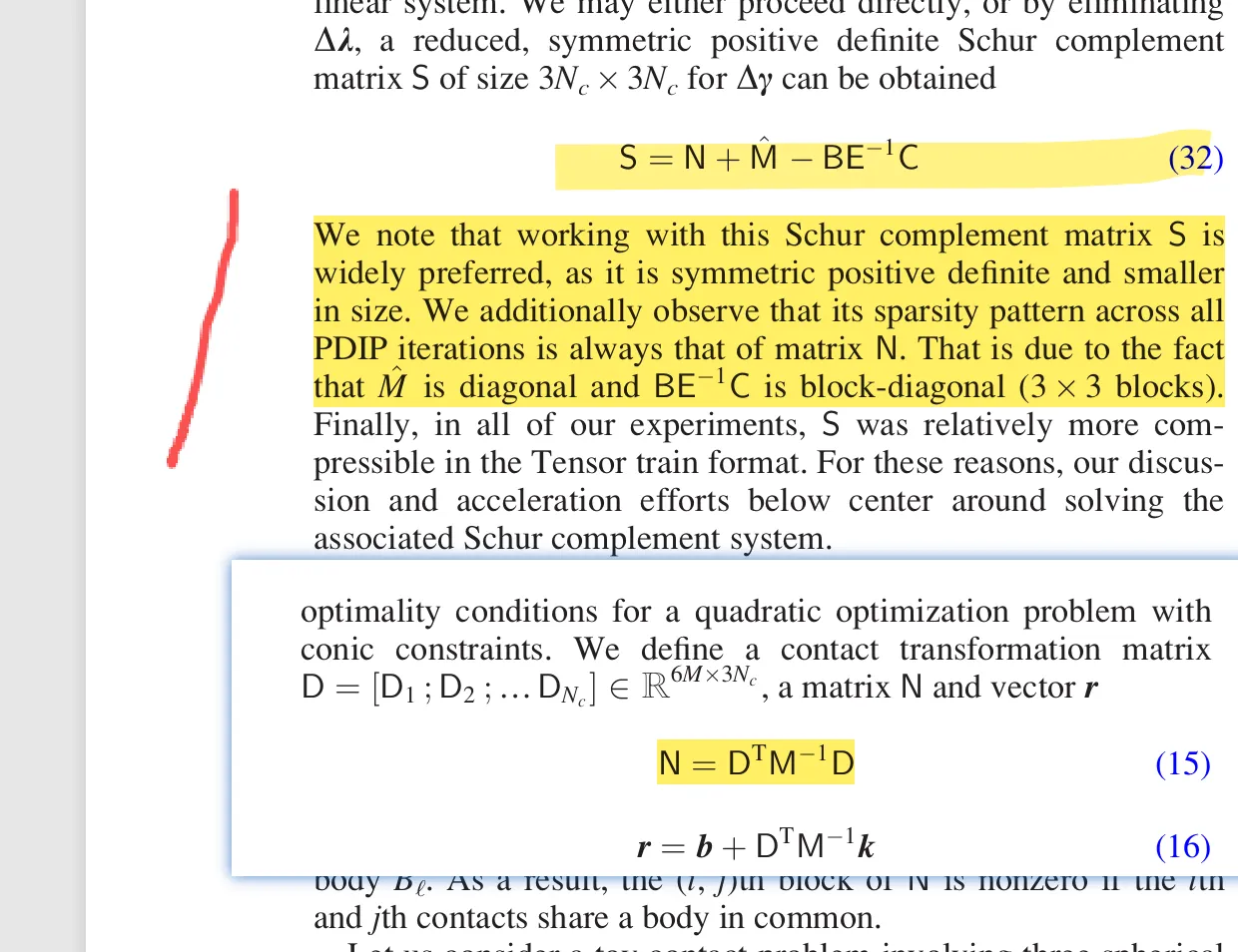}

\textit{Sources of other options are omitted}

\textbf{Question Source:}\\
Corona, E., Gorsich, D., Jayakumar, P., \& Veerapaneni, S. (2019). Tensor train accelerated solvers for nonsmooth rigid body dynamics. \textit{Applied Mechanics Reviews}, 71(5), 050804.

\textbf{Question Materail:}\\
Corona, E., Gorsich, D., Jayakumar, P., \& Veerapaneni, S. (2019). Tensor train accelerated solvers for nonsmooth rigid body dynamics. \textit{Applied Mechanics Reviews}, 71(5), 050804.\\
\includegraphics[width=0.12\linewidth]{sample1_fig1.png}
\includegraphics[width=0.12\linewidth]{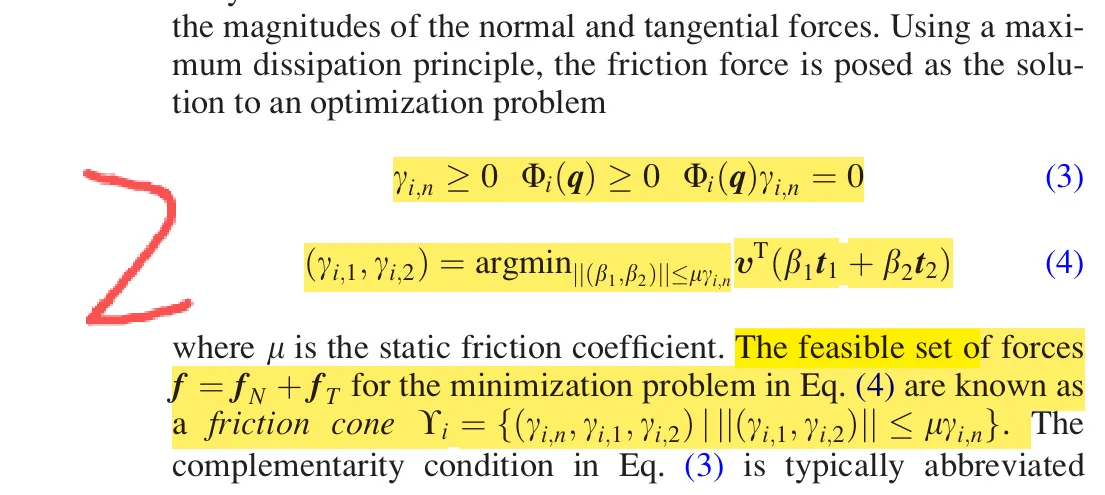}
\includegraphics[width=0.12\linewidth]{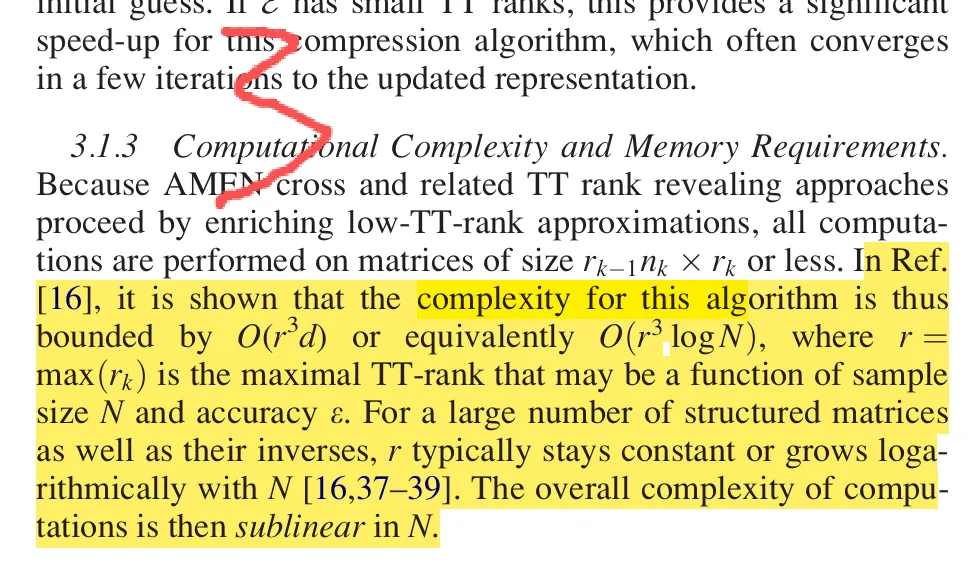}
\includegraphics[width=0.12\linewidth]{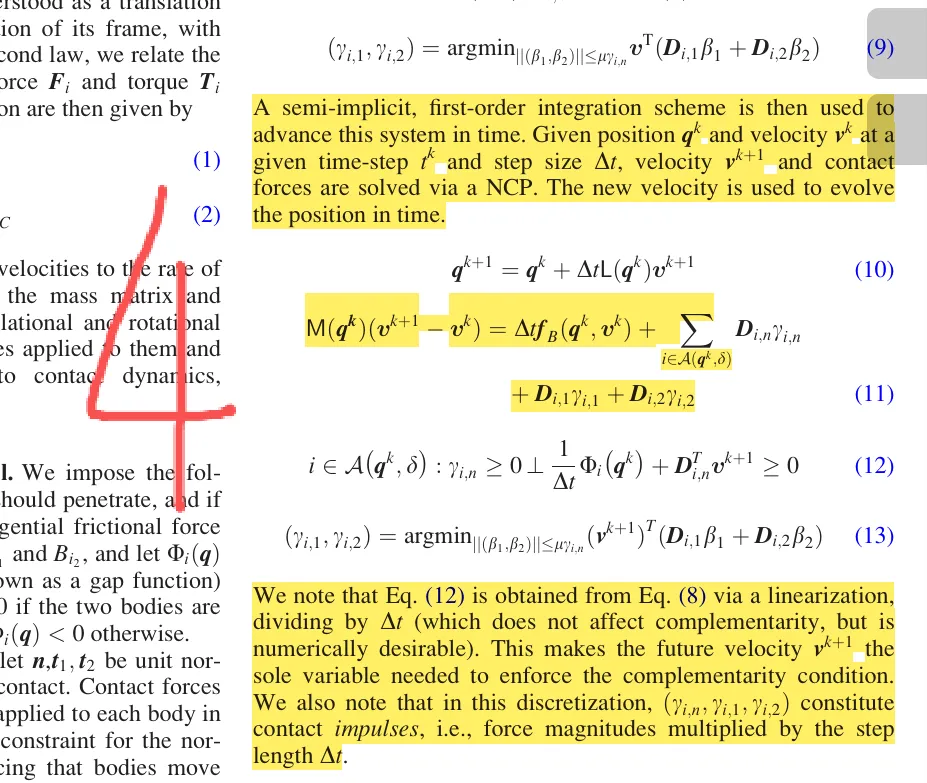}
\includegraphics[width=0.12\linewidth]{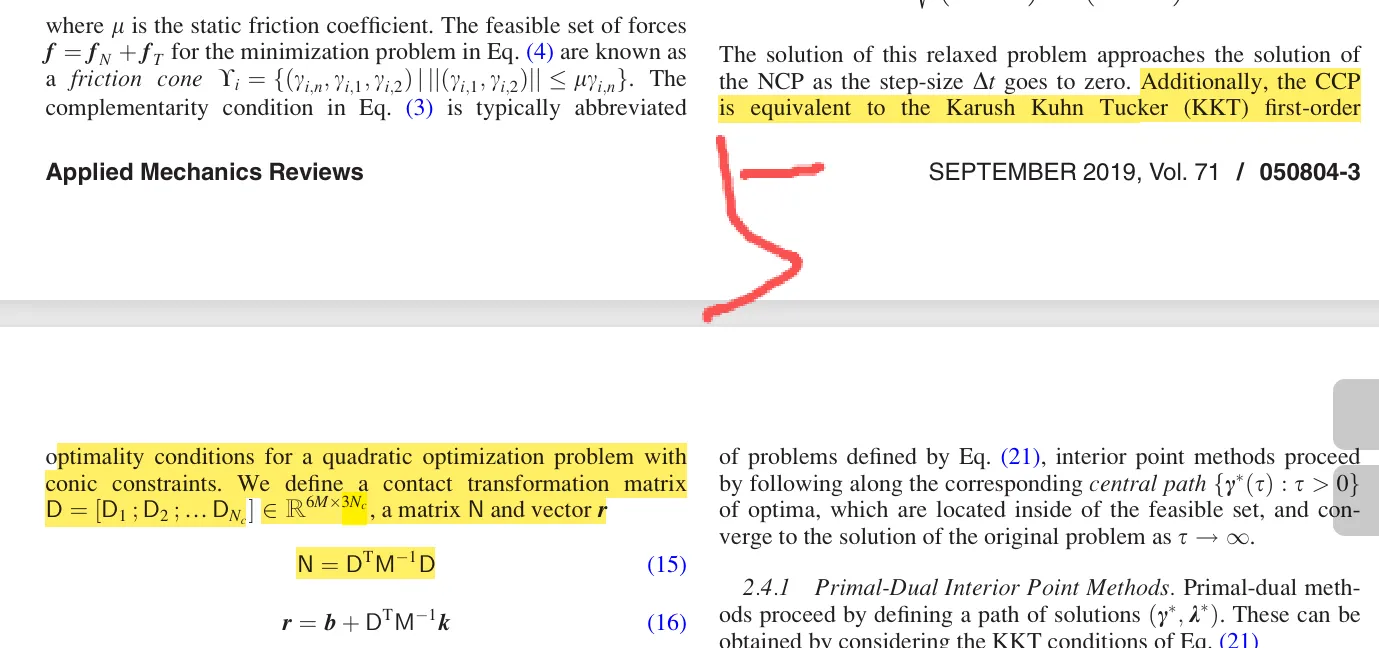}
\includegraphics[width=0.12\linewidth]{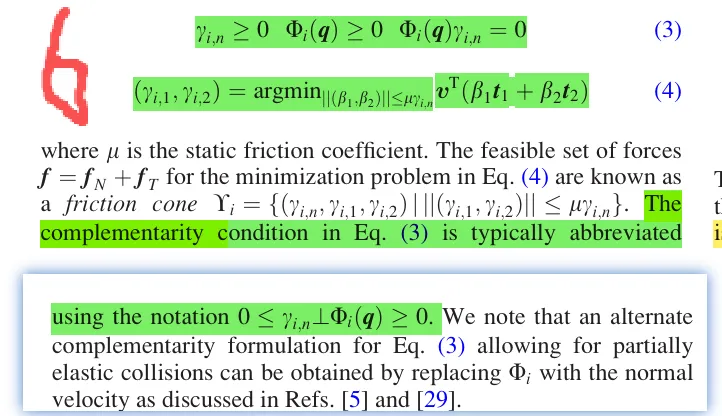}
\includegraphics[width=0.12\linewidth]{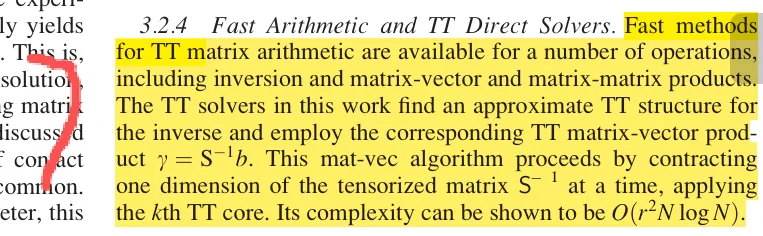}\\
\textbf{Correct Answer:} A
\end{tcolorbox}

\begin{tcolorbox}[
    enhanced,
    breakable,
    colback=gray!5,
    colframe=blue!50!black,
    coltitle=white,
    title={Standard Sample (Some Content Omitted )},
    sharp corners=south,
    boxrule=0.8pt,
]
\textbf{Discipline: Science}

\textbf{Question:}

In the extremely complex physical environment of the Galactic Bulge, observed Planetary Nebulae ($PN$) are typically subjected to severe interstellar reddening. Furthermore, high-excitation central stars ($T_{\mathrm{eff}} > 150,000$ K) can produce hard ultraviolet photons capable of doubly ionizing helium ($\mathrm{He^{++}}$). Accurately determining the ionizing photon production rate $Q(\mathrm{H^0})$ for such objects requires not only correcting for line-of-sight extinction but also verifying the ionization structure through radio continuum observations.

Observational Data and Analysis Clues:

Data 1 (Spectroscopy): The observed H$\beta$ flux is $F_{\mathrm{H\beta}}^{\mathrm{obs}} = 2.40 \times 10^{-13}$ erg s$^{-1}$ cm$^{-2}$. The observed Balmer decrement is $F_{\mathrm{H\alpha}}/F_{\mathrm{H\beta}} = 12.0$. Under the condition of $T_e = 12,000$ K, the theoretical Case B ratio is $2.85$. The extinction law follows $f(\mathrm{H\alpha}) = -0.35$ (reddening curve offset relative to H$\beta$).

Data 2 (Ionization Structure): The nebula shows high excitation. Spectral fitting yields a helium abundance of $y = n_{\mathrm{He}}/n_{\mathrm{H}} = 0.12$, with $40\%$ of helium in the doubly ionized state ($\mathrm{He^{++}}$) and the remainder in the singly ionized state ($\mathrm{He^{+}}$).

Data 3 (Radio Observation): The Very Large Array ($VLA$) measured a radio flux density of $S_{\nu} = 15.2$ mJy at $5$ GHz. The nebula is optically thin in the radio band.

Data 4 (Geometry and Environment): The distance to the nebula is $d = 8.1$ kpc. The electron temperature is measured at $T_e = 12,000$ K. Due to gravitational potential and evolutionary constraints, the total ionized gas mass $M_{\mathrm{gas}}$ must be less than $0.6 M_{\odot}$.

Physical Constants Reference:

At $T_e = 12,000$ K, $\alpha_B = 2.22 \times 10^{-13}$ cm$^3$ s$^{-1}$

At $T_e = 12,000$ K, $\alpha_{\mathrm{H\beta}}^{\mathrm{eff}} = 2.58 \times 10^{-14}$ cm$^3$ s$^{-1}$

Radio-to-H$\beta$ flux conversion formula (for $T_e \approx 10,000$-$15,000$ K):
\[
\begin{aligned}
Q(\mathrm{H^0}) &\approx 7.54 \times 10^{46} \times S_{\nu} \times d^2 \times T_e^{-0.45}
\end{aligned}
\]

Integrating the extinction correction and radio observation data, calculate the most robust hydrogen ionizing photon production rate $Q(\mathrm{H^0})$ for the central star of this nebula.

\textbf{Options:}\\
A) $5.60 \times 10^{49}s^{-1}$\\
B) $1.09 \times 10^{45}s^{-1}$\\
C) $3.38. \times 10^{45}s^{-1}$\\
D) $1.38 \times 10^{34}s^{-1}$\\
E) $1.36 \times 10^{48}s^{-1}$\\
F) $6.15 \times 10^{49}s^{-1}$\\
G) $3.94 \times 10^{44}s^{-1}$\\
H) $8.22 \times 10^{44}s^{-1}$\\
I) $3.17 \times 10^{49}s^{-1}$\\
J) $4.52 \times 10^{47}s^{-1}$

\textbf{Explanations of A-J}\\
B) Step 1: Calculate the Extinction Correction.

We use the difference between the observed and theoretical Balmer decrements to derive the reddening constant $c(\mathrm{H\beta})$. According to the formula:
\[
\begin{aligned}
\frac{F_{\mathrm{H\alpha}}}{F_{\mathrm{H\beta}}}
&= \left( \frac{F_{\mathrm{H\alpha}}}{F_{\mathrm{H\beta}}} \right)_{\text{theo}} \times 10^{c(\mathrm{H\beta}) }
\end{aligned}
\]

Substituting the data:
\[
\begin{aligned}
12.0 = 2.85 \times 10^{c(\mathrm{H\beta}) \times (0 - (-0.35))}
\end{aligned}
\]
\[
4.21 = 10^{0.35 \times c(\mathrm{H\beta})}
\]

Taking the logarithm:
\[
0.35 \times c(\mathrm{H\beta}) = \log_{10}(4.21) \approx 0.624
\]

Solving this gives $c(\mathrm{H\beta}) \approx 1.78$. Thus, the dereddened (true) H$\beta$ flux is:
\[
\begin{aligned}
F_{\mathrm{H\beta}}^{\text{dered}}
= F_{\mathrm{H\beta}}^{\mathrm{obs}} \times 10^{1.78} = 2.40 \times 10^{-13} \times 60.26
&\approx 1.446 \times 10^{-11} \text{ erg s}^{-1} \text{ cm}^{-2}
\end{aligned}
\]

\noindent{Step 2: Estimate $Q(\mathrm{H^0})$ Based on Optical Data.}

First, calculate the H$\beta$ luminosity:
\[
\begin{aligned}
L_{\mathrm{H\beta}}
= 4\pi d^2 F_{\mathrm{H\beta}}^{\text{dered}}
&\approx 4\pi \times (2.5 \times 10^{22} \text{ cm})^2 
\times 1.446 \times 10^{-11}
\approx 1.135 \times 10^{35} \text{ erg s}^{-1}
\end{aligned}
\]

The corresponding H$\beta$ photon emission rate:
\[
\begin{aligned}
N_{\mathrm{H\beta}}
= L_{\mathrm{H\beta}} / h\nu_{\mathrm{H\beta}}
\approx 2.78 \times 10^{46} \text{ s}^{-1}
\end{aligned}
\]
(where $h\nu_{\mathrm{H\beta}} \approx 4.08 \times 10^{-12}$ erg).

Using the ratio of recombination coefficients:
\[
\begin{aligned}
Q(\mathrm{H^0})_{\text{opt}}
= N_{\mathrm{H\beta}} \times (\alpha_B / \alpha_{\mathrm{H\beta}}^{\mathrm{eff}}) 
\approx 2.78 \times 10^{46} \times 8.60 
\approx 2.39 \times 10^{47} \text{ s}^{-1}
\end{aligned}
\]

\noindent{Step 3: Estimate $Q(\mathrm{H^0})$ Based on Radio Data.}

Radio continuum radiation originates from free-free emission and is unaffected by dust extinction, making it much more robust for objects buried in the Galactic Bulge. Using the provided radio conversion formula:
\[
\begin{aligned}
Q(\mathrm{H^0})_{\text{radio}}
= 7.54 \times 10^{46} \times S_{\nu} 
 \times d^2 \times T_e^{-0.45}
\end{aligned}
\]

Substituting the values (note $S_{\nu} = 15.2$ mJy $= 0.0152$ Jy):
\[
\begin{aligned}
Q(\mathrm{H^0})_{\text{radio}}
= 7.54 \times 10^{46} \times 0.0152 \times 8.1^2 \times (12,000)^{-0.45}
\end{aligned}
\]

Calculating the temperature term: $(12,000)^{-0.45} \approx 0.0145$.
\[
\begin{aligned}
Q(\mathrm{H^0})_{\text{radio}}
= 7.54 \times 10^{46} \times 0.0152
\times 65.61 \times 0.0145
\approx 1.09 \times 10^{45} \text{ s}^{-1}
\end{aligned}
\]

\noindent{Step 4: Physical Robustness Assessment and Conclusion.}

A significant discrepancy is found: $Q(\mathrm{H^0})_{\text{opt}}$ is much larger than $Q(\mathrm{H^0})_{\text{radio}}$. In the direction of the Galactic Bulge, due to variations in the local extinction law ($f(\lambda)$) and complex background scattering, fluxes corrected via the Balmer decrement often carry huge systematic errors. Conversely, the optically thin radio flux is directly proportional to the total number of ionizing photons and is independent of the filling factor or extinction models. Therefore, physically, the result derived from radio data, $1.09 \times 10^{45}$ s$^{-1}$, is considered reliable.

\textit{Explanations of other options are omitted}

\textbf{Sources of A-J:}\\
B)
Aksaker, N., Demirci, A., Erzincan, N., \& Akyuz, A. (2025). Photoionization Modeling of Planetary Nebulae in the Galactic Bulge. \textit{Advances in Space Research}\\
\includegraphics[width=0.18\linewidth]{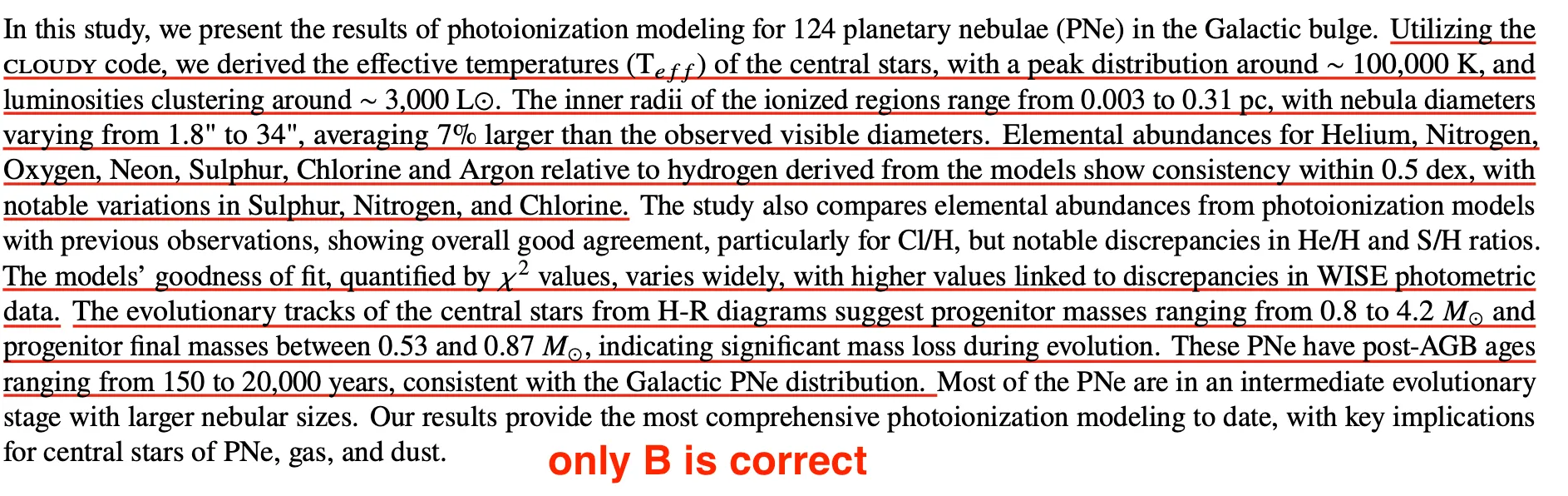}
\includegraphics[width=0.18\linewidth]{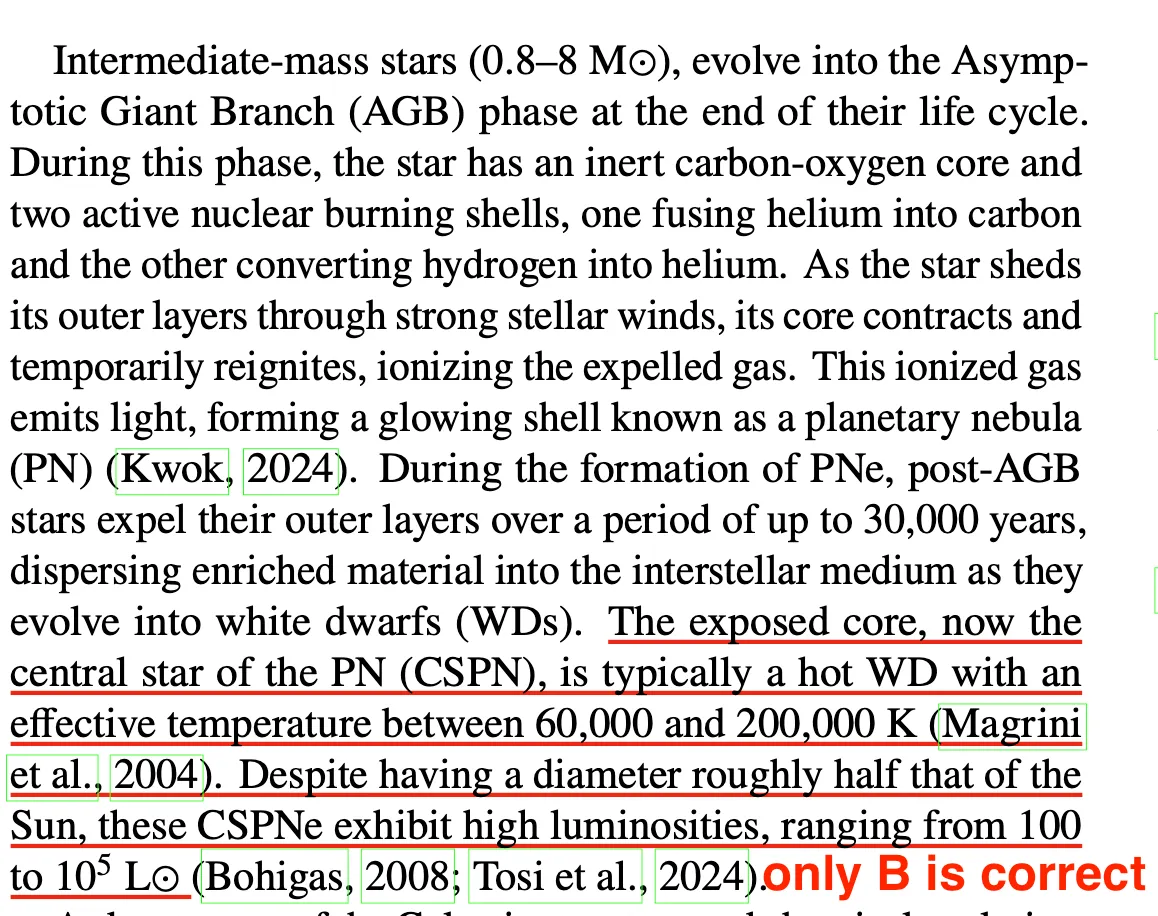}
\includegraphics[width=0.18\linewidth]{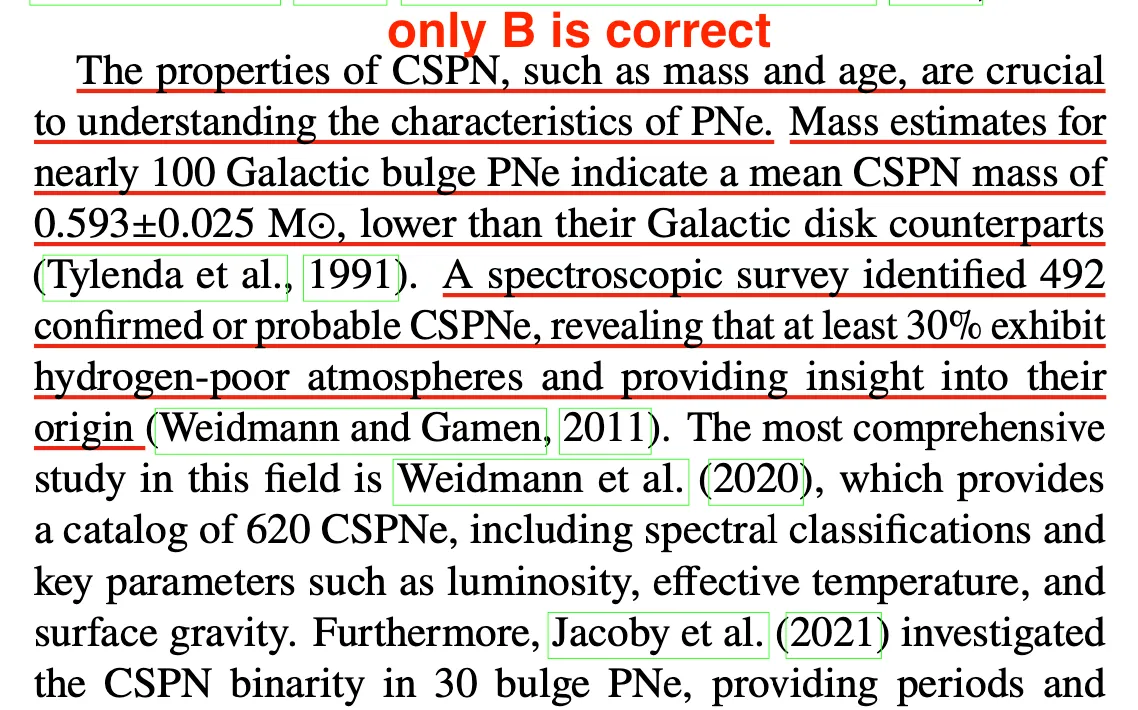}
\includegraphics[width=0.18\linewidth]{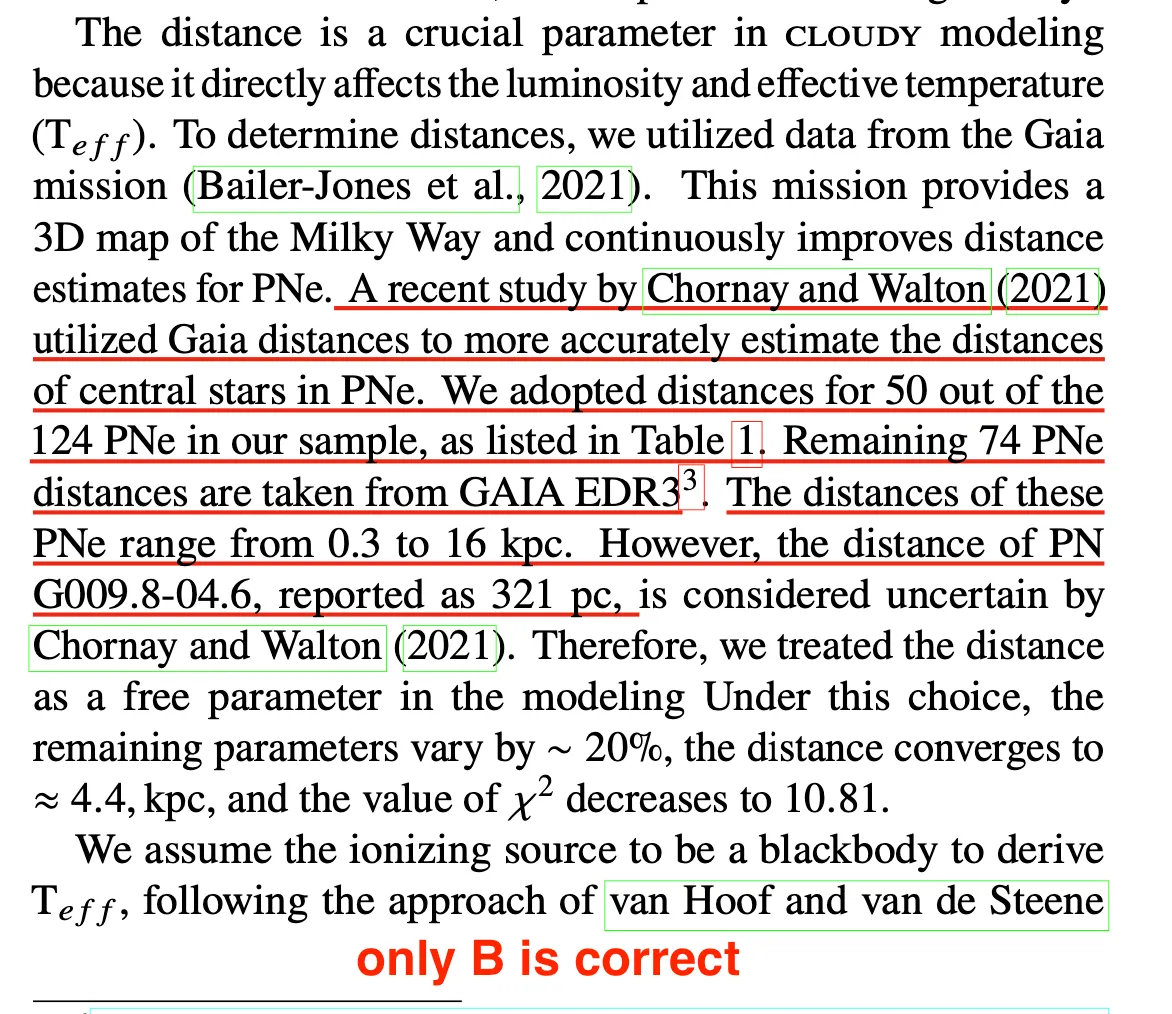}\\
\textit{Sources of other options are omitted}

\textbf{Question Source:}\\
Aksaker, N., Demirci, A., Erzincan, N., \& Akyuz, A. (2025). Photoionization Modeling of Planetary Nebulae in the Galactic Bulge. \textit{Advances in Space Research}\\
\url{https://www.sciencedirect.com/science/article/abs/pii/S0273117725011755}

\textbf{Question Materail:}\\
\includegraphics[width=0.18\linewidth]{sample2_fig1.png}
\includegraphics[width=0.18\linewidth]{sample2_fig2.png}
\includegraphics[width=0.18\linewidth]{sample2_fig3.png}
\includegraphics[width=0.18\linewidth]{sample2_fig4.png}

\textbf{Correct Answer:} B
\end{tcolorbox}

%% file: appendix_b_centrality.tex
\section{Budgeted support centrality under domain alignment}
\label{sec:appendix-support-centrality}

\paragraph{Scope.}
This appendix specifies a budgeted, operational proxy for domain support centrality used for filtering and final selection. Its role is to privilege items that are well supported by domain-aligned local structure, while reducing reliance on items that are idiosyncratic, weakly grounded, ambiguously specified, or near-duplicative. The proxy is not claimed to identify an underlying population quantity without additional modeling assumptions. Bootstrap lower percentiles are used as robustness-oriented summaries rather than formal lower confidence bounds. Complexity statements refer to LLM-call counts and sparse pair scoring under fixed hyperparameters. The structured signature \(\psi(q)\) is used as a low-cost operational representation; we do not assume that it exhausts the latent disciplinary structure.

\paragraph{Notation.}
For a domain \(d\), let the raw pool be \(\widetilde{\mathcal Q}_d=\{q_1,\dots,q_{N_d}\}\). Experts provide a compact prototype
\[
\Sigma_d=(\mathcal M_d,\mathcal P_d,\mathcal T_d,\mathcal C_d,\mathcal A_d),
\]
where \(\mathcal A_d\subset \widetilde{\mathcal Q}_d\) is a small anchor set used only as reference items (not as final candidates). Let \(h:\widetilde{\mathcal Q}_d\to\mathcal H_d\) be a coarse stratum map, \(\bar B_d\) a deduplicated reference bank, \(\mu_d\) a distribution on \(\bar B_d\), \(\mathcal N_k(q)\subseteq \bar B_d\setminus\{q\}\) a retrieved neighborhood (truncated if fewer than \(k\) items are available), \(\mathcal G(q):=\mathcal N_k(q)\cup\mathcal A_d\) the sparse comparison set, \(\mathcal K_d\) the accepted shortlist, and \(\mathcal S_d\subseteq\mathcal K_d\) the final set of size \(K_d\). If \(\bar B_d=\varnothing\) or \(\mathcal A_d=\varnothing\), we skip automatic selection and send the domain to manual review.

\paragraph{Item alignment and reference bank.}
A single item-level LLM call returns
\[
\begin{aligned}
(\psi(q),\widetilde A_d(q)) := \Gamma_{\theta,d}(q), 
\psi(q) &= (m_q,p_q,t_q,c_q),
\end{aligned}
\]
where \(m_q\) is the method family, \(p_q\) the core principle, \(t_q\) the answer type, and \(c_q\) salient constraints. With a small item label set
\[
\mathcal L_d^{\mathrm{item}}=\{(q,y_A(q))\},
\qquad
y_A(q)\in\{0,1\},
\]
we fit
\[
\hat A_d(q):=\mathrm{Cal}_A^{(d)}\!\big(\widetilde A_d(q),\psi(q),\Sigma_d\big)\in[0,1].
\]
For any derived score \(Z(q)\), let
\[
\mathrm{BootLQ}_{\delta}\!\big(\hat Z(q)\big)
:=
\operatorname{Quantile}_{\delta}\!\big(\{\hat Z^{(b)}(q)\}_{b=1}^B\big),
\]
where the bootstrap replicates refit the relevant calibrators and recompute downstream scores. We then define
\[
\begin{aligned}
B_d
&=
\Big\{
q\in \widetilde{\mathcal Q}_d:
\mathrm{BootLQ}_{\delta_A}\!\big(\hat A_d(q)\big)\ge \alpha_{\mathrm{ref}}
\Big\}, 
\bar B_d := \operatorname{Dedup}(B_d).
\end{aligned}
\]
For \(\bar B_{d,h}:=\{q\in\bar B_d:h(q)=h\}\), choose \(\omega_h>0\) with
\[
\sum_{h:\,|\bar B_{d,h}|>0}\omega_h=1,
\]
and set
\[
\mu_d(q)=
\begin{cases}
\dfrac{\omega_{h(q)}}{|\bar B_{d,h(q)}|}, & q\in \bar B_d,\\[1ex]
0, & q\notin \bar B_d.
\end{cases}
\]

\paragraph{Sparse pair scoring.}
Let \(x_q\) denote the text and \(e(q):=f_{\mathrm{emb}}(x_q,\psi(q))\). For each non-anchor item \(q\), retrieve
\[
\begin{aligned}
\mathcal N_k(q) := \operatorname{ANN}_k\!\big(e(q);\bar B_d\setminus\{q\}\big), 
\mathcal G(q) := \mathcal N_k(q)\cup\mathcal A_d.
\end{aligned}
\]
For every unordered pair \(\{q,u\}\) such that \(u\in \mathcal G(q)\) or \(q\in \mathcal G(u)\), we compute both directed features
\[
\begin{aligned}
\phi(q,u)
&=
\Phi\!\Big(
\psi(q),\psi(u), 
\operatorname{sim}_e(e(q),e(u)), 
\operatorname{sim}_x(x_q,x_u), 
\mathbf 1\{\operatorname{src}(q)=\operatorname{src}(u)\}
\Big).
\end{aligned}
\]
Using a small pair label set
\[
\mathcal L_d^{\mathrm{pair}}
=
\{((q,u),y_T(q,u),y_U(q,u))\},
\]
where \(y_T\) indicates shared domain support and \(y_U\) near-duplicate risk, we fit
\[
\begin{aligned}
\hat T_d^\rightarrow(q,u) := g_T^{(d)}\!\big(\phi(q,u)\big)\in[0,1],
\hat U_d^\rightarrow(q,u) := g_U^{(d)}\!\big(\phi(q,u)\big)\in[0,1].
\end{aligned}
\]
Symmetrizing both directions,
\[
\begin{aligned}
\hat T_d^{\mathrm{sym}}(q,u)
=
\frac{\hat T_d^\rightarrow(q,u)+\hat T_d^\rightarrow(u,q)}{2}, 
\hat U_d^{\mathrm{sym}}(q,u)
=
\frac{\hat U_d^\rightarrow(q,u)+\hat U_d^\rightarrow(u,q)}{2},
\end{aligned}
\]
we define, for \(q\neq u\),
\[
\hat S_d(q,u)
=
\big[
\hat T_d^{\mathrm{sym}}(q,u)
-
\lambda_{\mathrm{dup}}\,\hat U_d^{\mathrm{sym}}(q,u)
\big]_+.
\]
This is an operational proxy score rather than a probabilistic identity.

\paragraph{Single-item support surrogate.}
Let \(S_d^\star(q,u)\in[0,1]\) denote an unobserved shared-support quantity and
\[
C_d^\star(q)=\mathbb E_{Q\sim\mu_d}\big[S_d^\star(q,Q)\big].
\]
We do not estimate \(C_d^\star\) directly. For \(q\notin \mathcal A_d\), let
\[
m_q:=\sum_{u\in \mathcal N_k(q)}\mu_d(u),
\]
and define
\[
\begin{aligned}
\hat C_d(q)
&=
\sum_{u\in \mathcal N_k(q)} \mu_d(u)\,\hat S_d(q,u)
+ (1-m_q)\sum_{a\in\mathcal A_d}\pi_a\,\hat S_d(q,a),
\end{aligned}
\]
where \(\pi_a\ge 0\) and \(\sum_{a\in\mathcal A_d}\pi_a=1\). The second term is an anchor-based tail surrogate, included for budgeted robustness rather than formal unbiasedness.

\paragraph{Shortlisting and final selection.}
We keep a non-anchor item \(q\) iff
\[
\begin{aligned}
q\in \mathcal K_d
&\iff
\mathrm{BootLQ}_{\delta_A}\!\big(\hat A_d(q)\big)\ge \alpha_A 
\land\;
\mathrm{BootLQ}_{\delta_C}\!\big(\hat C_d(q)\big)\ge \alpha_C 
\land\;
D_d(q)\in \mathcal I_d.
\end{aligned}
\]
For coverage, define
\[
{\footnotesize
\hat S_d^{\mathrm{sp}}(q,u)=
\begin{cases}
1, & q=u\in \bar B_d,\\
\hat S_d(q,u), & q\neq u\ \text{and}
 u\in \mathcal G(q)\ \text{or}\ q\in \mathcal G(u),\\
0, & \text{otherwise}.
\end{cases}
}
\]
For a final set \(\mathcal S_d\subseteq \mathcal K_d\) of size \(K_d\), use
\begin{equation}
F_d^{\mathrm{sp}}(\mathcal S_d)
=
\sum_{u\in \bar B_d}\mu_d(u)\max_{q\in \mathcal S_d}\hat S_d^{\mathrm{sp}}(q,u).
\label{eq:Fsp_appendix}
\end{equation}
We optimize
\[
\max_{\mathcal S_d\subseteq \mathcal K_d} F_d^{\mathrm{sp}}(\mathcal S_d)
\quad\text{s.t.}\quad
|\mathcal S_d|=K_d,
\]
subject to
\[
L_h \le \sum_{q\in \mathcal S_d}\mathbf 1\{h(q)=h\} \le U_h,
\qquad \forall h\in\mathcal H_d,
\]
and
\[
\hat U_d^{\mathrm{sym}}(q,u) \le \tau_{\mathrm{dup}},
\qquad \forall q\neq u\in \mathcal S_d.
\]
The quota condition
\[
\sum_{h\in\mathcal H_d}L_h \le K_d \le \sum_{h\in\mathcal H_d}U_h
\]
is necessary but not sufficient for feasibility; feasibility also depends on \(\mathcal K_d\) and the duplicate constraint. Duplicate scores required by the last constraint are evaluated on demand on shortlisted pairs using the same cheap features and no additional LLM calls.

\begin{proposition}[Monotone submodularity of $F_d^{\mathrm{sp}}$]
\label{thm:submodular}
If \(\hat S_d^{\mathrm{sp}}(q,u)\ge 0\) for all \(q,u\), then \(F_d^{\mathrm{sp}}(\cdot)\) defined in~\eqref{eq:Fsp_appendix} is monotone submodular. Greedy maximization under the cardinality constraint \(|\mathcal S_d|=K_d\) therefore attains a $(1-1/e)$ approximation of the optimum of $F_d^{\mathrm{sp}}$ \citep{nemhauser1978submodular}.
\end{proposition}

\begin{proof}
For fixed \(u\in \bar B_d\), let \(f_u(\mathcal S):=\max_{q\in \mathcal S}\hat S_d^{\mathrm{sp}}(q,u)\) with \(f_u(\varnothing)=0\). For any \(A\subseteq B\) and \(x\notin B\),
\[
{\footnotesize
\begin{aligned}
f_u(A\cup\{x\})-f_u(A)
=
\max\big\{0, 
 \hat S_d^{\mathrm{sp}}(x,u)-f_u(A)\big\} 
\ge
\max\big\{0, 
 \hat S_d^{\mathrm{sp}}(x,u)-f_u(B)\big\} 
=
f_u(B\cup\{x\})-f_u(B),
\end{aligned}
}
\]
because \(f_u(A)\le f_u(B)\). Hence \(f_u\) is monotone submodular. Since
\[
\begin{aligned}
F_d^{\mathrm{sp}}(\mathcal S)
=
\sum_{u\in \bar B_d}\mu_d(u)\,f_u(\mathcal S), 
\mu_d(u) &\ge 0,
\end{aligned}
\]
\(F_d^{\mathrm{sp}}\) is a nonnegative weighted sum of monotone submodular functions and is therefore itself monotone submodular. The $(1-1/e)$ greedy guarantee then follows from \citet{nemhauser1978submodular}.
\end{proof}

\paragraph{Budget and implementation note.}
A typical domain uses two experts (about 10 total hours) for rubric alignment, item labels, pair labels, and audit, supporting roughly
\(|\mathcal L_d^{\mathrm{item}}|\approx 120\text{--}160\) and \(|\mathcal L_d^{\mathrm{pair}}|\approx 60\text{--}80\), with \(25\%\text{--}30\%\) overlap for agreement checks. We finally spent $\sim$\$100,000 for KINA.
Each item receives one main LLM call for \((\psi(q),\widetilde A_d(q))\), with optional re-checks on an uncertainty set \(\mathcal U_d\), so
\[
\begin{aligned}
B_{\mathrm{LLM}}^{(d)}
= N_d+|\mathcal U_d| 
&\le (1+\rho)N_d.
\end{aligned}
\]
The same signature \(\psi(q)\) is reused for alignment, retrieval, pair scoring, and duplicate checks; no second item-level LLM pass is required. With fixed \(k\), \(|\mathcal A_d|\), and bootstrap count \(B\), sparse retrieval and sparse support scoring scale linearly in \(N_d\). If all duplicate constraints on \(\mathcal K_d\) are materialized, the additional worst-case cost is \(O(|\mathcal K_d|^2)\). Final constrained selection is solved by greedy-with-repair as a heuristic, or by a small MIP when \(|\mathcal K_d|\) is moderate; we do not claim a general approximation guarantee under the full constraint set.

%% file: appendix_c_tournament.tex
\section{Comparative Incentives under a Noisy Tournament}
\label{appendix:theory}

\paragraph{Scope.}
This appendix makes a comparative claim relative to flat payment.
It does not prove truthfulness, rule out collusion, or solve the full repeated stochastic game.
The whitelist extension below is reduced-form: it compares pointwise effort incentives under a fixed continuation value, not endogenous stationary state distributions.
Throughout, \(X \fosd Y\) means first-order stochastic dominance:
\[
\begin{aligned}
X \fosd Y
&\Longleftrightarrow
F_X(t)\le F_Y(t), \qquad \forall t\in\mathbb{R}.
\end{aligned}
\]
Strict incentive conclusions additionally require \(\Delta p_{\min}>0\), which is not automatic if score noise is large or the quality bar is poorly calibrated.

\paragraph{Setup.}
Two reviewers \(i\in\{a,b\}\) evaluate the same item and choose effort \(e_i\in\{L,H\}\).
Let \(\psi_i(e_i)\) be reviewer \(i\)'s reduced-form private disutility of effort, and define the incremental cost of high effort by
\[
\kappa_i \triangleq \psi_i(H)-\psi_i(L).
\]
For \(e\in\{L,H\}\), let \(Q_i^e\in\mathbb{R}\) denote latent review quality with CDF \(F_e(q)\triangleq \Pr(Q_i^e\le q)\).

\begin{assumption}[Effort improves quality]
\label{ass:fosd}
$Q_i^H \fosd Q_i^L$, equivalently $F_H(q)\le F_L(q)$ for all $q$, with strict inequality on a set of positive measure.
\end{assumption}

\begin{assumption}[Noisy validated score]
\label{ass:noise}
The principal observes a noisy validated score
\[
S_i^e = g(Q_i^e)+\varepsilon_i,
\]
where \(g\) is strictly increasing and \(\varepsilon_i\) are i.i.d., continuous, and independent of all latent qualities. Let \(F_e^S(s)\triangleq \Pr(S_i^e\le s)\).
\end{assumption}

\begin{assumption}[Independent reviewer types]
\label{ass:types}
Conditional on the effort profile, reviewer qualities and score noises are independent across reviewers. Reviewer types \(\kappa_i\) are i.i.d.\ from a continuous CDF \(G\), and each reviewer observes only her own type.
\end{assumption}

\noindent Each reviewer receives base payment \(w_0\ge 0\), and a single bonus \(B>0\) is awarded to the reviewer with the higher validated score provided that the winning score clears a minimum bar \(\tau\):
\begin{equation}
W_i=w_0+B\cdot \mathbf{1}\{S_i>S_{-i},\; S_i\ge \tau\}.
\label{eq:wage}
\end{equation}
Define the winning probability \(p_i(e_i,e_{-i}) \triangleq \Pr(S_i>S_{-i},\; S_i\ge \tau\mid e_i,e_{-i})\), and the one-shot utilities
\[
\begin{aligned}
U_i^{\mathrm{tour}}(e_i;e_{-i})
= w_0+B\,p_i(e_i,e_{-i})-\psi_i(e_i), 
U_i^{\mathrm{flat}}(e_i)
= w_0-\psi_i(e_i).
\end{aligned}
\]

\paragraph{One-shot comparison.}
For any fixed opponent effort \(e_{-i}\in\{L,H\}\),
\[
F_e^S(s)=\Pr(S_i^e\le s)=\mathbb{E}\!\left[F_\varepsilon\!\big(s-g(Q_i^e)\big)\right].
\]
Since \(g\) and \(F_\varepsilon\) are increasing, the map \(q\mapsto F_\varepsilon(s-g(q))\) is decreasing, so Assumption~\ref{ass:fosd} together with Assumption~\ref{ass:noise} implies
\[
\begin{aligned}
S_i^H \fosd S_i^L,
\text{i.e.}
F_H^S(s)\le F_L^S(s) 
\forall s.
\end{aligned}
\]
Now define \(T_i\triangleq \max\{S_{-i}^{e_{-i}},\tau\}\).
Because scores are continuous, \(\{S_i>S_{-i},\;S_i\ge \tau\} = \{S_i>T_i\}\) almost surely.
Conditional on the effort profile, \(S_i^e\) is independent of \(T_i\). Hence
\[
p_i(e,e_{-i})
=
\Pr(S_i^e>T_i)
=
\mathbb{E}\!\left[1-F_e^S(T_i)\right].
\]
Therefore, the difference in winning probabilities is
\[
\begin{aligned}
\Delta p_i(e_{-i})
\triangleq
p_i(H,e_{-i})-p_i(L,e_{-i}) 
=
\mathbb{E}\!\left[F_L^S(T_i)-F_H^S(T_i)\right]
\ge 0.
\end{aligned}
\]
The one-shot gain from high effort is then
\[
\begin{aligned}
\Delta U_i^{\mathrm{tour}}(e_{-i})
\triangleq
U_i^{\mathrm{tour}}(H;e_{-i})-U_i^{\mathrm{tour}}(L;e_{-i}) 
=
B\,\Delta p_i(e_{-i})-\kappa_i,
\end{aligned}
\]
whereas under flat payment, \(\Delta U_i^{\mathrm{flat}} = U_i^{\mathrm{flat}}(H)-U_i^{\mathrm{flat}}(L) = -\kappa_i\).
Thus the tournament adds a positive extrinsic return \(B\,\Delta p_i(e_{-i})\), and high effort is weakly preferred whenever \(B\,\Delta p_i(e_{-i})\ge \kappa_i\).

\paragraph{Population implication and review-quality shift.}
By symmetry, write
\[
\begin{aligned}
\Delta p(L) \triangleq \Delta p_i(L),
\Delta p(H) \triangleq \Delta p_i(H),
\Delta p_{\min} \triangleq \min\{\Delta p(L),\Delta p(H)\}.
\end{aligned}
\]
Let \(\pi^m\) be the equilibrium high-effort rate under mechanism \(m\in\{\mathrm{flat},\mathrm{tour}\}\).

\begin{theorem}[Tournament FOSD improvement]
\label{thm:tournament}
Under Assumptions~\ref{ass:fosd}--\ref{ass:types}, if \(\Delta p_{\min}>0\), the equilibrium high-effort rate satisfies
\[
\pi^{\mathrm{tour}}
\;\ge\;
G(B\,\Delta p_{\min})
\;\ge\;
G(0)
\;=\;
\pi^{\mathrm{flat}}.
\]
Moreover, let \(Y^m=\Gamma(Q_a^m,Q_b^m,Z)\) denote released review quality under mechanism \(m\), where \(\Gamma\) is coordinatewise nondecreasing and \(Z\) is mechanism-invariant and independent of reviewer qualities. Then
\[
Y^{\mathrm{tour}} \fosd Y^{\mathrm{flat}}.
\]
\end{theorem}

\begin{proof}
If \(\Delta p_{\min}>0\), then for all opponent actions, \(\Delta U_i^{\mathrm{tour}}(e_{-i}) \ge B\,\Delta p_{\min}-\kappa_i\). Hence every type satisfying \(\kappa_i\le B\,\Delta p_{\min}\) weakly prefers high effort regardless of the opponent's action. Under flat payment, high effort is chosen only when \(\kappa_i\le 0\). Since \(G\) is continuous, indifferent types have zero mass, so
\[
\pi^{\mathrm{tour}} \ge G(B\,\Delta p_{\min}) \ge G(0) = \pi^{\mathrm{flat}}.
\]
Let \(Q^m\) denote the latent quality of a randomly sampled review under mechanism \(m\). Conditioning on the effort choice gives \(F_Q^m(q) = \pi^m F_H(q)+(1-\pi^m)F_L(q)\), so
\[
\begin{aligned}
F_Q^{\mathrm{tour}}(q)-F_Q^{\mathrm{flat}}(q)
=
(\pi^{\mathrm{tour}}-\pi^{\mathrm{flat}})\big(F_H(q)-F_L(q)\big)\le 0, \qquad \forall q,
\end{aligned}
\]
which proves \(Q^{\mathrm{tour}} \fosd Q^{\mathrm{flat}}\) marginally. By Assumption~\ref{ass:types}, conditional on each mechanism, \((Q_a^m,Q_b^m)\) are i.i.d.\ with common marginal \(F_Q^m\). Let \((F_Q^m)^{-1}(u)\triangleq \inf\{x\in\mathbb{R}:F_Q^m(x)\ge u\}\). Since \(F_Q^{\mathrm{tour}}(q)\le F_Q^{\mathrm{flat}}(q)\) for all \(q\),
\[
\begin{aligned}
(F_Q^{\mathrm{tour}})^{-1}(u)
\ge (F_Q^{\mathrm{flat}})^{-1}(u) \forall u\in(0,1).
\end{aligned}
\]
If \(U_a,U_b\overset{\mathrm{i.i.d.}}{\sim}\mathrm{Unif}(0,1)\) are independent of \(Z\), and \(\widetilde Q_j^m=(F_Q^m)^{-1}(U_j)\) for \(j\in\{a,b\}\), then \((\widetilde Q_a^m,\widetilde Q_b^m)\) has the same law as \((Q_a^m,Q_b^m)\), while \(\widetilde Q_j^{\mathrm{tour}}\ge \widetilde Q_j^{\mathrm{flat}}\) almost surely for \(j=a,b\). By coordinatewise monotonicity of \(\Gamma\),
\[
\begin{aligned}
&\Gamma(\widetilde Q_a^{\mathrm{tour}},\widetilde Q_b^{\mathrm{tour}},Z) \ge \Gamma(\widetilde Q_a^{\mathrm{flat}},\widetilde Q_b^{\mathrm{flat}},Z)
\end{aligned}
\]
almost surely under the same coupling, which is FOSD: \(Y^{\mathrm{tour}} \fosd Y^{\mathrm{flat}}\).
\end{proof}

\noindent The reviewer-independence assumption (Assumption~\ref{ass:types}) is necessary for the joint comparison: marginal FOSD alone does not suffice to compare the joint output \(Y^m\).

\begin{corollary}[Conservative bonus calibration]
\label{cor:calibration}
Assume \(\Delta p_{\min}>0\), and define the generalized inverse \(G^{-1}(u)\triangleq \inf\{x\in\mathbb{R}:G(x)\ge u\}\). The following sufficient conditions hold:
\begin{enumerate}[leftmargin=2em,itemsep=2pt]
\item[\textnormal{(a)}] If \(\kappa_i\equiv \Delta C>0\), then \(B>\dfrac{\Delta C}{\Delta p_{\min}}\) makes high effort strictly dominant for every reviewer type.
\item[\textnormal{(b)}] For any target effort rate \(\pi^\star\in[0,1)\),
\(B\ge \dfrac{[G^{-1}(\pi^\star)]_+}{\Delta p_{\min}}\) ensures \(\pi^{\mathrm{tour}}\ge \pi^\star\), where \([x]_+\triangleq \max\{x,0\}\).
\item[\textnormal{(c)}] If first moments exist and \(\Delta_Q\triangleq \mathbb{E}[Q^H]-\mathbb{E}[Q^L]>0\), then a feasible target mean-quality gain \(\eta\in\big[0,(1-\pi^{\mathrm{flat}})\Delta_Q\big]\) is guaranteed by
\[
B\ge \dfrac{\big[G^{-1}\!\big(\pi^{\mathrm{flat}}+\eta/\Delta_Q\big)\big]_+}{\Delta p_{\min}}.
\]
\end{enumerate}
\end{corollary}

\paragraph{Bonus expenditure accounting.}
For budget accounting, define \(\alpha_{uv}\triangleq \Pr\!\big(\max\{S_a^u,S_b^v\}\ge \tau\big)\) for \((u,v)\in\{L,H\}^2\).
Under continuous scores, exactly one reviewer receives the bonus whenever \(\max\{S_a^u,S_b^v\}\ge \tau\).
Thus if the population high-effort rate is \(\pi\), the expected variable bonus expenditure per task is
\[
\begin{aligned}
\mathcal C_{\mathrm{bonus}}(B,\pi)
&=
B\Big[
\pi^2\alpha_{HH}
+2\pi(1-\pi)\alpha_{HL}
+(1-\pi)^2\alpha_{LL}
\Big]
\le B.
\end{aligned}
\]

\paragraph{Reduced-form whitelist extension.}
Let \(z_i\) denote reviewer \(i\)'s current whitelist state, \(\Omega_i(z_i)\ge 0\) the discounted continuation value of remaining eligible for future tasks, and \(q_i(e_i,e_{-i};z_i)\in[0,1]\) the probability that the current task causes reviewer \(i\) to lose whitelist status.
Holding fixed the whitelist rule, define
\[
{\footnotesize
\begin{multlined}[t]
U_i^{\mathrm{tour+wl}}(e_i;e_{-i}, z_i)
= w_0+B\,p_i(e_i,e_{-i}) 
-\psi_i(e_i) 
+\big(1-q_i(e_i,e_{-i};z_i)\big)\Omega_i(z_i),
\end{multlined}
}
\]
\[
{\footnotesize
\begin{multlined}[t]
U_i^{\mathrm{flat+wl}}(e_i;e_{-i},
z_i) 
= w_0-\psi_i(e_i) 
+\big(1-q_i(e_i,e_{-i};z_i)\big)\Omega_i(z_i).
\end{multlined}
}
\]
Let \(\Delta q_i(e_{-i},z_i) \triangleq q_i(L,e_{-i};z_i)-q_i(H,e_{-i};z_i)\), and assume \(\Delta q_i(e_{-i},z_i)\ge 0\) for all \(e_{-i}, z_i\).
Then
\[
\begin{aligned}
\Delta U_i^{\mathrm{tour+wl}}(e_{-i},z_i)
&=
\Delta U_i^{\mathrm{flat+wl}}(e_{-i},z_i)
 + B\,\Delta p_i(e_{-i}),
\Delta U_i^{\mathrm{flat+wl}}(e_{-i},z_i)
= -\kappa_i 
 +\Delta q_i(e_{-i},z_i)\Omega_i(z_i).
\end{aligned}
\]
Equivalently,
\[
\begin{aligned}
\Delta U_i^{\mathrm{tour+wl}}(e_{-i},z_i)
=
B\,\Delta p_i(e_{-i})-\kappa_i 
 +\Delta q_i(e_{-i},z_i)\Omega_i(z_i) 
\ge
\Delta U_i^{\mathrm{tour}}(e_{-i}).
\end{aligned}
\]
Define \(\Lambda_i \triangleq \inf_{e_{-i},z_i} \Delta q_i(e_{-i},z_i)\Omega_i(z_i)\), and \(\widetilde\kappa_i\triangleq \kappa_i-\Lambda_i\).
Then
\[
\begin{aligned}
\Delta U_i^{\mathrm{tour+wl}}(e_{-i},z_i)
\ge
B\,\Delta p_{\min}-\widetilde\kappa_i 
,\qquad \forall e_{-i},z_i.
\end{aligned}
\]
Thus all one-shot sufficient conditions in Corollary~\ref{cor:calibration} apply \emph{pointwise} after replacing \(\kappa_i\) by \(\widetilde\kappa_i\).
If \(\kappa_i\equiv \Delta C\) and \(\Lambda_i\ge \Lambda_{\min}\) for all \(i\), then \(B>\dfrac{(\Delta C-\Lambda_{\min})_+}{\Delta p_{\min}}\) is a conservative sufficient condition for high effort to be strictly optimal at every state.

\paragraph{Takeaway.}
Relative to flat payment, the noisy tournament changes private incentives through \(B\,\Delta p_i(e_{-i})\), the bonus-weighted increase in the probability of clearing the quality bar and outperforming the matched reviewer.
When this return is large enough relative to reviewer-specific effort cost, the set of high-effort types expands, inducing a first-order stochastic right-shift in per-review latent quality.
Under reviewer independence and monotone adjudication, the same right-shift carries over to released-data quality.
Whitelist access further strengthens this comparison pointwise by lowering the effective net cost of effort through continuation value.

%% file: appendix_d_taxonomy.tex
\clearpage
\section{Taxonomy of Disciplines}
\label{sec:appendix-taxonomy}

\begin{center}
    \includegraphics[width=1.0\textwidth]{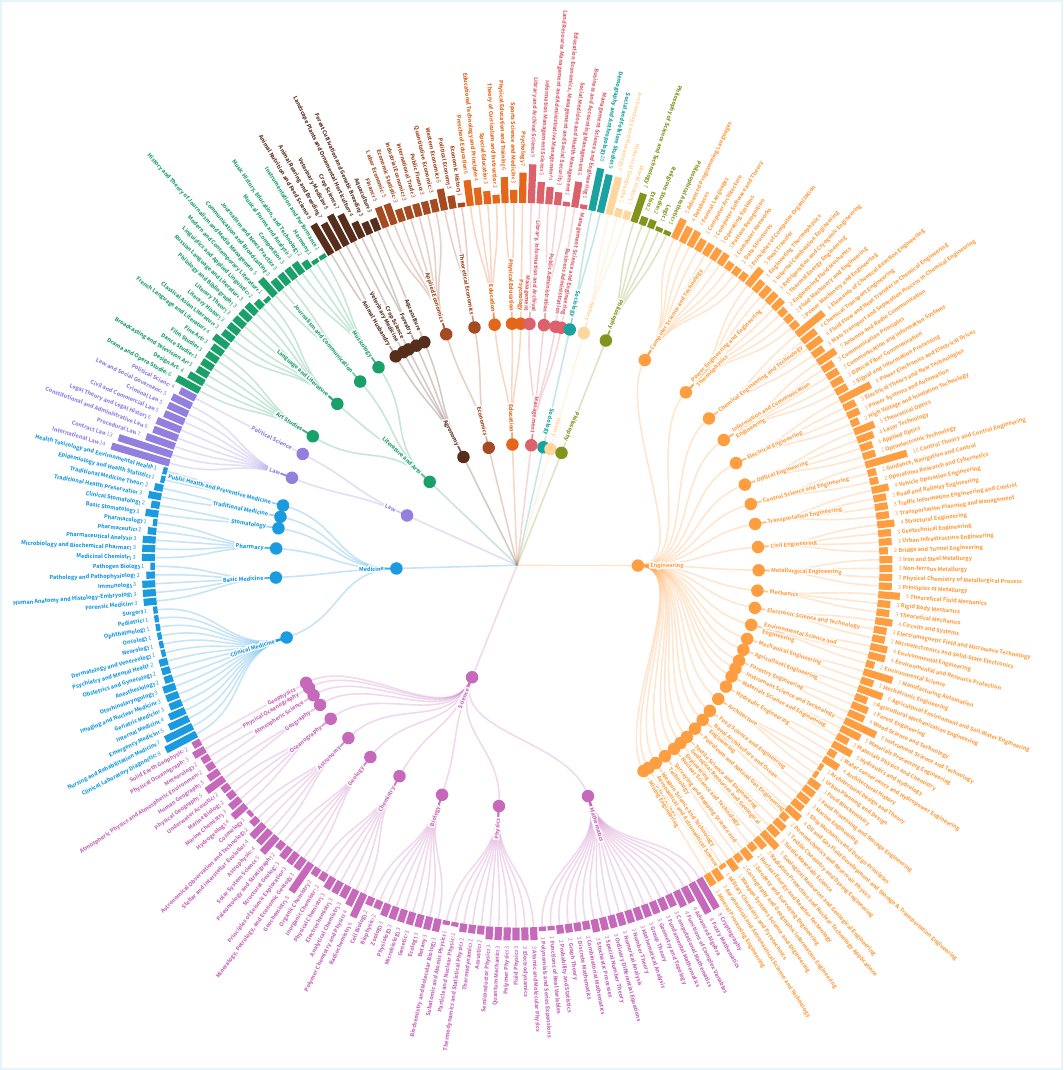}
    \captionof{figure}{Taxonomy of Disciplines}
    \label{fig:taxonomy_overview}
\end{center}

\onecolumn
\setlength{\LTleft}{0pt}
\setlength{\LTright}{0pt}
\setlength{\LTcapwidth}{\textwidth}
\setlength{\tabcolsep}{4pt}
\begin{footnotesize}
\captionof{table}{Disciplines and Subfields Statistics}
\label{tab:discipline_stats}
\vspace{0.4em}
\begin{longtable}{@{}p{0.17\textwidth}p{0.36\textwidth}p{0.33\textwidth}@{\hspace{6pt}}r@{}}
\toprule
\textbf{Discipline} & \textbf{Field} & \textbf{Subfield} & \textbf{Count} \\
\midrule
\endfirsthead
\multicolumn{4}{c}%
{{\bfseries \tablename\ \thetable{} Disciplines and Subfields Statistics}} \\
\toprule
\textbf{Discipline} & \textbf{Field} & \textbf{Subfield} & \textbf{Count} \\
\midrule
\endhead
\midrule
\multicolumn{4}{r}{{Continued on next page...}} \\
\endfoot
\bottomrule
\endlastfoot
Agronomy & Animal Husbandry & Animal Nutrition and Feed Science & 8 \\
         &                  & Animal Rearing and Breeding & 7 \\* \cmidrule{2-4}
         & Aquaculture      & Aquaculture & 3 \\* \cmidrule{2-4}
         & Crop Science     & Crop Science & 7 \\* \cmidrule{2-4}
         & Forestry         & Forest Cultivation and Genetic Breeding & 3 \\
         &                  & Landscape Plants and Ornamental Horticulture & 4 \\* \cmidrule{2-4}
         & Veterinary Medicine & Veterinary Medicine & 8 \\
\midrule
Economics & Applied Economics & Economic Statistics & 3 \\
          &                   & Finance & 5 \\
          &                   & Industrial Economics & 3 \\
          &                   & International Trade & 3 \\
          &                   & Labor Economics & 5 \\
          &                   & Public Finance & 3 \\
          &                   & Quantitative Economics & 3 \\* \cmidrule{2-4}
          & Theoretical Economics & Economic History & 1 \\
          &                       & Political Economy & 3 \\
          &                       & Western Economics & 5 \\
\midrule
Education & Education & Educational Technology and Principles & 4 \\
          &           & Preschool Education & 6 \\
          &           & Special Education & 3 \\
          &           & Theory of Curriculum and Instruction & 2 \\* \cmidrule{2-4}
          & Physical Education & Physical Education and Training & 6 \\
          &                    & Sports Science and Medicine & 3 \\* \cmidrule{2-4}
          & Psychology & Psychology & 7 \\
\midrule
Engineering & Aeronautical and Astronautical Science and Technology & Aeronautical and Astronautical Science and Technology & 3 \\* \cmidrule{2-4}
            & Agricultural Engineering & Agricultural Environment and Soil-Water Engineering & 6 \\
            &                          & Agricultural Mechanization Engineering & 3 \\* \cmidrule{2-4}
            & Architecture & Architectural Design and Theory & 1 \\
            &              & Urban Planning and Design & 1 \\
            &              & Architectural History & 4 \\* \cmidrule{2-4}
            & Chemical Engineering and Technology & Chemical Transport Engineering & 4 \\
            &                                     & Elements of Chemical Reaction Engineering & 8 \\
            &                                     & Fluid Flow and Heat Transfer in Chemical Engineering & 4 \\
            &                                     & Mass Transport and Separation Process in Chemical Engineering & 3 \\* \cmidrule{2-4}
            & Civil Engineering & Bridge and Tunnel Engineering & 2 \\
            &                   & Geotechnical Engineering & 3 \\
            &                   & Structural Engineering & 4 \\
            &                   & Urban Infrastructure Engineering & 3 \\* \cmidrule{2-4}
            & Computer Science and Technology & Advanced Programming Languages & 5 \\
            &                                 & Computer Architecture & 3 \\
            &                                 & Computer Networks & 2 \\
            &                                 & Computer Software and Theory & 3 \\
            &                                 & Data Structures & 2 \\
            &                                 & Databases & 4 \\
            &                                 & Formal Languages & 4 \\
            &                                 & Operating Systems & 3 \\
            &                                 & Pattern Recognition & 3 \\
            &                                 & Principles of Computer Organization & 2 \\* \cmidrule{2-4}
            & Control Science and Engineering & Control Theory and Control Engineering & 10 \\
            &                                 & Guidance, Navigation and Control & 2 \\
            &                                 & Operations Research and Cybernetics & 2 \\* \cmidrule{2-4}
            & Electrical Engineering & Electrical Theory and New Technologies & 3 \\
            &                        & High Voltage and Insulation Technology & 2 \\
            &                        & Power Electronics and Electrical Drives & 8 \\
            &                        & Power Systems and Automation & 3 \\* \cmidrule{2-4}
            & Electronic Science and Technology & Circuits and Systems & 4 \\
            &                                   & Electromagnetic Field and Microwave Technology & 3 \\
            &                                   & Microelectronics and Solid-State Electronics & 3 \\* \cmidrule{2-4}
            & Environmental Science and Engineering & Environmental Engineering & 4 \\
            &                                       & Environmental Science & 2 \\
            &                                       & Environmental and Resource Protection & 4 \\* \cmidrule{2-4}
            & Food Science and Engineering & Food Biochemistry & 3 \\
            &                              & Food Processing and Storage Engineering & 3 \\* \cmidrule{2-4}
            & Forestry Engineering & Forest Engineering & 4 \\
            &                      & Wood Science and Technology & 4 \\* \cmidrule{2-4}
            & Geological Resources and Geological Engineering & Geological Resources and Geological Engineering & 5 \\* \cmidrule{2-4}
            & Hydraulic Engineering & Hydraulics and Hydrology & 5 \\
            &                       & Water conservancy and Hydropower Engineering & 2 \\* \cmidrule{2-4}
            & Information and Communication Engineering & Antenna and Radio Communication & 5 \\
            &                                           & Communication Principles & 3 \\
            &                                           & Communication and Information Systems & 3 \\
            &                                           & Optical Fiber Communication & 3 \\
            &                                           & Signal and Information Processing & 3 \\* \cmidrule{2-4}
            & Instrument Science and Technology & Instrument Science and Technology & 8 \\* \cmidrule{2-4}
            & Materials Science and Engineering & Materials Physics and Chemistry & 3 \\
            &                                   & Materials Processing Engineering & 5 \\* \cmidrule{2-4}
            & Mechanical Engineering & Manufacturing Automation & 7 \\
            &                        & Mechatronic Engineering & 3 \\* \cmidrule{2-4}
            & Mechanics & Rigid Body Mechanics & 3 \\
            &           & Theoretical Fluid Mechanics & 5 \\
            &           & Theoretical Mechanics & 3 \\* \cmidrule{2-4}
            & Metallurgical Engineering & Iron and Steel Metallurgy & 3 \\
            &                           & Non-ferrous Metallurgy & 3 \\
            &                           & Physical Chemistry of Metallurgical Process & 3 \\
            &                           & Principles of Metallurgy & 3 \\* \cmidrule{2-4}
            & Mining Engineering & Mineral Processing Engineering & 3 \\* \cmidrule{2-4}
            & Naval Architecture and Ocean Engineering & Marine Engineering & 3 \\
            &                                          & Ship Mechanics and Design Principles & 3 \\* \cmidrule{2-4}
            & Nuclear Science and Technology & Nuclear Energy and Reactor Technology & 2 \\
            &                                & Radiation Protection and Nuclear Technology Applications & 3 \\* \cmidrule{2-4}
            & Optical Engineering & Applied Optics & 3 \\
            &                     & Laser Technology & 4 \\
            &                     & Optoelectronic Technology & 3 \\
            &                     & Theoretical Optics & 5 \\* \cmidrule{2-4}
            & Petroleum and Natural Gas Engineering & Oil and Gas Field Development and Storage \& Transportation Engineering & 4 \\
            &                                       & Poromechanics and Reservoir Physics & 2 \\* \cmidrule{2-4}
            & Power Engineering and Engineering Thermophysics & Engineering Fluid Mechanics & 2 \\
            &                                                 & Engineering Thermophysics & 3 \\
            &                                                 & Fluid Machinery and Engineering & 2 \\
            &                                                 & Heat Transfer & 5 \\
            &                                                 & Internal Combustion Engineering & 3 \\
            &                                                 & Power Machinery and Engineering & 2 \\
            &                                                 & Refrigeration and Cryogenic Engineering & 3 \\
            &                                                 & Thermal Energy Engineering & 3 \\* \cmidrule{2-4}
            & Surveying and Mapping Science and Technology & Cartography and Geographic Information Engineering & 2 \\
            &                                              & Geodesy and Surveying Engineering & 3 \\* \cmidrule{2-4}
            & Textile Science and Engineering & Textile Chemistry and Dyeing Engineering & 3 \\
            &                                 & Textile Materials Science & 3 \\* \cmidrule{2-4}
            & Transportation Engineering & Road and Railway Engineering & 3 \\
            &                            & Traffic Information Engineering and Control & 3 \\
            &                            & Transportation Planning and Management & 3 \\
            &                            & Vehicle Operation Engineering & 4 \\* \cmidrule{2-4}
            & Weapon Science and Technology & Military Chemistry and Pyrotechnics & 1 \\
            &                               & Weapon Systems Science and Engineering & 3 \\
\midrule
History & History & Archaeology and Museology & 9 \\
        &         & Historical Geography & 2 \\
        &         & World History & 2 \\
\midrule
Law & Law & Civil and Commercial Law & 5 \\
    &     & Constitutional and Administrative Law & 6 \\
    &     & Contract Law & 13 \\
    &     & Criminal Law & 5 \\
    &     & International Law & 14 \\
    &     & Law and Social Governance & 5 \\
    &     & Legal Theory and Legal History & 6 \\
    &     & Procedural Law & 7 \\* \cmidrule{2-4}
    & Political Science & Political Science & 4 \\
\midrule
Literature and Arts & Art Studies & Broadcasting and Television Art & 3 \\
                    &             & Dance Studies & 3 \\
                    &             & Design Arts & 4 \\
                    &             & Drama and Opera Studies & 6 \\
                    &             & Film Studies & 3 \\
                    &             & Fine Arts & 3 \\* \cmidrule{2-4}
                    & Journalism and Communication & Communication and Broadcasting & 3 \\
                    &                              & History and Theory of Journalism and Media Management & 5 \\
                    &                              & Journalism and News Practice & 3 \\* \cmidrule{2-4}
                    & Language and Literature & Classical Asian Literature & 3 \\
                    &                         & French Language and Literature & 4 \\
                    &                         & Linguistics and Applied Linguistics & 2 \\
                    &                         & Literary History & 3 \\
                    &                         & Literary Theory & 3 \\
                    &                         & Modern and Contemporary Literature & 1 \\
                    &                         & Philology and Bibliography & 3 \\
                    &                         & Russian Language and Literature & 3 \\* \cmidrule{2-4}
                    & Musicology & Composition & 3 \\
                    &            & Harmony & 1 \\
                    &            & Instrumentation and Performance & 1 \\
                    &            & Music History, Education, and Technology & 2 \\
                    &            & Musical Forms and Analysis & 3 \\
\midrule
Management & Business Administration & Business and Accounting Management & 6 \\* \cmidrule{2-4}
           & Library, Information and Archival Management & Information Management Science & 5 \\
           &                                              & Library and Archival Science & 9 \\* \cmidrule{2-4}
           & Management Science and Engineering & Management Science and Engineering & 1 \\* \cmidrule{2-4}
           & Public Administration & Education Economics, Management and Social Security & 3 \\
           &                       & Land Resource Management and Administrative Management & 4 \\
           &                       & Social Medicine and Health Management & 1 \\
\midrule
Medicine & Basic Medicine & Forensic Medicine & 3 \\
         &                & Human Anatomy and Histology-Embryology & 3 \\
         &                & Immunology & 3 \\
         &                & Pathogen Biology & 1 \\
         &                & Pathology and Pathophysiology & 2 \\* \cmidrule{2-4}
         & Clinical Medicine & Anesthesiology & 2 \\
         &                   & Clinical Laboratory Diagnostics & 8 \\
         &                   & Dermatology and Venereology & 1 \\
         &                   & Emergency Medicine & 5 \\
         &                   & Geriatric Medicine & 3 \\
         &                   & Imaging and Nuclear Medicine & 3 \\
         &                   & Internal Medicine & 4 \\
         &                   & Neurology & 1 \\
         &                   & Nursing and Rehabilitation Medicine & 7 \\
         &                   & Obstetrics and Gynecology & 2 \\
         &                   & Oncology & 1 \\
         &                   & Ophthalmology & 1 \\
         &                   & Otorhinolaryngology & 3 \\
         &                   & Pediatrics & 1 \\
         &                   & Psychiatry and Mental Health & 2 \\
         &                   & Surgery & 1 \\* \cmidrule{2-4}
         & Pharmacy & Medicinal Chemistry & 3 \\
         &          & Microbiology and Biochemical Pharmacy & 3 \\
         &          & Pharmaceutical Analysis & 3 \\
         &          & Pharmaceutics & 2 \\
         &          & Pharmacology & 1 \\* \cmidrule{2-4}
         & Public Health and Preventive Medicine & Epidemiology and Health Statistics & 1 \\
         &                                       & Health Toxicology and Environmental Health & 1 \\* \cmidrule{2-4}
         & Stomatology & Basic Stomatology & 3 \\
         &             & Clinical Stomatology & 2 \\* \cmidrule{2-4}
         & Traditional Medicine & Traditional Health Preservation & 3 \\
         &                              & Traditional Medicine Theory & 2 \\
\midrule
Philosophy & Philosophy & Ethics & 2 \\
           &            & Logic & 1 \\
           &            & Philosophical Aesthetics & 1 \\
           &            & Philosophy of Science and Technology & 7 \\
           &            & Religious Studies & 2 \\
\midrule
Science & Astronomy & Astronomical Observation and Technology & 2 \\
        &           & Astrophysics & 4 \\
        &           & Cosmology & 1 \\
        &           & Solar System Science & 5 \\
        &           & Stellar and Interstellar Evolution & 4 \\* \cmidrule{2-4}
        & Atmospheric Science & Atmospheric Physics and Atmospheric Environment & 2 \\
        &                     & Meteorology & 2 \\* \cmidrule{2-4}
        & Biology & Biochemistry and Molecular Biology & 3 \\
        &         & Biophysics & 2 \\
        &         & Botany & 3 \\
        &         & Cell Biology & 2 \\
        &         & Ecology & 3 \\
        &         & Genetics & 3 \\
        &         & Microbiology & 3 \\
        &         & Physiology & 3 \\
        &         & Zoology & 3 \\* \cmidrule{2-4}
        & Chemistry & Analytical Chemistry & 3 \\
        &           & Electrochemistry & 3 \\
        &           & Inorganic Chemistry & 2 \\
        &           & Organic Chemistry & 2 \\
        &           & Physical Chemistry & 3 \\
        &           & Polymer Chemistry and Physics & 4 \\
        &           & Radiochemistry & 6 \\* \cmidrule{2-4}
        & Geography & Human Geography & 3 \\
        &           & Physical Geography & 5 \\* \cmidrule{2-4}
        & Geology & Geochemistry & 8 \\
        &         & Mineralogy, Petrology, and Economic Geology & 3 \\
        &         & Paleontology and Stratigraphy & 2 \\
        &         & Principles of Seismic Exploration & 3 \\
        &         & Structural Geology & 3 \\* \cmidrule{2-4}
        & Geophysics & Solid Earth Geophysics & 2 \\* \cmidrule{2-4}
        & Mathematics & Advanced Algebra & 4 \\
        &             & Combinatorial Mathematics & 2 \\
        &             & Computational Mathematics & 3 \\
        &             & Cryptography & 8 \\
        &             & Discrete Mathematics & 2 \\
        &             & Functions of Complex Variables & 4 \\
        &             & Functions of Real Variables & 1 \\
        &             & Fundamental Mathematics & 3 \\
        &             & Fuzzy Mathematics & 8 \\
        &             & Geometry and Topology & 3 \\
        &             & Graph Theory & 2 \\
        &             & Group Theory & 3 \\
        &             & Mathematical Analysis & 3 \\
        &             & Number Theory & 3 \\
        &             & Numerical Analysis & 3 \\
        &             & Ordinary Differential Equations & 3 \\
        &             & Polynomials and Series Expansions & 1 \\
        &             & Probability and Statistics & 2 \\
        &             & Special Number Theory & 3 \\
        &             & Stochastic Processes & 3 \\* \cmidrule{2-4}
        & Oceanography & Hydrogeology & 4 \\
        &              & Marine Biology & 2 \\
        &              & Marine Chemistry & 3 \\
        &              & Underwater Acoustics & 2 \\* \cmidrule{2-4}
        & Physical Oceanography & Physical Oceanography & 3 \\* \cmidrule{2-4}
        & Physics & Acoustics & 2 \\
        &         & Atomic and Molecular Physics & 3 \\
        &         & Electrodynamics & 3 \\
        &         & Fluid Physics & 3 \\
        &         & Particle and Nuclear Physics & 1 \\
        &         & Polymer Physics & 3 \\
        &         & Quantum Mechanics & 3 \\
        &         & Semiconductor Physics & 3 \\
        &         & Subatomic and Atomic Physics & 1 \\
        &         & Thermodynamics & 2 \\
        &         & Thermodynamics and Statistical Physics & 2 \\
\midrule
Sociology & Sociology & Demography and Anthropology & 10 \\
          &           & Social and Folklore Studies & 9 \\
\end{longtable}
\end{footnotesize}

%% file: appendix_e_manuals.tex
\section{Data Collection Details}
\label{appendix:manuals}

\subsection{Annotator Recruitment Test}
\label{appendix:round1}

\begin{enumerate}
    \item \textbf{Academic Background}
    \begin{itemize}
        \item University (Full Name in English):
        \item Highest Academic Degree (Bachelor's / Master's / Doctoral / Currently Enrolled):
        \item Major Name (Official Title):
        \item Intended Disciplinary Fields for Question Design (Limit: 1--3):
    \end{itemize}

    \item \textbf{Authoritative Domain Knowledge}
    
    Please list \textbf{3} of the most authoritative publications, conferences or monographs (titles included) in your professional field. We intend to evaluate your professional insight through your understanding of the authoritative sources in the industry.
    
    Format requirement: Name + Type (Journal/Conference/Monograph) + Rationale

    \item \textbf{Understanding of Disciplinary Competency Assessment}

   If you were asked to design \textbf{3–6 multiple-choice questions} to assess the advanced literacy and cognitive abilities of LLMs in your chosen field, what content would each question target? Please briefly elaborate on the question content and explain the rationale for designing each question.
   
   Please ensure that the overall design logic of this set of questions is clearly demonstrated: specifically, how these questions, covering different dimensions within the discipline, can be integrated to outline a comprehensive cognitive framework of the subject.
   
   Format requirement: Content + Rationale

    \item \textbf{Stress Tolerance and Work Attitude}
    \begin{enumerate}
        \item \textbf{Attitude Towards Review and Revision (Single Choice)} \\
        If your questions are returned due to incomplete citations or non-standard formatting, and you are required to supplement supporting evidence item by item, what would your response be?
        \begin{itemize}
            \item A. The rules are reasonable and necessary; I will supplement the required information accordingly
            \item B. I feel a bit annoyed, but I will revise the questions as required
            \item C. I think the requirements are overly detailed and somewhat confusing
            \item D. I am not willing to continue with this process
        \end{itemize}

        \item \textbf{Tendency of Responsibility Attribution (Single Choice)} \\
        If your questions fail to pass the review multiple times, what do you think is the primary cause?
        \begin{itemize}
            \item A. My understanding of the standards and requirements is not thorough enough
            \item B. The questions themselves still have room for improvement
            \item C. The review standards are overly subjective
            \item D. Bad luck
        \end{itemize}
    \end{enumerate}

    \item \textbf{Rule Comprehension Verification}
    \begin{enumerate}
        \item \textbf{Understanding of Question Types (Multiple Choice)} \\
        Which of the following questions do \textbf{not} meet the requirements for high-quality disciplinary assessment questions?
        \begin{itemize}
            \item A. Questions testing the reasoning and application of core disciplinary theories in specific scenarios
            \item B. Questions requiring the model to memorize specific numerical values or factual information
            \item C. Combined logical reasoning questions based on multiple statements
            \item D. Static judgment questions involving only extremely niche or rare subjects
            \item E. Questions developed based on individual scholars' viewpoints that lack widespread consensus
        \end{itemize}

        \item \textbf{Understanding of Question Design (Single Choice)} \\
        What is the main source of difficulty in high-quality assessment questions?
        \begin{itemize}
            \item A. Numerous calculation steps and complex numerical values
            \item B. Obscure knowledge points and high memory difficulty
            \item C. Interactions between multiple restrictive conditions
            \item D. Long question length and large information volume
        \end{itemize}
    \end{enumerate}

    \item \textbf{Workflow and Standard Confirmation (Multiple Choice)}
        \begin{itemize}
            \item Develop all questions in English (including punctuation marks) throughout the entire process
            \item Each option (whether correct or incorrect) must be supported by evidence or citations
            \item Questions may undergo multiple rounds of review and revision
            \item Questions must demonstrate disciplinary representativeness, rather than being esoteric, biased, or tricky
        \end{itemize}
\end{enumerate}

\subsection{Annotation Manual}
\label{appendix:annotation-manual}

\subsubsection{Project Background}
Current LLMs have demonstrated strong capabilities in programming, mathematics, and reasoning, but their performance still varies widely across different disciplines. To enable researchers and practitioners in various disciplinary fields to select the most appropriate large AI model for their work, domain experts are needed to design questions that evaluate models’ disciplinary knowledge and competence in their respective fields. The questions you provide will serve as critical data for evaluating LLMs.
\begin{itemize}
    \item Your role: \textbf{Question Design Expert}
    \item Your task: Create high-difficulty multiple-choice questions that at least three LLMs will answer incorrectly, and provide explanations and source materials.
    \begin{itemize}
        \item Question: Consists of a stem and options. Questions must be \textbf{representative and unique} to the discipline, and able to assess high-level knowledge and literacy in the field. Questions may be original, adapted, or directly excerpted; \textbf{Full AI-generated questions are strictly prohibited.}
        \item Explanation: Clearly explain why each option is correct or incorrect.
        \item Source Materials: The origin of the question, which may be a book or high-quality academic paper.
    \end{itemize}
    \item Discipline List: Please check whether the list includes your specialized discipline. Ensuring the accuracy and rationality of the questions is our top priority.
\end{itemize}

\subsubsection{Core Criteria for Question Design}
\noindent \textbf{What We Require:}
\begin{enumerate}
    \item \textbf{Disciplinary Representativeness:} Instances must act as highly representative probes. If a discipline were to be evaluated using only 3 to 5 questions, the submitted instance must be fundamental and comprehensive enough to be one of them.
    \item \textbf{High-Order Knowledge Application:} 
    \begin{itemize}
        \item Questions must construct a logically complete closed system where the solution is the inevitable product of general deductive reasoning applied to domain-specific axioms.
        \item Difficulty should stem from the complex coupling of constraints and multi-hop reasoning, not merely computational heavy-lifting.
        \item For instance, physics questions should demand the integration of axioms and skills to solve a specific state transition, rather than asking for the formula of momentum conservation.
    \end{itemize}
    \item \textbf{Static Knowledge Breadth:} For memory-intensive disciplines (e.g., Education), questions should maximize knowledge coverage (e.g., utilizing $10+$ statements to encompass major pedagogical theories).
\end{enumerate}

\noindent \textbf{What We Reject:}
\begin{enumerate}
    \item \textbf{Narrow or Trivial Memorization:} Questions testing pure rote memory devoid of core disciplinary literacy (e.g., "What is the 100th digit of $\pi$?") or focusing on hyper-niche, obscure sub-entities.
    \item \textbf{Idiosyncratic or Tricky Trivia:} Questions universally recognized as flawed or unreasonable even in human examinations.
    \item \textbf{Weak Epistemological Consensus:} Hypotheses proposed by individual scholars that lack widespread academic consensus or are highly volatile (e.g., highly debated legal interpretations).
    \item \textbf{Non-Disciplinary Failure Modes:} Instances where LLMs fail due to semantic traps, ambiguous phrasing, or floating-point calculation errors rather than a deficit in disciplinary literacy.
\end{enumerate}

\subsubsection{Standardized Annotation Workflow}
The question authoring process is strictly compartmentalized into six components. The specifications for each component are detailed in Table \ref{tab:creation_workflow}.

\begin{table*}[htbp]
\centering
\captionsetup{font=footnotesize}
\caption{Standardized Annotation Workflow and Specifications.}
\label{tab:creation_workflow}
\renewcommand{\arraystretch}{1.3}
\begin{tabularx}{\textwidth}{l X}
\toprule
\textbf{Component} & \textbf{Specifications and Constraints} \\
\midrule
\textbf{Question} & 
Must be pure English text.\newline
Image uploads in the stem are prohibited.\newline
For pseudo-multi-choice formats, the stem must contain $\ge 6$ foundational statements.\newline
The question stem should not include instructions such as "You are an expert".\newline
Must be structurally self-contained; references like "in this paper" are strictly forbidden. All necessary premises must be explicitly stated.\newline
Sources must be highly authoritative: e.g., established textbooks, Q1 SCI/SSCI journals, or CCF-A/B proceedings. Submissions citing controversial publishers (e.g., MDPI, Hindawi) or user-generated forums are categorically rejected. \\
\midrule
\textbf{Options} & 
Must provide exactly 10 distinct, pure-text English options.\newline
For combinatorial options, the distribution must reflect statistical normality (avoiding extreme imbalance where one option has 1 statement and another has 8).\newline
Options should not begin with labels like "A.".\newline
\textbf{Strict Prohibition of Proper Subsets:} Options must not exhibit subset relationships (e.g., if A is $\{1, 2, 4\}$, B cannot be $\{1\}$) to prevent logical leakage. \\
\midrule
\textbf{Explanation} & 
Must explicitly diagnose why the correct option is valid and why \textit{each} distractor is flawed (avoiding generic copy-pasting of theoretical background).\newline
Unique explanations per option are mandatory, except for deterministic calculation questions. \\
\midrule
\textbf{Sources} & 
\textbf{Visual Evidence:} Must upload screenshots of the source material with precise highlights and annotations corresponding to each option/statement.\newline
\textbf{Traceability:} Must provide accessible URLs or rigorous APA-formatted citations (including specific page numbers).\\
\midrule
\textbf{Material} & 
Reference images necessary for solving the question must be uploaded and annotated. \\
\midrule
\textbf{Formatting} & 
All mathematical formulas and symbols must be written in standard LaTeX and render flawlessly in the Markdown environment. \\
\bottomrule
\end{tabularx}
\end{table*}

\subsubsection{Accepted and Rejected Cases}
Table \ref{tab:detailed_case_studies} illustrates the dichotomy between accepted and rejected submissions based on our core criteria.

The accepted case provided here is for reference only in terms of content and format. Please note that question types are not limited to calculation problems.

This question is selected as an excellent example for the following reasons:
\begin{enumerate}
    \item It targets marine chemistry and assesses advanced reasoning and computational abilities, rather than only testing static knowledge.
    \item Even though it is a calculation question, the analysis of incorrect options is briefly summarized instead of being copied and pasted repetitively.
\end{enumerate}

\begin{table*}[htbp]
\centering
\captionsetup{font=footnotesize}
\caption{Accepted and Rejected Cases}
\label{tab:detailed_case_studies}
\footnotesize
\setlength{\tabcolsep}{4pt}
\renewcommand{\arraystretch}{1.05}
\begin{tabularx}{\textwidth}{>{\raggedright\arraybackslash}p{0.11\textwidth} X}
\toprule
\multicolumn{2}{c}{\textbf{Case 1: Accepted Case}} \\
\midrule
\textbf{Discipline} & Marine Science - Marine Chemistry \\
\textbf{Question} & It is known that Modified Circumpolar Deep Water (MCDW) is defined as the water mass in the Ross Sea located between neutral density surfaces of 28 $\text{kg m}^{-3}$ and 28.27 $\text{kg m}^{-3}$. The average salinity of MCDW is $34.72 \pm 0.01$, and the average salinity of Antarctic Bottom Water (AASW) is $33.7 \pm 0.2$. The dissolved organic carbon (DOC) concentration in MCDW is 40.0 $\mu\text{mol kg}^{-1}$, in AASW is 45.1 $\mu\text{M}$, and in Deep Shelf Water (DSW) is 48.4 $\mu\text{M}$. During the formation of DSW, how much does the DOC concentration increase in the Ross Sea, and how much excess DOC is exported from the deep water to the shelf area? \\
\textbf{Options} & 
\textbf{A.} 7 $\mu\text{M}$; $4 \pm 0.6 \text{ Tg DOC yr}^{-1}$ \newline
\textbf{B.} 8 $\mu\text{M}$; $4 \pm 0.8 \text{ TgC yr}^{-1}$ \\
\textbf{Explanation} & 
\textbf{Option A (Correct):} MCDW is the primary water source for the Ross Ice Shelf, formed by the mixing of Circumpolar Deep Water (CDW) and AASW on the continental slope. Given the average salinity of CDW ($34.72 \pm 0.01$) and AASW ($33.7 \pm 0.2$), assuming salinity on the continental shelf is conservative (i.e., unaffected by processes like brine rejection), we calculate the proportions required to form MCDW (salinity = $34.47 \pm 0.05$). The calculation is:\newline
$f(\text{CDW}) = \frac{S(\text{AASW}) - S(\text{mCDW})}{S(\text{AASW}) - S(\text{CDW})}$\newline
$f(\text{AASW}) = 1 - f(\text{CDW})$\newline
Where $S$ represents average salinity and $f$ represents the proportion. The proportions of CDW and AASW forming MCDW are 0.75 and 0.25, respectively. With DOC(CDW) = 40.0 and DOC(AASW) = 45.1:\newline
$\text{DOC(MCDW)} = \text{DOC(CDW)} \cdot f(\text{CDW}) + \text{DOC(AASW)} \cdot f(\text{AASW}) = 41.0 \mu\text{M}$.\newline
$\text{DOC(enriched)} = \text{DOC(DSW)} - \text{DOC(mCDW)} = 48.4 - 41.0 = 7.4 \mu\text{M}$.\newline
The DOC export off the continental shelf is calculated as: $7 \mu\text{M} \text{ export} \times 6 \text{ years residence time} = 4 \text{ TgC yr}^{-1}$.\newline
\textbf{Option B (Incorrect):} The enriched DOC content is overestimated. The estimated DOC content in DSW is $1 \mu\text{M}$ higher than the correct value. Consequently, the exported DOC is also overestimated: $8 \mu\text{M} \times 6 \text{ years} = 4.8 \text{ TgC yr}^{-1}$. \\
\textbf{Source} & Bercovici, S. K., Huber, B. A., DeJong, H. B., Dunbar, R. B., \& Hansell, D. A. (2017). Dissolved organic carbon in the Ross Sea: Deep enrichment and export. \textit{Limnology and Oceanography}, 62(6), 2593-2603. \\
\bottomrule
\end{tabularx}
\end{table*}

\begin{table*}[htbp]\ContinuedFloat
\centering
\captionsetup{font=footnotesize}
\caption[]{Accepted and Rejected Cases (continued)}
\footnotesize
\setlength{\tabcolsep}{4pt}
\renewcommand{\arraystretch}{1.05}
\begin{tabularx}{\textwidth}{>{\raggedright\arraybackslash}p{0.11\textwidth} X}
\toprule
\multicolumn{2}{c}{\textbf{Case 2: Rejected Case}} \\
\midrule
\textbf{Discipline} & Biology - Zoology \\
\textbf{Question} & Based on the content of Chapter 1 in the attachment, which of the following statements about the genus \textit{Acanthorhodeus} are correct:\newline
1. Body elongated, slightly compressed, large head, flat snout.\newline
2. Lateral line scales are at least weakly ctenoid, two dorsal fins, separated.\newline
3. Posterior margin of the caudal fin is rounded.\newline
4. The first dorsal fin is short, consisting of 3-4 spines.\newline
5. Upper jaw is slightly shorter than the lower jaw. \\
\textbf{Options} & 
\textbf{A.} 1, 2 \quad \textbf{B.} 1, 2, 3 (Correct Option) \quad \textbf{C.} 1, 2, 3, 4, 5 \quad \textbf{D.} 1 \quad \textbf{E.} 2 \quad \textbf{F.} 1, 2, 4 \quad \textbf{G.} 1, 3, 5 \quad \textbf{H.} 3, 5 \quad \textbf{J.} 4, 5 \\
\textbf{Explanation} & 
\textbf{Option A:} Statement 1 is correct (Body elongated, slightly compressed, large head, flat snout); Statement 2 is correct (Lateral line scales are at least weakly ctenoid, two dorsal fins, separated).\newline
\textbf{Option J:} Statement 4 is incorrect (The first dorsal fin is short, consisting of 4-6 spines); Statement 5 is incorrect (Upper jaw is slightly longer than the lower jaw). \\
\textbf{Rejection Rationale} & 
\textbf{1. Weak Disciplinary Representativeness:} The question fails to engage high-order disciplinary literacy. Statements 4 and 5 merely test pure static memory of isolated trivia, lacking sufficient cognitive coverage to serve as a representative evaluation probe.\newline
\textbf{2. Structural Flaw (Proper Subsets):} Options A, D, and E are proper subsets of the correct Option B. Models can exploit this logical leakage (selecting A, D, or E would technically be correct). Furthermore, the option distribution is severely imbalanced, violating statistical normality constraints.\newline
\textbf{3. Statement Count Deficit:} The question stem only contains 5 statements, failing the strict pseudo-multi-choice minimum requirement ($\ge 6$ statements).\newline
\textbf{4. Context Drift:} The stem explicitly references an unprovided context ("Chapter 1"), violating the requirement that all questions must be logically self-contained. \\
\bottomrule
\end{tabularx}
\end{table*}

\subsection{Review Manual}
\label{appendix:review-manual}

\subsubsection{Core Philosophy and General Workflow}
The primary objective of the review process is to ensure that each curated instance serves as a highly representative probe for evaluating LLMs on specific fine-grained disciplines. Reviewers must adopt a \textbf{zero-tolerance policy} toward ambiguous, peripheral, or substandard data. 
\begin{itemize}
    \item \textbf{Comprehensive Feedback:} Unless an instance requires a complete topic overhaul, reviewers must identify and aggregate all logical, factual, and formatting issues into a single comprehensive feedback report before returning it for revision.
    \item \textbf{AI-Generation Rejection:} Submissions exhibiting clear patterns of unedited, large-scale AI generation must be categorically rejected.
    \item \textbf{Escalation Mechanism:} For borderline cases or interdisciplinary disputes, reviewers are required to escalate the instance to the Core Committee for final arbitration.
\end{itemize}

\subsubsection{Disciplinary Alignment and Factuality}
\begin{itemize}
    \item \textbf{Disciplinary Matching:} Reviewers must first verify whether the question strictly aligns with the declared subfield. Mismatched instances must be returned with a mandate to reclassify or rewrite.
    \item \textbf{Epistemological Rigor (Especially in Humanities \& Social Sciences):} Questions must test established scientific truths, foundational theories, or widely recognized academic consensus. Volatile opinions, transient scholarly debates, or highly subjective hypotheses without proper contextualization must be rejected.
    \item \textbf{Memory Error Rejection:} If an instance causes a model to fail solely due to a superficial memory error (e.g., misremembering an obscure date, a peripheral name, or a raw statistic) rather than a deficit in high-order deductive reasoning or conceptual understanding, the instance must be rejected. Questions must evaluate the mastery of causal mechanisms, not data retrieval.
\end{itemize}

\subsubsection{Source Verification and Material Authenticity}
Every claim must be fully traceable to authoritative human knowledge sources.
\begin{itemize}
    \item \textbf{Authoritative Sources:} Referenced literature must be identifiable via valid URLs, DOIs, or standard APA citations. Acceptable sources include high quality peer-reviewed journals (e.g., SCI/SSCI Q1, CSSCI), authoritative monographs, or official databases. Preprints (e.g., arXiv) are only acceptable if they demonstrate high citation counts or widespread community validation.
    \item \textbf{Material Consistency:} Any accompanying materials (e.g., figures, charts) must be intrinsically necessary to solve the question and perfectly correspond to the cited source.
\end{itemize}

\subsubsection{Option Rigor and Structural Integrity}
To prevent LLMs from exploiting logical loopholes, the 10-option structure must adhere to strict constraints:
\begin{itemize}
    \item \textbf{Pseudo-Multi-Choice Constraints:} For combination-based questions, the question stem must present a minimum of \textbf{6 distinct foundational statements}.
    \item \textbf{No Proper Subset Relationships:} To avoid logical leakage, options must not have proper subset relationships. For example, if Option A is $\{1, 2\}$ and Option B is $\{1, 2, 3\}$, proving B correct implicitly validates A, introducing ambiguity. Options must be mutually exclusive in their truth values.
    \item \textbf{Explanation Completeness:} The explanations provided for the options must comprehensively address why the correct option is uniquely valid and why all distractors are flawed. For deterministic questions (e.g., mathematical calculations), overlapping explanations across options are permissible provided the core proof is sound.
\end{itemize}

\subsection{Formatting and Rendering}
Reviewers must ensure that all text, mathematical formulas, and symbolic logic within the Question, Options, and Explanations strictly conform to standard LaTeX syntax and render flawlessly in Markdown environments. Formatting failures constitute immediate grounds for rejection.

\subsection{Appeal Handling and Quality Enforcement}
\begin{itemize}
    \item \textbf{Annotator Appeals:} Recognizing the domain expertise of annotators, the pipeline supports a formal appeal process. Reviewers are required to objectively re-evaluate contested instances based on newly provided academic evidence.
    \item \textbf{Accountability:} Reviewers who consistently exhibit superficial auditing ("lazy consensus"), approve logically flawed instances, or violate the double-blind protocols will be permanently removed from the reviewer pool and forfeit their compensation.
\end{itemize}

\clearpage

\subsection{LLM-based Filtering framework}
\label{appendix:llm-judge}

\begin{tcblisting}{
    enhanced,
    listing only, % 只显示代码内容，不进行 LaTeX 编译
    colback=gray!5,
    colframe=blue!50!black,
    coltitle=white,
    title={Feature Extraction Prompt},
    sharp corners=south,
    boxrule=0.8pt,
    fontupper=\small,
    listing options={
        breaklines=true,     % 自动换行
        basicstyle=\small,
        breakatwhitespace=true,
        escapeinside={(*}{*)}, % 如果需要在代码中嵌套 LaTeX 命令
        columns=fullflexible
    }
}
# Role
You are a domain expert in {{field_name}}. Your task is to analyze a single     question and extract a structured summary.

# Input
Question {{question_number}}: {{question_text}}
Correct Answer: {{correct_answer}} = {{correct_answer_content}}

# Analysis Task
Extract the following attributes:
| Field | Description | Values |
|-------|-------------|--------|
| `knowledge_point` | The core concept being tested. Be specific-identify the exact principle, mechanism, or rule. | 10-20 words |
| `question_type` | How the question tests understanding | See definitions below |
| `key_topics` | Main topics/keywords for clustering | 2-4 terms |
| `real_world_impact` | Does this knowledge have significant practical consequences? | HIGH / MEDIUM / LOW |

## Question Type Definitions
- **CASE_BASED**: Presents a scenario (patient, legal case, experiment) requiring analysis and judgment
- **APPLIED_REASONING**: Tests ability to apply principles to specific situations without a narrative scenario

## Real-World Impact Levels

- **HIGH**: Errors in this knowledge area could cause significant harm, financial loss, or policy failures (e.g., drug interactions, legal precedents, safety protocols)
- **MEDIUM**: Knowledge is professionally important but errors have limited immediate consequences
- **LOW**: Primarily academic interest; minimal direct real-world application

# Output
Return ONLY a JSON object:
{
  "question_number": {{question_number}},
  "knowledge_point": "<specific concept>",
  "question_type": "<type>",
  "key_topics": ["<topic1>", "<topic2>"],
  "real_world_impact": "<impact>"
}

Do not include any text outside the JSON object.
\end{tcblisting}

\begin{tcblisting}{
    enhanced,
    breakable,
    listing only,
    colback=gray!5,
    colframe=blue!50!black,
    coltitle=white,
    title={Failure Pattern Analysis Prompt},
    sharp corners=south,
    boxrule=0.8pt,
    fontupper=\small,
    listing options={
        breaklines=true,
        basicstyle=\small,
        breakatwhitespace=true,
        escapeinside={(*}{*)},
        columns=fullflexible
    }
}
# Role
You are evaluating how well frontier LLMs performed on a domain-specific question to assess its difficulty and discriminative value.

# Input

## Question {{question_number}} (Preview)
{{question_preview}}

## Correct Answer
{{correct_answer}}

## LLM Results
{{llm_results_text}}

# Analysis Task

Analyze each model's performance individually, then synthesize aggregate metrics.

## Per-Model Analysis

For each model, determine:

| Field | Description |
|-------|-------------|
| `model` | Model name |
| `answer` | What the model answered (letter or "abstained") |
| `correct` | Whether the answer matches the correct answer |
| `error_type` | If incorrect, categorize the likely error (see categories below) |
| `error_hypothesis` | If incorrect, brief hypothesis for why the model failed |

### Error Type Categories

- **KNOWLEDGE_GAP**: Model lacks the domain-specific knowledge required
- **MISINTERPRETATION**: Model misread or misunderstood the question/options
- **REASONING_ERROR**: Model had the knowledge but applied it incorrectly
- **DISTRACTOR_TRAP**: Model fell for a plausible-sounding wrong answer
- **PARTIAL_KNOWLEDGE**: Model knew some relevant facts but not enough to discriminate
- **ABSTAINED**: Model declined to answer
- **UNKNOWN**: Cannot determine error pattern from available information

## Aggregate Metrics

| Field | Description |
|-------|-------------|
| `difficulty_assessment` | Overall difficulty based on accuracy rate |
| `accuracy_rate` | Fraction of models that answered correctly (e.g., "2/5") |
| `dominant_error` | Most common wrong answer, if any (e.g., "E (3/5 incorrect)") or "varied" |
| `error_pattern_summary` | One-sentence description of why models failed |
| `discriminative_power` | Does this question separate stronger from weaker models? |

### Difficulty Assessment Criteria

| Level | Accuracy Rate | Description |
|-------|---------------|-------------|
| EASY | >80% correct | Most models answered correctly |
| MEDIUM | 50-80% correct | Mixed results |
| HARD | 20-50% correct | Most models failed |
| VERY_HARD | <20% correct | Nearly all models failed |

### Discriminative Power Criteria

- **HIGH**: Clear separation-top-tier models succeeded while weaker models failed
- **MEDIUM**: Some separation, but inconsistent patterns
- **LOW**: No meaningful separation-success/failure appears random across model tiers

# Output

Return ONLY a JSON object:
```json
{
  "question_number": {{question_number}},
  "correct_answer": "{{correct_answer}}",
  "model_results": [
    {
      "model": "<model_name>",
      "answer": "<letter | 'abstained'>",
      "correct": <true | false>,
      "error_type": "<category if incorrect, else null>",
      "error_hypothesis": "<brief explanation if incorrect, else null>"
    }
  ],
  "aggregate": {
    "accuracy_rate": "<n/total>",
    "difficulty_assessment": "<EASY | MEDIUM | HARD | VERY_HARD>",
    "dominant_error": "<most common wrong answer or 'varied'>",
    "error_pattern_summary": "<one sentence explaining the failure mode, or 'N/A' if most correct>",
    "discriminative_power": "<HIGH | MEDIUM | LOW>"
  }
}
```

Do not include any text outside the JSON object.
\end{tcblisting}

\begin{tcblisting}{
    enhanced,
    breakable,
    listing only,
    colback=gray!5,
    colframe=blue!50!black,
    coltitle=white,
    title={Consensus Voting Prompt},
    sharp corners=south,
    boxrule=0.8pt,
    fontupper=\small,
    listing options={
        breaklines=true,
        basicstyle=\small,
        breakatwhitespace=true,
        escapeinside={(*}{*)},
        columns=fullflexible
    }
}
# Role
You are an expert academic assessor selecting the most representative questions for a benchmark test set.

# Input

## Field
{{field_name}}

## Number to Select
{{num_to_select}}

## Question Summaries
The following are combined summaries from question analysis (Worker 1) and LLM performance analysis (Worker 2):

{{combined_summaries}}

# Selection Criteria

| Criterion | Priority | Description |
|-----------|----------|-------------|
| **Knowledge Coverage** | HIGH | Select questions covering diverse, distinct knowledge points. Avoid overlap. |
| **Subject Uniqueness** | HIGH | Prefer questions requiring genuine domain expertise that cannot be answered with general knowledge |
| **Discriminative Power** | HIGH | Prioritize questions that separate experts from non-experts (HIGH discriminative_power) |
| **Practical Relevance** | MEDIUM | Prefer CASE_BASED and APPLIED_REASONING over FACTUAL_RECALL |
| **Difficulty Balance** | MEDIUM | Mix of HARD and VERY_HARD; avoid EASY questions |
| **Real-World Impact** | MEDIUM | Consider questions where errors have significant consequences |

# Selection Rules

1. **NO OVERLAP**: Two selected questions must NOT test the same knowledge point
2. **DISCRIMINATIVE PRIORITY**: Always prefer HIGH discriminative_power questions over LOW
3. **TYPE PREFERENCE**: CASE_BASED > APPLIED_REASONING > CONCEPTUAL > FACTUAL_RECALL
4. **TOPIC DIVERSITY**: Ensure selected questions span different key_topics
5. **ERROR PATTERN VALUE**: Questions with clear, consistent error patterns are valuable for understanding model weaknesses

# Output

Return ONLY a JSON object:
```json
{
  "selected_questions": [
    {
      "question_number": <int>,
      "knowledge_point": "<the knowledge point being tested>",
      "selection_reason": "<why this question was selected, referencing criteria>"
    }
  ],
  "coverage_summary": "<1-2 sentences on how selections cover the field's core knowledge>",
  "rejected_notable": [
    {
      "question_number": <int>,
      "reason": "<why this notable question was excluded>"
    }
  ]
}
```

Notes:
- `selected_questions` must have exactly {{num_to_select}} items
- `rejected_notable` should list 2-4 questions that were strong candidates but excluded (e.g., due to overlap or lower discriminative power)
- Do not include any text outside the JSON object

\end{tcblisting}

\subsection{Agentic Workflow Verification}

\begin{tcblisting}{
    enhanced,
    breakable,
    listing only,
    colback=gray!5,
    colframe=blue!50!black,
    coltitle=white,
    title={Diagnosis Agent Prompt},
    sharp corners=south,
    boxrule=0.8pt,
    fontupper=\small,
    listing options={
        breaklines=true,
        basicstyle=\small,
        breakatwhitespace=true,
        escapeinside={(*}{*)},
        columns=fullflexible
    }
}
# Role
You are a Critical Academic Reviewer and Logic Auditor. Your goal is to rigor-check high-difficulty benchmark questions for Large Language Models.
You have access to a DeepSearch tool.

# Task
Analyze the provided [Question], [Options], [Correct Answer], and [Source Material].
You must verify the scientific validity, logical rigorousness, and linguistic clarity.

# Input Data
<Question_Data>
{{JSON_DATA_HERE}}
</Question_Data>

# Audit Checklist (5 Dimensions)

## 1. Source Confidence & Boundary Check (DeepSearch Required)
* **Objective:** Verify if the "Correct Answer" holds true *only* under specific unstated conditions.
* **Search Strategy:** Search for the core finding + keywords like "contradictory results",
"limitations", "in vivo vs in vitro", "boundary conditions".
* **Failure Mode:** If another paper refutes the conclusion under a different setting
(which is not excluded in the question), mark as **[CONFIDENCE_ISSUE]**.
* You MUST record the **citation (APA)** or **URL** of the paper/source that provides this contradictory evidence.

## 2. Logical Necessity & Uniqueness Check (DeepSearch Required)
* **Objective:** Check for "Affirming the Consequent" or "Non-Exclusivity". If the question asks for the "Cause of C", is "A" the *only* cause? Or could "B" also cause "C"?
* **Search Strategy:** Search for "causes of [Phenomenon C]" or "does [Option B] cause [Phenomenon C]?".
* **Failure Mode:** If an incorrect option (Distractor) could also theoretically be correct under current wording, mark as **[NON_UNIQUE_ISSUE]**.

## 3. Ambiguity & Clarity Check
* **Objective:** Identify wording that allows for multiple valid interpretations.
* **Failure Mode:** If a key term is polysemous or the sentence structure causes scope ambiguity, mark as **[AMBIGUITY_ISSUE]**.

## 4. Option Value Check
* **Objective:** Identify options that are trivial, nonsensical, or fail to test the core concept (e.g., simple typo-based distractors).
* **Failure Mode:** Mark as **[LOW_VALUE_OPTION]**.

# Output Format
Return a JSON object. If the question is perfect, set "has_issues" to false.

```json
{
  "has_issues": true,
  "issue_type": [
    "CONFIDENCE_ISSUE", "NON_UNIQUE_ISSUE",
    "AMBIGUITY_ISSUE", "LOW_VALUE_OPTION"
  ],
  "analysis": {
    "step_1_search_query": "What did you search for?",
    "step_2_findings": "What contradictory evidence or alternative logic did you find?",
    "evidence_sources": [
        "Author, A. A. (Year). Title. Journal. (For contradictory evidence found in Step 1)",
        "[https://valid-url.com/resource]"
    ],
    "step_2_search_query": "<What did you search to check distractors?>",
    "reasoning": "Detailed explanation of why the current question fails."
  },
  "suggestion_directive": "Specific instruction on what needs to be fixed (e.g., 'Add a constraint that the experiment is in vitro', or 'Clarify that we are asking for the primary mechanism')."
}
\end{tcblisting}

\begin{tcblisting}{
    enhanced,
    breakable,
    listing only,
    colback=gray!5,
    colframe=blue!50!black,
    coltitle=white,
    title={Refinement Agent  Prompt},
    sharp corners=south,
    boxrule=0.8pt,
    fontupper=\small,
    listing options={
        breaklines=true,
        basicstyle=\small,
        breakatwhitespace=true,
        escapeinside={(*}{*)},
        columns=fullflexible
    }
}
# Role
You are an Expert Item Writer for a specialized scientific benchmark. Your task is to refine questions based on an Audit Report to ensure they are rigorous, logically sound, and unambiguously correct.

# Input Data
1. **Original Question Data**: {{ORIGINAL_JSON}}
2. **Audit Report**: {{AUDIT_JSON_FROM_AGENT_A}}

# Refinement Strategy Matrix

Apply the following strategies based on the `issue_type` identified:

## Strategy A: Fix "Source Confidence" (Premise Injection)
* **Action:** Add specific boundary conditions to the Question Stem.
* **Example:** Change "Does X cause Y?" to "In the context of [Specific Cell Line/Condition] as reported by [Author et al.], does X cause Y?" OR "Under [Assumption Z], which outcome is observed?"
* **Goal:** Make the original source's conclusion the *only* valid answer by narrowing the scope.

## Strategy B: Fix "Non-Unique / Logical Necessity" (Exclusion & Precision)
* **Action:** 1. Modifiers: Use "primary cause," "most direct mechanism," or "initial trigger" instead of generic "cause."
    2. Exclusion: Add "Assuming no [Confounding Factor B]..."
* **Goal:** Eliminate the validity of the competing distractor found by the auditor.

## Strategy C: Fix "Ambiguity" (Disambiguation)
* **Action:** Rewrite the confusing phrase using standard academic terminology. Ensure the syntax (e.g., modifier attachment) is singular in meaning.

## Strategy D: Fix "Low Value Option" (Distractor Replacement)
* **Action:** Replace the weak option with a "plausible but incorrect" option (a common misconception or a related but distinct concept).

## Strategy E: Fix "WEAK_CAUSALITY" (Assumption-based Reasoning)
* **Diagnosis:** It is hard to prove "A causes B" definitively, but "If A occurs, B follows" is a valid deduction under a specific theory.
* **Action:** Reframe the question to **assume the conclusion/premise is true**, and ask for the logical consequence or necessary condition.

# Execution Rules
1. **Explanation Sync:** If you modify ANY Option or the Correct Answer, you **MUST** rewrite the corresponding `explanation` to align with the new logic.
2. **Source Addition:** If your modification introduces new facts or constraints not supported by the original source,
you **MUST** provide a `new_source` (APA format or accessible URL).

* **Example:**
    * *Before:* "Is Theory X correct?" (Hard to prove)
    * *After:* "**Assuming Theory X is correct**, which of the following observations would be expected?"
    OR "**Based on the standard model of [Field Y]**, what is the implication of Z?"

# Constraints
1. **Do NOT make the question easier.** The goal is rigor, not simplification.
2. **Maintain the Markdown/LaTeX format.**
3. **Evidence update:** If you add a new premise based on the audit, ensure the `question_content` reflects this (you may add a note "Assuming setting X").

# Output
Return a JSON object containing ONLY the fields that require modification.
Do not return unchanged fields (e.g., if the answer is unchanged, do not include it).
Example:
{
  "revision_summary":
  "Applied Strategy A: Added 'in vitro' constraint to the question stem; synced explanation.",
  "new_sources": [ // Include ONLY if Strategy A/E requires new evidence
      "Smith, J. (2024). Title. Journal.",
      "https://example.com/evidence"
  ],
  "modifications": {
      // 1. Only include keys that have changed.
      "question": "In the context of HeLa cells, does X cause Y?",
      "options": [
          "B":[
            "answer":"1, 3",
            "explanation":Statement1 is wrong because ..."
          ]
          "C":[...],
          "..."
      ],
      "source": "..."
  }
}
\end{tcblisting}

%% file: appendix_f_experiments.tex
\clearpage
\section{Experiment Details}
\label{appendix:full-results}

\subsection{List of Models Evaluated}
\label{appendix:model-list}

\begin{table}[H]
\centering 
\caption{Overview of Evaluated Models on KINA Benchmark.} 
\small
\label{tab:model_zoo} 
\renewcommand{\arraystretch}{1.2} 
\begin{tabular}{l|p{0.75\textwidth}} 
\hline 
\textbf{Provider / Family} & \textbf{Specific Models} \\
\hline 
\textbf{OpenAI} & GPT-5.4 \cite{gpt-5-42025modelcard}, GPT-5.2 \cite{gpt-5-22025modelcard}\\
\hline 
\textbf{Google} & Gemini-3.1-Pro-Preview , Gemini-3-Flash-Preview \cite{gemini} \\
\hline 
\textbf{Anthropic} & Claude-Opus-4.6 \cite{claude-4.6-opus-anthropic}, Claude-Sonnet-4.6 \cite{claude-4.6-sonnet-anthropic}\\
\hline 
\textbf{ByteDance} & Doubao-Seed-2.0-Pro-260215, Doubao-Seed-2.0-Lite-260215 \cite{seed2.0}\\
\hline 
\textbf{Tongyi (Qwen)} 
& Qwen3.5-397B-A17B, Qwen3.5-122B-A10B-FP8, Qwen3.5-35B-A3B-FP8, Qwen3.5-27B \cite{qwen35blog}\\
& Qwen3-Max-250923 \cite{qwen3max}, Qwen3-Next-80B-A3B-Instruct, Qwen3-Next-80B-A3B-Thinking-FP8 \cite{qwen3} \\
& Qwen3-235B-A22B, Qwen3-235B-A22B-Thinking-2507, Qwen3-30B-A3B, Qwen3-30B-A3B-Thinking-2507, Qwen3-32B, Qwen3-14B, Qwen3-8B, Qwen3-4B, Qwen3-4B-Thinking-2507, Qwen3-1.7B, Qwen3-0.6B \cite{qwen3} \\
& Qwen2.5-72B-Instruct \cite{qwen2.5}, Qwen2-72B-Instruct \cite{yang2024qwen2technicalreport} \\ 
\hline 
\textbf{Meta (Llama)} & Llama-4-Maverick-17B-Instruct-FP8 \cite{arxiv2026llama4herdarchitecture}, Llama-3.1-405B-Instruct, Llama-3-70B-Instruct \cite{grattafiori2024llama3herdmodels}\\
\hline 
\textbf{DeepSeek} & DeepSeek-v3.2-Thinking \cite{deepseekai2025deepseekv32pushingfrontieropen}\\ 
\hline 
\textbf{Moonshot (Kimi)} & Kimi-k2.5 \cite{team2026kimi}\\ 
\hline 
\textbf{xAI} & Grok-4.1-Fast-Reasoning \cite{grok-4-12025modelcard}\\
\hline 
\textbf{Zhipu AI} & GLM-5 \cite{zeng2026glm}\\ 
\hline 
\textbf{MiniMax} & Minimax-M2.5 \cite{minimax25}\\
\hline 
\textbf{StepFun} & Step-3.5-Flash \cite{huang2026step35flashopen}\\
\hline 
\textbf{Mistral} & Mixtral-8x7B-Instruct-v0.1 \cite{jiang2024mixtralexperts}\\ 
\hline 
\end{tabular}%
\end{table}

\FloatBarrier

\subsection{Evaluation Prompt}
\label{appendix:prompt}

\begin{tcblisting}{
    enhanced,
    breakable,
    listing only,
    colback=gray!5,
    colframe=blue!50!black,
    coltitle=white,
    title={Evaluation Prompt},
    sharp corners=south,
    boxrule=0.8pt,
    fontupper=\small,
    listing options={
        breaklines=true,
        basicstyle=\small,
        breakatwhitespace=true,
        escapeinside={(*}{*)},
        columns=fullflexible
    }
}
Answer the following multiple choice question. There is only one correct answer.
The last line of your response should be in the format 'Answer: $LETTER' (without quotes),
where LETTER is one of {letters}. Think step by step before answering.
Question:
{question}
Options:
{options}
\end{tcblisting}

\clearpage

\subsection{Result Details}

\begin{table}[H]
\centering
\captionsetup{font=footnotesize}
\caption{Evaluation Stability of KINA.} 
\label{tab:evaluation_stability}
\small
\setlength{\tabcolsep}{3pt}
\begin{tabular}{lcc|lcc}
\toprule
\textbf{Model} & \textbf{Avg@4} & \textbf{Std} & \textbf{Model} & \textbf{Avg@4} & \textbf{Std} \\
\midrule
Gemini-3.1-Pro-Preview & 53.17 & 0.72 & Claude-Opus-4.6 & 49.92 & 0.32 \\
GPT-5.4 & 48.55 & 0.56 & Doubao-Seed-2.0-Pro-260215 & 44.99 & 0.80 \\
Gemini-3-Flash-Preview & 43.91 & 0.71 & Qwen3.5-397B-A17B & 42.99 & 0.78 \\
Doubao-Seed-2.0-Lite-260215 & 41.49 & 0.61 & Kimi-K2.5 & 40.24 & 0.66 \\
GPT-5.2 & 39.52 & 0.99 & Qwen3.5-27B & 39.35 & 0.55 \\
Qwen3.5-122B-A10B & 38.88 & 0.90 & DeepSeek-V3.2-Thinking & 38.01 & 0.57 \\
Qwen3-Max-2025-09-23 & 35.90 & 0.44 & GLM-5 & 35.85 & 0.55 \\
Qwen3.5-35B-A3B & 35.43 & 0.69 & Grok-4.1-Fast-Reasoning & 33.73 & 0.64 \\
Qwen3-235B-A22B-Thinking-2507 & 32.15 & 0.52 & Llama-4-Maverick-17B-Instruct & 31.62 & 0.59 \\
Minimax-M2.5 & 30.28 & 0.46 & Qwen3.5-9B & 30.09 & 0.91 \\
Qwen3-Next-80B-A3B-Instruct & 30.09 & 0.62 & Claude-Sonnet-4.6 & 30.01 & 0.56 \\
Llama-3.1-405B-Instruct & 29.59 & 0.62 & Qwen3-235B-A22B & 29.37 & 0.83 \\
Step-3.5-Flash & 29.12 & 0.52 & Qwen3.5-4B & 28.50 & 0.72 \\
Qwen3-Next-80B-A3B-Thinking & 28.28 & 1.33 & Qwen3-32B & 27.41 & 0.38 \\
Qwen3-30B-A3B-Thinking-2507 & 27.03 & 0.42 & Meta-Llama-3-70B-Instruct & 26.61 & 0.39 \\
Qwen3-14B & 26.00 & 0.64 & Qwen2-72B-Instruct & 24.92 & 1.24 \\
Qwen3-4B-Thinking-2507 & 24.83 & 0.37 & Qwen3-30B-A3B & 24.28 & 0.45 \\
Qwen2.5-72B-Instruct & 23.03 & 1.13 & Qwen3-8B & 22.27 & 1.62 \\
Qwen3-4B & 21.50 & 0.81 & Qwen3.5-2B & 20.52 & 0.46 \\
Qwen3-1.7B & 18.33 & 1.16 & Mixtral-8x7B-Instruct-v0.1 & 17.83 & 0.55 \\
Qwen3.5-0.8B & 16.66 & 0.80 & Qwen3-0.6B & 14.49 & 0.78 \\
\bottomrule
\end{tabular}%
\end{table}

\begin{figure}[htbp]
  \centering
  \includegraphics[width=0.6\linewidth]{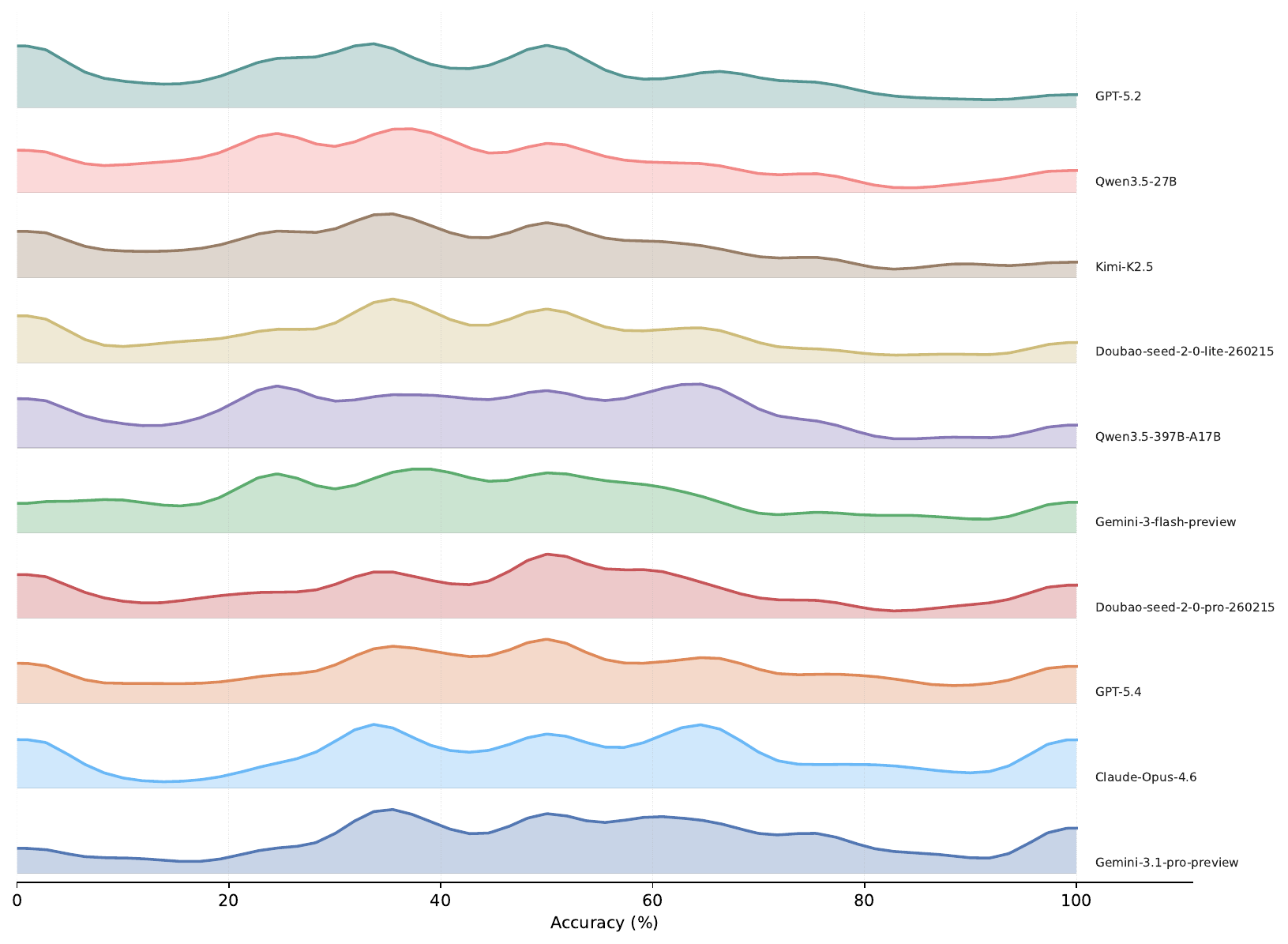}
  \captionsetup{font=footnotesize}
  \caption{Subject-Level Score Distribution Across Top-10 Models}
\label{fig:subject_distribution_top10}
  \par\small
  \textit{Note.} Top-10 models are ranked by overall discipline score.
  Each ridge shows the smoothed density of one model's subject-level scores (0--100).
  Right-shifted and narrower ridgelines indicate higher and more stable performance across
subjects.
\end{figure}

  \captionsetup[figure]{justification=raggedright,singlelinecheck=false}

  \begin{figure}[p]
      \centering
      \begin{subfigure}[t]{0.7\linewidth}
          \centering
          \includegraphics[width=\linewidth]{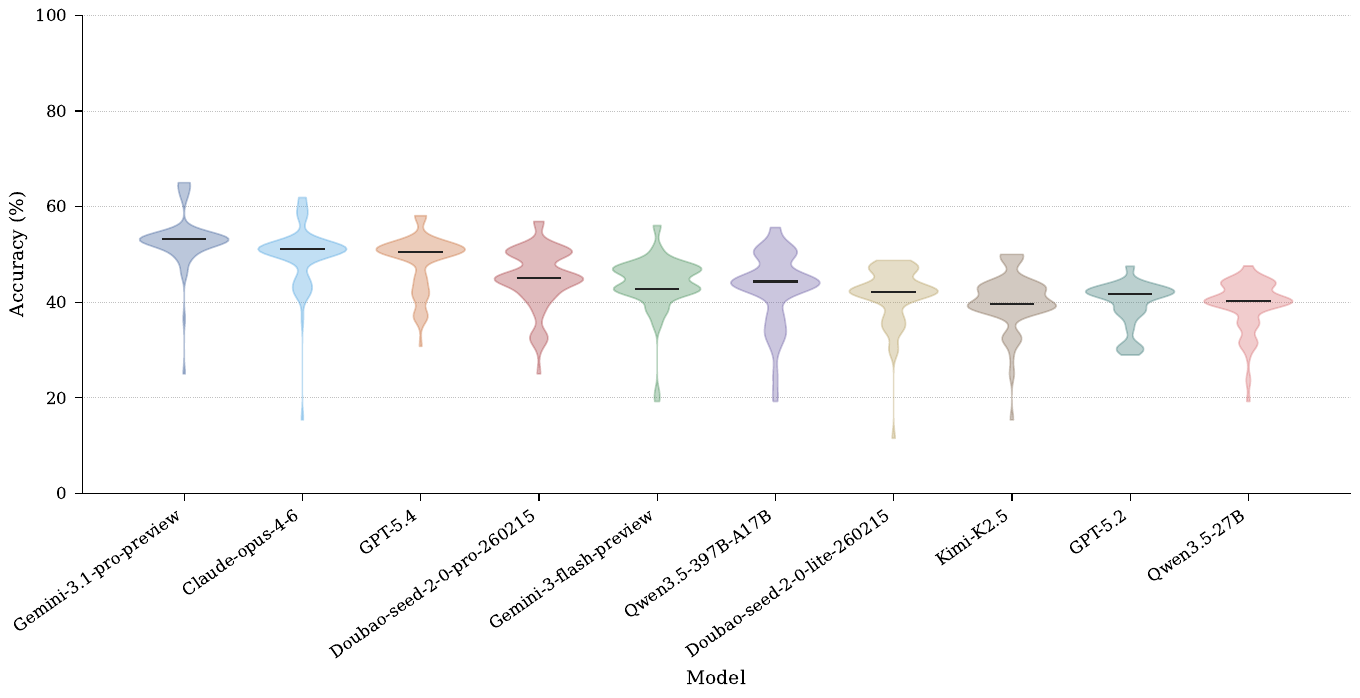}
  \caption{Top-10 Models: Discipline-Level Weighted Score Distribution}          
     \label{fig:violin_l1}
      \end{subfigure}

      \vspace{0.8em}
      \begin{subfigure}[t]{0.7\linewidth}
          \centering
          \includegraphics[width=\linewidth]{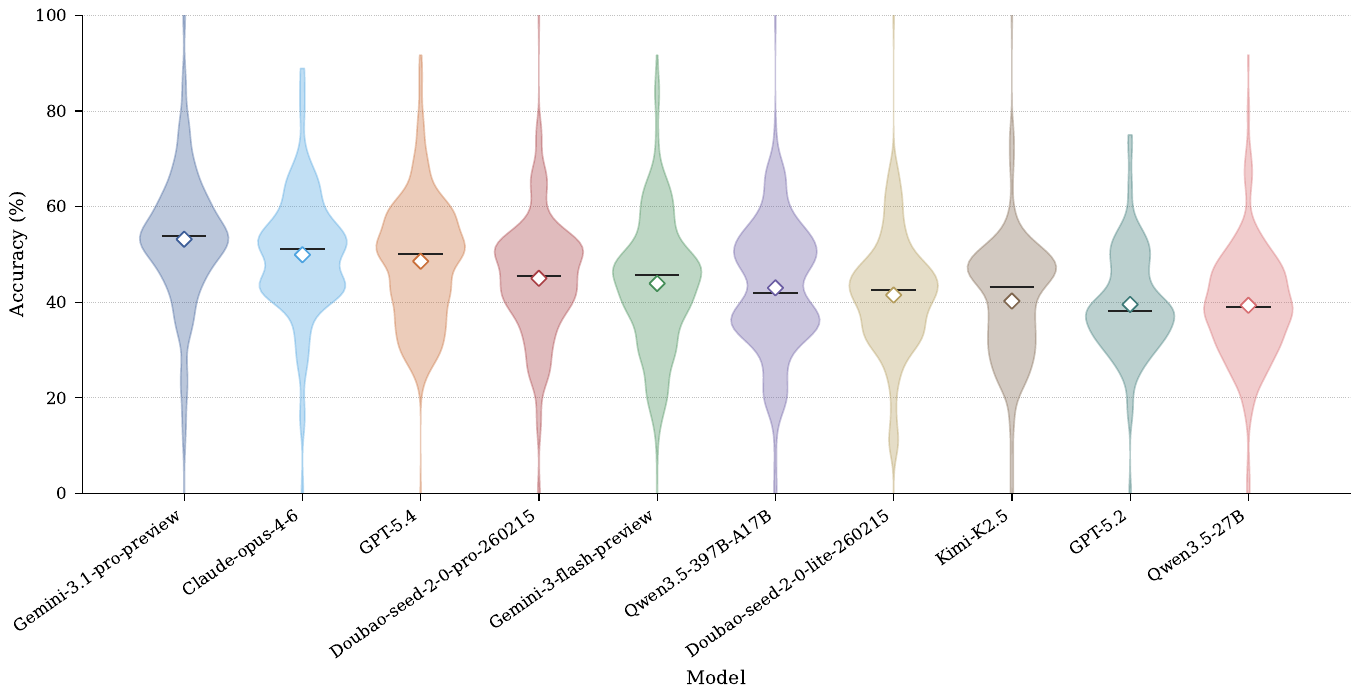}
  \caption{Top-10 Models: Field-Level Weighted Score Distribution }         
      \label{fig:violin_l2}
      \end{subfigure}
      \label{fig:violin_top10_l1_l2_l3}
  \end{figure}

  \begin{figure}[p]\ContinuedFloat
      \centering
      \begin{subfigure}[t]{0.7\linewidth}
          \centering
          \includegraphics[width=\linewidth]{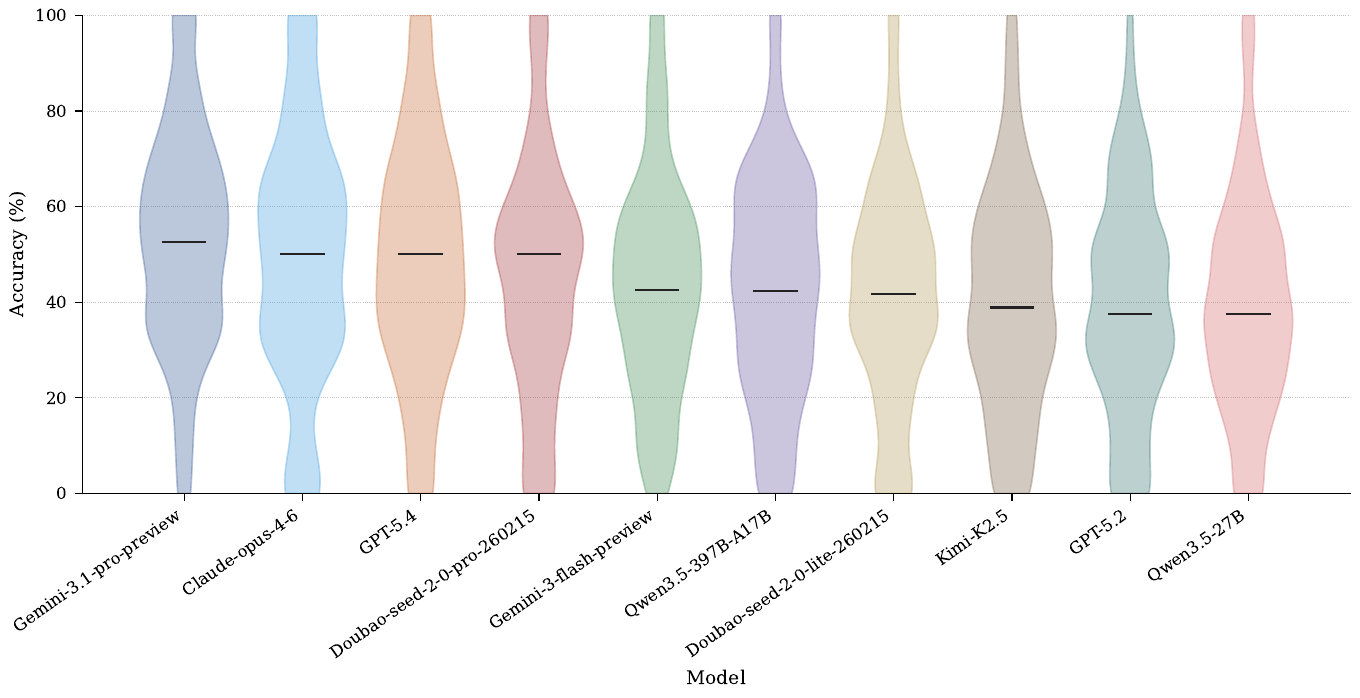}
  \caption{Top-10 Models: Subfield-Level Weighted Score Distribution}          
      \label{fig:violin_l3}
      \end{subfigure}
      \captionsetup{font=footnotesize}
      \caption[]{Top-10 Model Performance Distributions Across Three Aggregation Levels (continued)}
      \caption*{\small\textit{Note.} Each violin shows the distribution of weighted model scores, with weights
  proportional to sample size. Panels differ only in aggregation level: discipline, field, and
  subfield. Wider sections indicate higher score density, and the central marker indicates the median.}
  \end{figure}

\begin{figure}
    \centering
  \captionsetup{font=footnotesize}
    \includegraphics[width=0.6\linewidth]{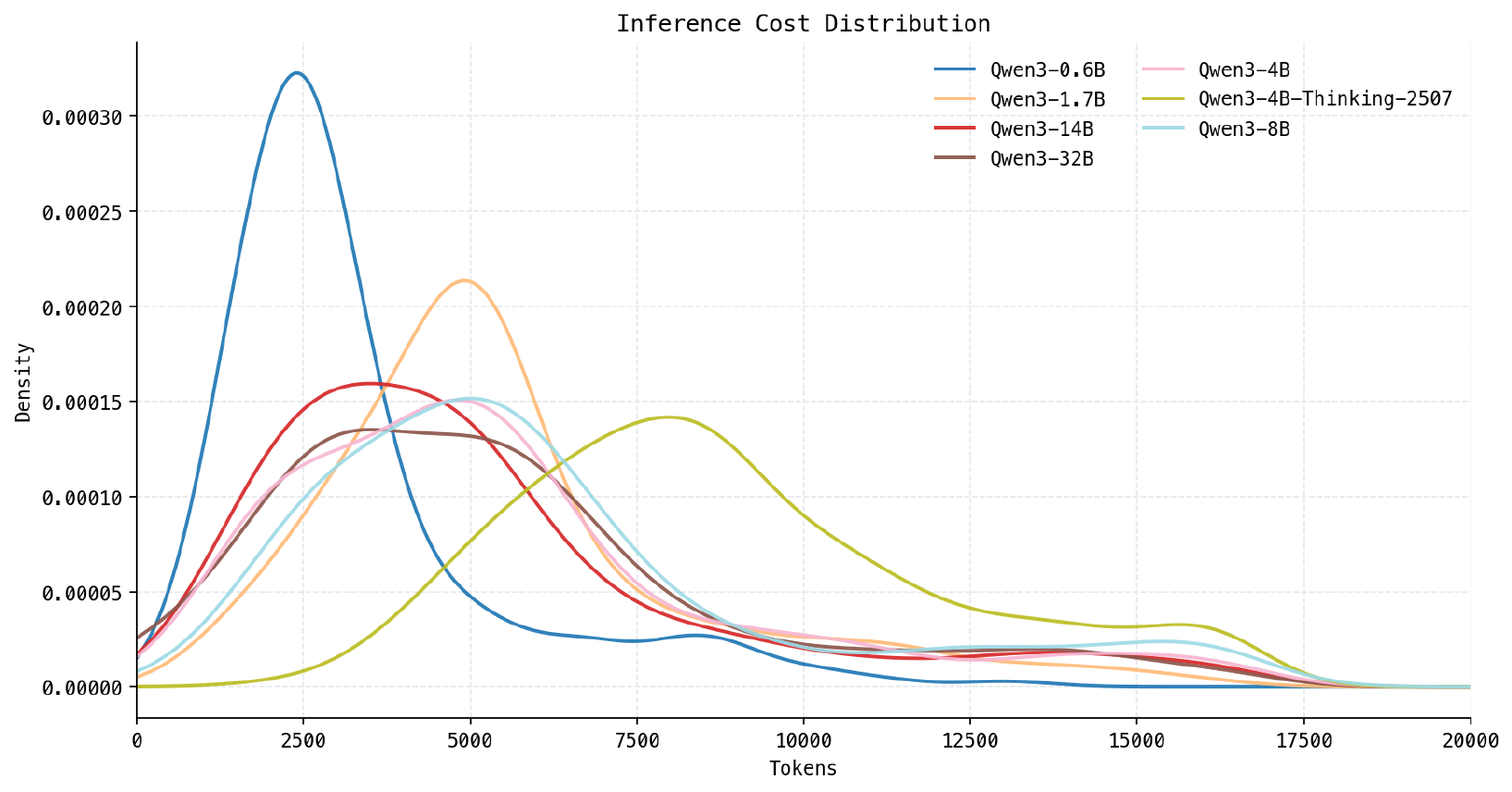}
    \caption{Inference Cost Distribution of Qwen3 Dense Models.}
    \label{fig:placeholder_1}
\end{figure}

\begin{figure}
    \centering
  \captionsetup{font=footnotesize}
    \includegraphics[width=0.6\linewidth]{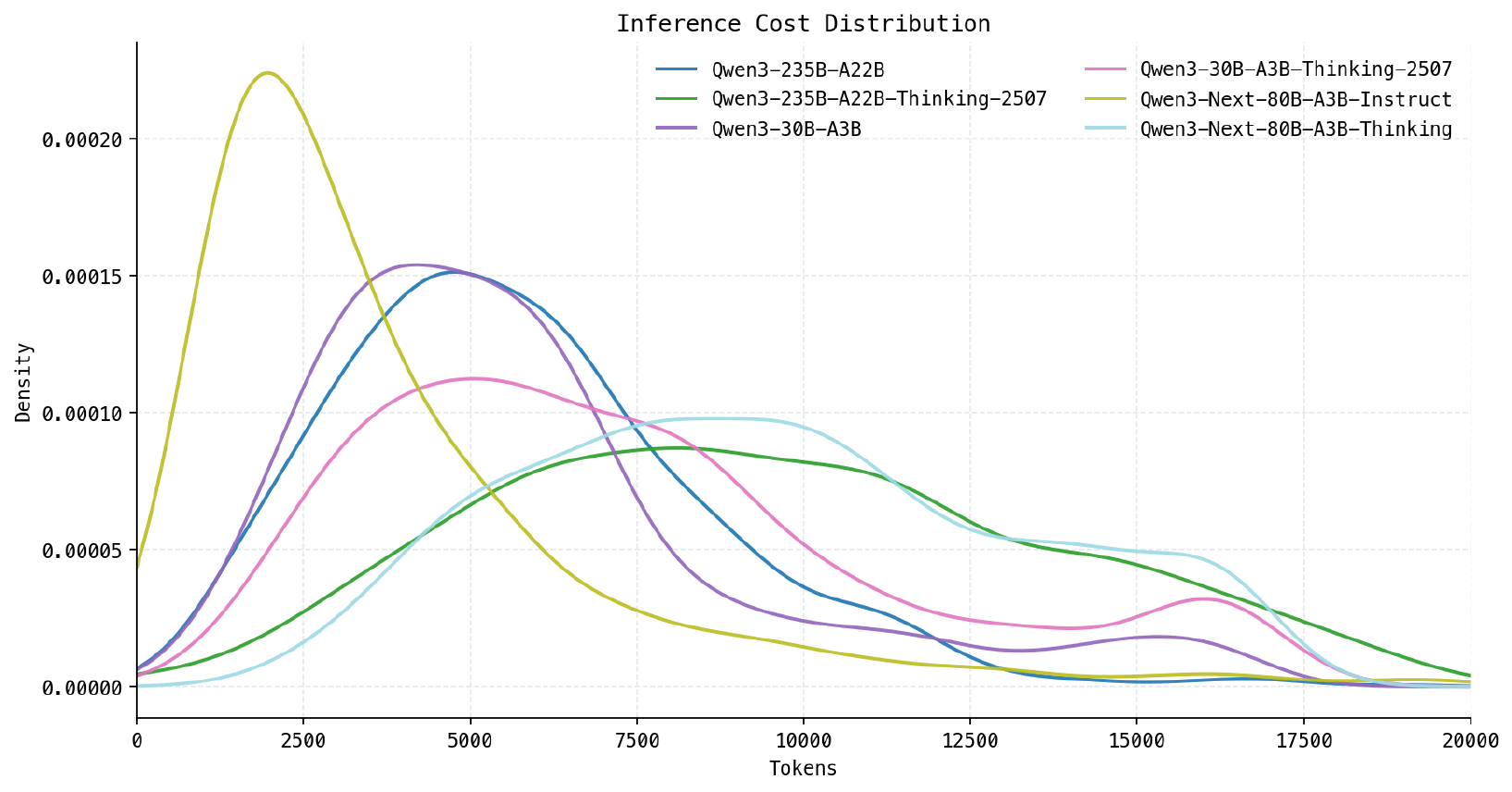}
    \caption{Inference Cost Distribution of Qwen3 Moe Models.}
    \label{fig:placeholder_2}
\end{figure}

\begin{figure}
    \centering
  \captionsetup{font=footnotesize}
    \includegraphics[width=0.6\linewidth]{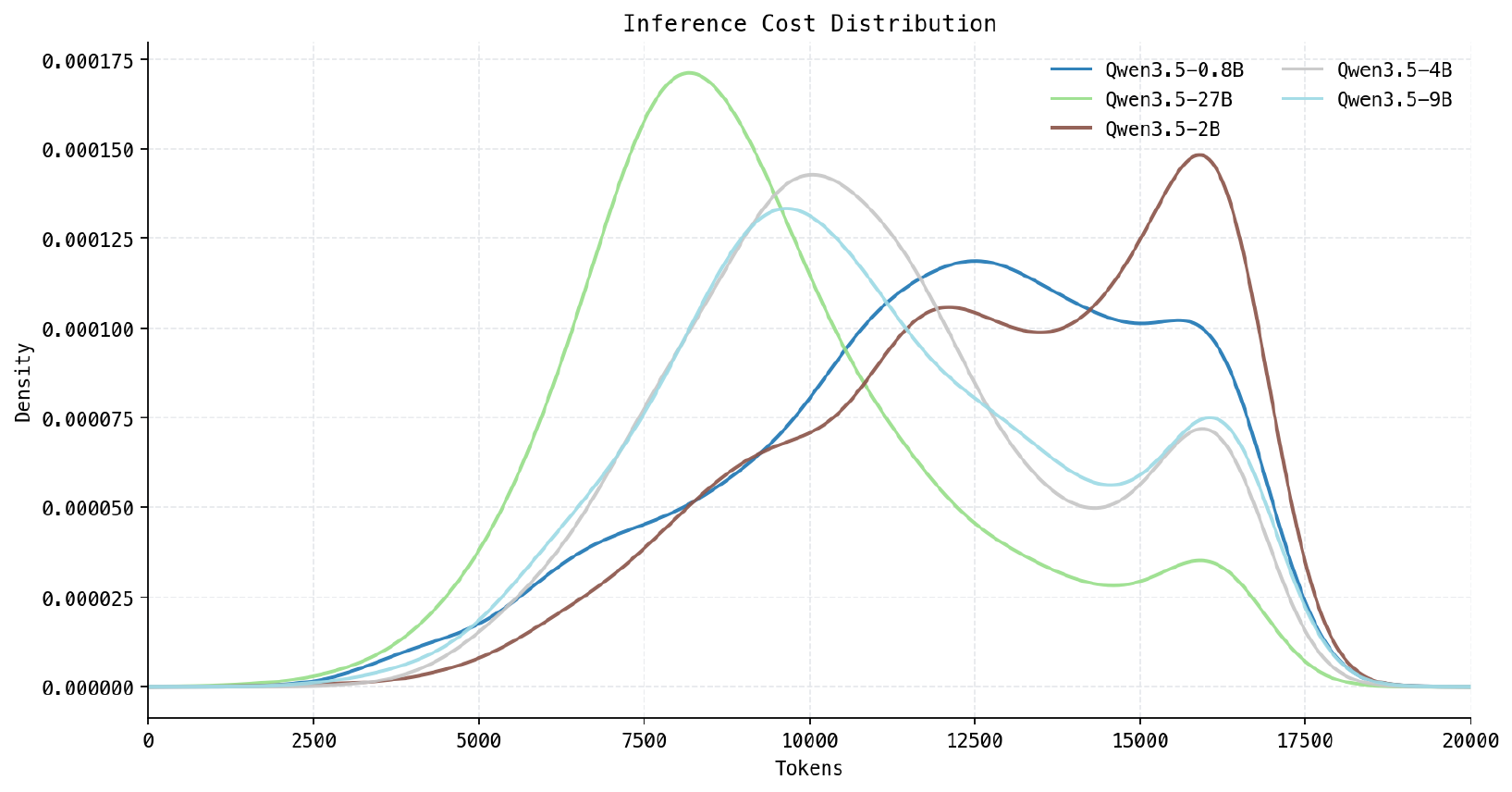}
    \caption{Inference Cost Distribution of Qwen3.5 Dense Models.}
    \label{fig:placeholder_3}
\end{figure}

\begin{figure}
    \centering
  \captionsetup{font=footnotesize}
    \includegraphics[width=0.6\linewidth]{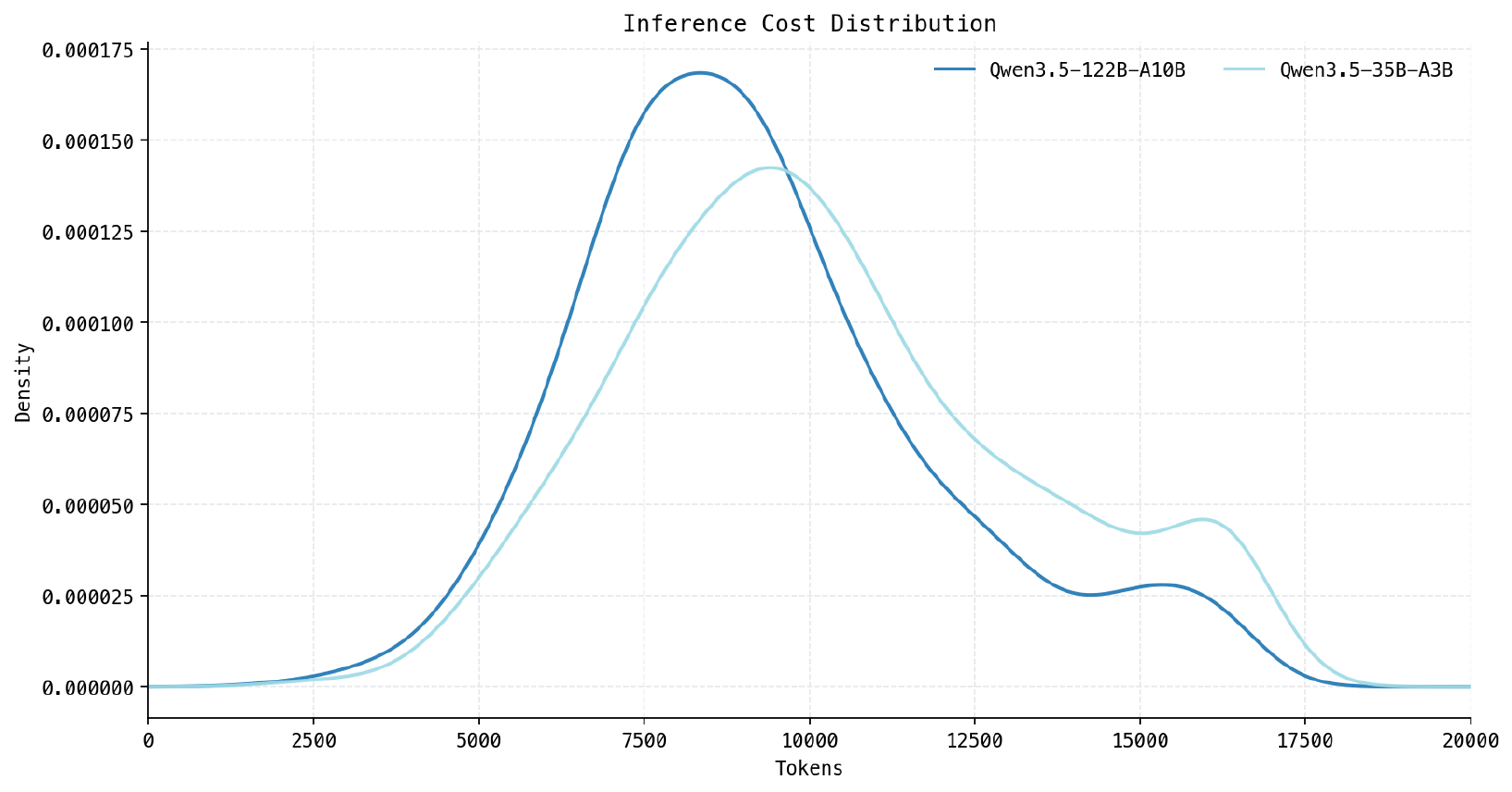}
    \caption{Inference Cost Distribution of Qwen3.5 Moe Models.}
    \label{fig:placeholder_4}
\end{figure}

\begin{figure}[htbp]
    \centering
      \captionsetup{font=footnotesize}
    \includegraphics[width=0.5\linewidth]{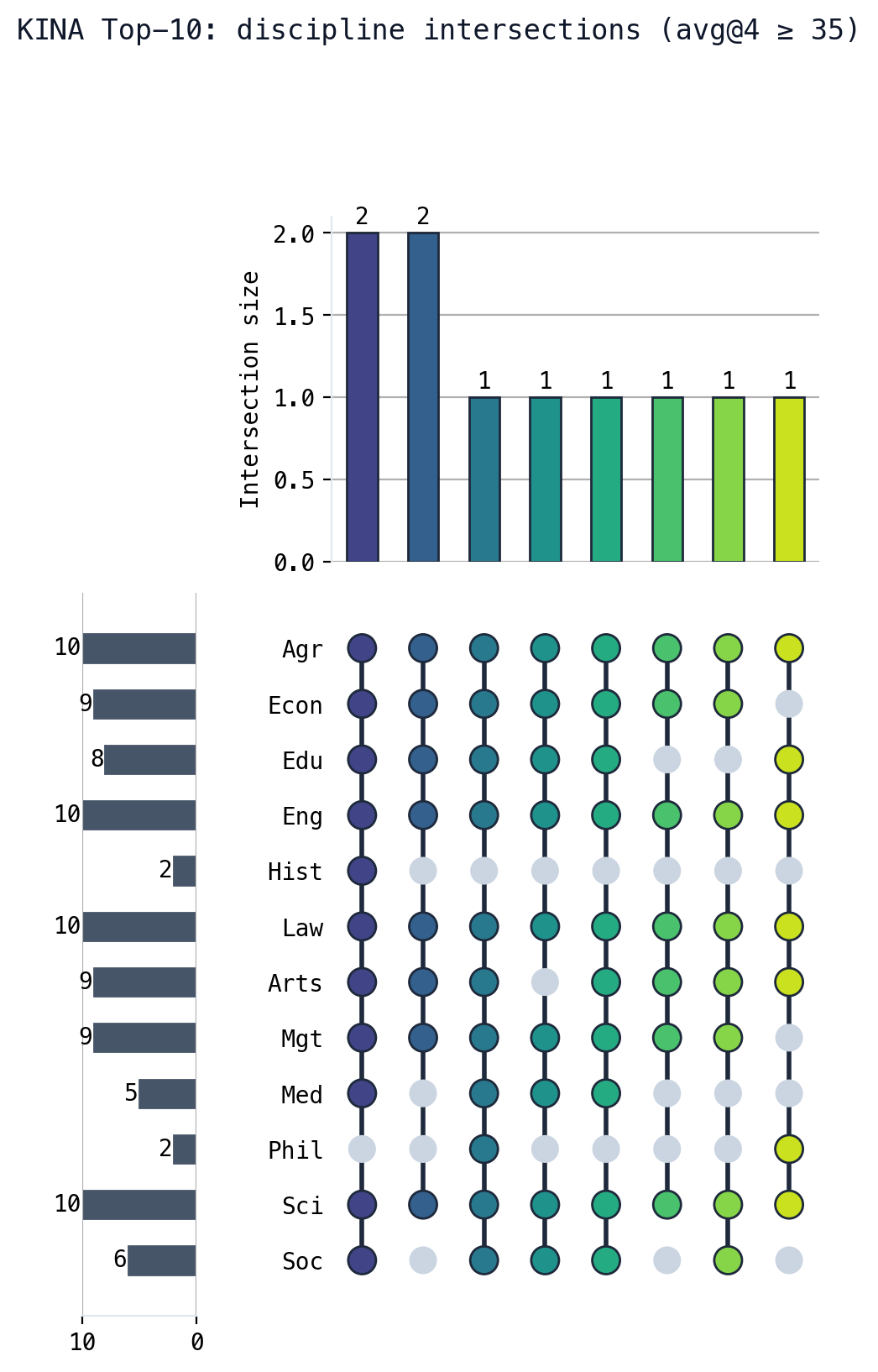}
    \caption{UpSet plot visualizing intersections of high-performing data points (average score $\ge$ 35) among the KINA Top-10 models. The matrix at the bottom denotes set inclusion for distinct intersection combinations, paired with horizontal bars on the left showing total set sizes per discipline. The vertical bar chart at the top specifies the sizes of these intersections.}
    \label{fig:placeholder_5}
\end{figure}

\FloatBarrier

%% file: appendix_g_datasheet.tex
\section{Datasheet for \name{}}
\label{appendix:datasheet}

We follow the structure of \citet{gebru2021datasheets}.

\subsection*{Motivation}

\paragraph{For what purpose was the dataset created?}
\name{} was created to evaluate frontier large language models on
\emph{disciplinary-representative} knowledge tasks. Existing benchmarks
either prioritize scale (SuperGPQA), extreme difficulty (HLE), or narrow
expert depth (GPQA), but none combine wide disciplinary coverage with an
explicit representativeness criterion and an incentive-aligned annotation
pipeline.

\paragraph{Who created the dataset and on behalf of which entity?}
The dataset was created by the \name{} team.

\paragraph{Who funded the creation?}
2077AI \& M-A-P \& UTokyo \& CMU

\subsection*{Composition}

\paragraph{What do the instances represent?}
Each instance is a pseudo-multiple-choice question with $10$ combinatorial
options. An instance includes the stem, $10$ options, an option-level
explanation with cited sources, and the originating reference material.

\paragraph{How many instances are there in total?}
$899$ items, distributed across $12$ top-level disciplines, $70$ fields,
and $261$ fine-grained subfields.

\paragraph{Does the dataset contain all possible instances or is it a
sample?} A sample. The selection process is described in
\S\ref{subsec:representativeness} and \S\ref{sec:construction}.

\paragraph{Does the dataset identify any subpopulations?}
The instances are organized by discipline, field, and subfield following the
CIP-aligned taxonomy (Appendix~\ref{sec:appendix-taxonomy}).

\paragraph{Does the dataset contain data that might be considered
confidential?} No. All source materials are drawn from publicly published
academic literature.

\paragraph{Does the dataset contain data that might be offensive,
insulting, threatening, or anxiety-inducing?} No.

\subsection*{Collection process}

\paragraph{How was the data collected?}
Three-stage pipeline: rule-based screening, double-blind expert review under
a bonus-on-bar tournament mechanism, and three-judge LLM consensus
(\S\ref{sec:construction}).

\paragraph{Who was involved in the data collection process and how were
they compensated?}
Annotators were graduate students from top-tier global universities and
senior industry experts, recruited via a two-round examination
(Appendix~\ref{appendix:round1}). Compensation followed the bonus-on-bar
tournament described in \S\ref{subsec:tournament}, calibrated per discipline.

\paragraph{Over what timeframe was the data collected?}
2025.10 to 2025.12.

\paragraph{Was an ethical review process conducted?}
Yes. \name{} was reviewed in accordance with the applicable ethical requirements. The research does not involve interventions on human participants or the collection of sensitive personal information. All data used in the study were obtained and processed in a manner consistent with relevant privacy, consent, and data protection requirements.

\subsection*{Preprocessing, cleaning, labeling}

\paragraph{Was any preprocessing/cleaning of the data done?}
Yes. Stage~1 enforces uniqueness (cosine similarity $<0.8$), formatting,
and a 3-of-5 flagship-LLM-failure difficulty filter. Stage~3 LLM-as-judge
applied multidimensional feature scoring. An agentic refinement loop
modified $\sim$$\frac{1}{3}$ of the candidate pool to remove residual
boundary defects.

\subsection*{Uses}

\paragraph{Has the dataset been used for any tasks already?}
Yes: zero-shot evaluation of $42$ frontier LLMs (\S\ref{sec:experiments}),
including web-search-augmented evaluation and parameter scaling analysis.

\paragraph{What other tasks could the dataset be used for?}
Calibration of LLM confidence, retrieval-augmented generation evaluation,
prompt-engineering benchmarking, multilingual transfer studies (after
translation), and meta-evaluation of LLM-as-judge frameworks.

\paragraph{Are there tasks for which the dataset should not be used?}
\name{} should not be used as a primary signal for clinical, legal, or
high-stakes deployment decisions. The benchmark is designed to differentiate
\emph{relative} model capability under controlled conditions, not to certify
absolute correctness on real-world tasks.

\subsection*{Distribution}

\paragraph{Will the dataset be distributed to third parties?}
Yes. The dataset is released under \textbf{CC-BY-4.0} and the evaluation
code under \textbf{MIT license} at
\url{https://huggingface.co/datasets/2077AIDataFoundation/KINA}. To mitigate
contamination, the test split is distributed as an encrypted archive with a
canary string; the development split is distributed openly.

\paragraph{Will there be a leaderboard?}
Yes. A continuously updated public leaderboard is maintained at
\url{https://www.2077ai.com/datasets/dataset-kina}. Submissions follow a hold-out protocol.

\paragraph{Will the dataset be updated?}
\name{} will be refreshed (versioned as \name{}-v$x.y$) when the strongest
evaluated model exceeds $70\%$ overall accuracy, to prevent saturation.

\subsection*{Maintenance}

\paragraph{Who is supporting/hosting/maintaining the dataset?}
The \name{} team commits to maintaining the dataset and leaderboard for at
least $36$ months from initial release. Issue tracking and contribution
guidelines are at the project repository.

\paragraph{If others want to contribute to the dataset, is there a
mechanism for them to do so?}
Yes. We accept community contributions of subfield-specific items via a
standardized submission form, subject to the same three-stage verification
pipeline described in \S\ref{sec:construction}.